%% file: 0-book.tex
\pdfoutput=1

\documentclass[graybox,envcountchap,sectrefs]{svmono}

\usepackage{mathptmx}
\usepackage{helvet}
\usepackage{courier}
\usepackage{type1cm}         

\usepackage{makeidx}         
\usepackage{graphicx}        
\usepackage{multicol}        
\usepackage[bottom]{footmisc}

\usepackage{hyperref}
\usepackage{mathtools}						
\usepackage[utf8]{inputenc}
\usepackage{amssymb}

\usepackage[nottoc]{tocbibind}

\usepackage{subcaption}

\usepackage[boxed]{algorithm2e}

\usepackage[normalem]{ulem} 

\newcommand{\dario}{}

\newcommand{\ri}{\right}

\newcommand {\bR}{{\mathbb R}}
\newcommand {\bN}{{\mathbb N}}
\newcommand {\bZ}{{\mathbb Z}}
\newcommand {\bC}{{\mathbb C}}

\newcommand {\bG}{{\mathbb G}}
\newcommand {\bH}{{\mathbb H}}
\newcommand {\bK}{{\mathbb K}}
\newcommand {\bT}{{\mathbb T}}

\newcommand {\bS}{{\mathbb S}}

\usepackage{dsfont}
\DeclareMathOperator{\idty}{{Id}}

\newcommand{\cA}{{\mathcal A}} %
\newcommand{\cB}{{\mathcal B}} 
\newcommand{\cC}{{\mathcal C}} %
\newcommand{\cE}{{\mathcal E}} %
\newcommand{\cF}{{\mathcal F}} %
\newcommand{\cG}{{\mathcal G}} %
\newcommand{\cH}{{\mathcal H}}
\newcommand{\cI}{{\mathcal I}} %
\newcommand{\cJ}{{\mathcal J}}
\newcommand{\cK}{{\mathcal K}}
\newcommand{\cL}{{\mathcal L}}

\newcommand{\cO}{{\mathcal O}} 
\newcommand{\cP}{{\mathcal P}}
\newcommand{\cR}{{\mathcal R}}
\newcommand{\cS}{{\mathcal S}} 
\newcommand{\cV}{{\mathcal V}}
\newcommand{\cU}{{\mathcal U}}

\newcommand{\cT}{{\mathcal T}}
\newcommand{\cX}{{\mathcal X}}
\newcommand{\cY}{{\mathcal Y}}

\newcommand{\bem}{\l(\! \begin{array}}
\newcommand{\eem}{\end{array}\!\ri)}
\newcommand{\bsm}{\left(\begin{smallmatrix}} 
\newcommand{\esm}{\end{smallmatrix}\right)}  

\DeclareMathOperator{\supp}{{supp}}

\DeclareMathOperator{\avg}{avg}

\DeclareMathOperator{\cent}{cent}
\DeclareMathOperator{\diag}{diag}
\DeclareMathOperator{\range}{range}
\DeclareMathOperator{\tr}{Tr}
\DeclareMathOperator{\Circ}{Circ}
\DeclareMathOperator{\rank}{rank}
\DeclareMathOperator{\spn}{span}
\DeclareMathOperator{\repr}{Rep}

\newcommand{\eps}{\varepsilon}
\DeclareMathOperator{\rep}{Rep}
\DeclareMathOperator{\AP}{AP}
\DeclareMathOperator{\HS}{HS}

\DeclareMathOperator{\PS}{PS}
\DeclareMathOperator{\BS}{BS}
\DeclareMathOperator{\RPS}{RPS}
\DeclareMathOperator{\RBS}{RBS}

\DeclareMathOperator{\ev}{ev}
\DeclareMathOperator{\sampl}{sampl}		

\DeclareMathOperator{\ce}{ce}
\DeclareMathOperator{\se}{se}

\newcommand{\bigo}{\mathcal O}

\usepackage[intoc]{nomencl}
\usepackage{etoolbox}
 
\makenomenclature

\begin{document}

\author{Dario Prandi \and Jean-Paul Gauthier}
\title{A semidiscrete version of the Citti-Petitot-Sarti model as a plausible model for anthropomorphic image reconstruction and pattern recognition}
\subtitle{-- Monograph --}
\maketitle

\frontmatter

\include{preface}

\tableofcontents

\mainmatter

\chapter{Introduction}\label{introduction-jp}
\include{introduction}

\chapter{Preliminaries}
\include{preliminaries}

\chapter{Lifts}\label{lifts}\label{ch:lift}
\include{lifts}

\chapter{Almost-periodic interpolation and approximation}\label{ch:ap_interp}
\include{ap_interp}

\chapter{Pattern recognition}\label{bispectrum}\label{ch:pattern}
\include{bispectrum}

\chapter{Image reconstruction}\label{part:image-reconstruction}
\include{inpainting}

\chapter{Applications}
\include{applications}

\include{appendix}

\backmatter
\include{glossary}

\bibliographystyle{spmpsci}
\nocite{Gourd1989}



\end{document}

%% file: preface.tex
%
%

\preface

  In his beautiful book \cite{PET1}, Jean Petitot proposes a sub-Riemannian model
for the primary visual cortex of mammals. This model is neurophysiologically
justified. Further developments of this theory lead to efficient algorithms
for image reconstruction, based upon the consideration of an associated
hypoelliptic diffusion. The sub-Riemannian model of Petitot and Citti-Sarti (or certain of its
improvements) is a left-invariant structure over the group $SE(2)$ of
rototranslations of the plane. Here, we propose a semi-discrete version of
this theory, leading to a left-invariant structure over the group $SE(2,N)$,
restricting to a finite number of rotations. This apparently very simple group
is in fact quite atypical: it is maximally almost periodic, which leads to
much simpler harmonic analysis compared to $SE(2).$ Based upon this
semi-discrete model, we improve on {\dario previous} image-reconstruction algorithms and we develop a pattern-recognition theory that leads also to very efficient algorithms in practice. 

\vspace{\baselineskip}

This research has  been supported by the European Research Council, ERC StG 2009 ``GeCoMethods'', contract n. 239748, by the iCODE institute (research project of the Idex Paris-Saclay), by the SMAI project ``BOUM'', and by the Grant ANR-15-CE40-0018 of the ANR. This research benefited from the support of the ``FMJH Program Gaspard Monge in optimization and operation research'' and from the support to this program from EDF.

This book contains, among others, results from \cite{G3,Remizov2013,ap_interp}. For their kind permission to reproduce parts of these papers we thank the Society for Industrial and Applied Mathematics and Springer-Verlag.

\vspace{\baselineskip}
\begin{flushright}\noindent
Paris, Toulon,\hfill {\it Dario, Prandi}\\
April 2017\hfill {\it Jean-Paul Gauthier}\\
\end{flushright}

%% file: introduction.tex
\section{Neurophysiological considerations and the Citti-Petitot-Sarti model}

The primary visual cortex V1 is a (not small) part of the brain, whose
location is shown on Figure~\ref{fig1}. It is responsible, after the retina, for
elementary representations of the visual field, by visual charts, that take
into account not only position, but also orientation.

\begin{figure}[b]%
\centering
\includegraphics[width=.4\textwidth]{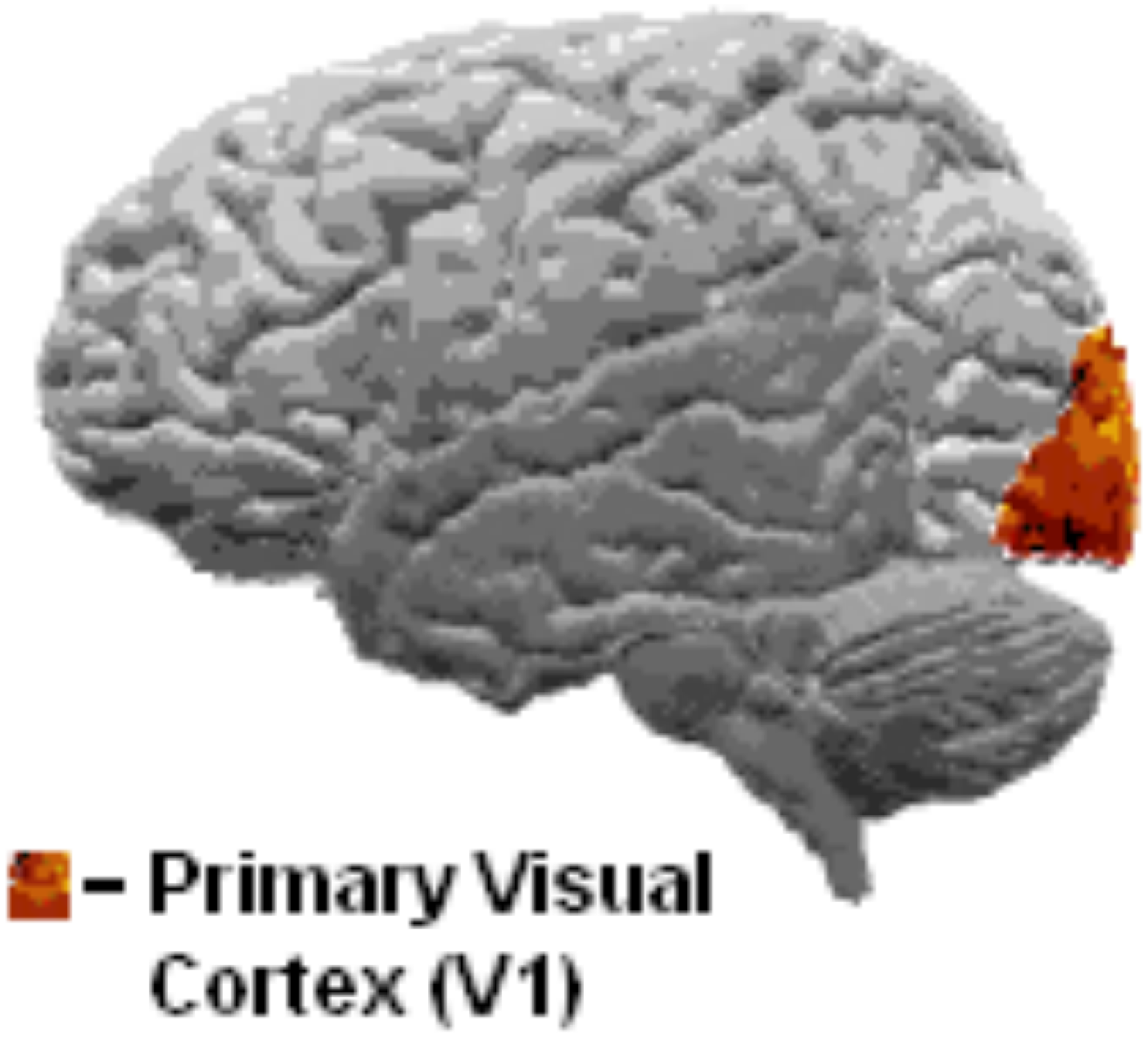}
\caption{The primary visual cortex}%
\label{fig1}%
\end{figure}

In the paper \cite{PET} and his beautiful book \cite{PET1}, Jean Petitot
describes a sub-Riemannian model of the visual cortex V1. The main idea goes
back to the paper by H\"{u}bel an Wiesel in 1959 (Nobel prize in 1981)
\cite{HUB} who showed that in the visual cortex V1, there are groups of
neurons that are sensitive to position and directions with connections between
them that are activated by the image. The key fact is that the system of
connections between neurons, which is called the functional architecture of
V1, preferentially connects neurons detecting alignments.

Roughly speaking, neurons of V1 are grouped into orientation columns, each of
them being sensitive to visual stimuli at a given point of the retina and for
a given direction on it. Orientation columns are themselves grouped into
hypercolumns, each of them being sensitive to stimuli at a given point with
any direction (see Figure~\ref{columns}).%

\begin{figure}[t]%
\centering
\includegraphics[width=.9\textwidth, bb=0 0 764 340]{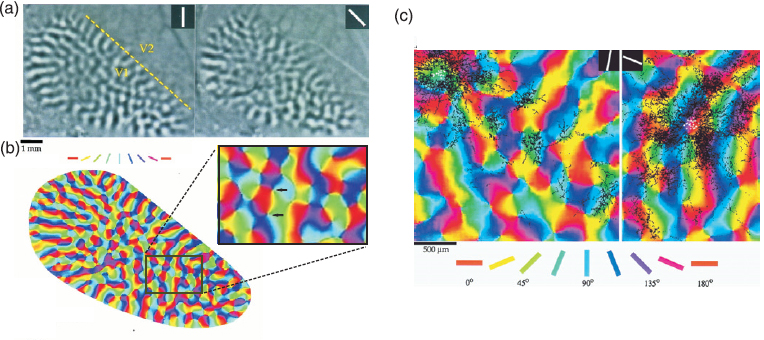}
\caption{Pinwheel structure of V1 (from \cite{KARL})}%
\label{fig2}%
\end{figure}

In the visual cortex there are two types of connections: the vertical
connections among orientation columns in the same hypercolumn, and the
horizontal connections among orientation columns belonging to different
hypercolumns and sensitive to the same orientation. For an orientation column
it is easy to activate another orientation column which is a \textquotedblleft
first neighbor\textquotedblright\ either by horizontal or by vertical
connections.
This is the so-called pinwheels structure of V1. (See Figures~\ref{fig2} and \ref{columns}.)
Pinwheels are the locations where multiple orientation columns converge.
Orientation columns are organized radially around a point known as a
singularity. As one can check on Figure~\ref{fig2}, there are both clockwise and
counterclockwise oriented pinwheels.

\begin{figure}%
\centering
\includegraphics[width=.8\textwidth]{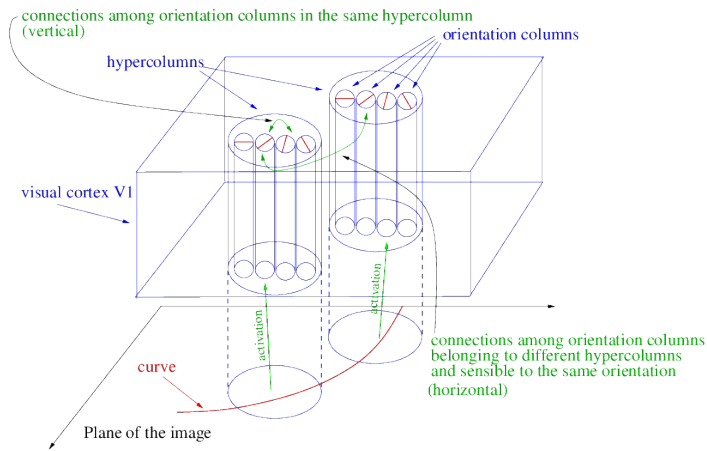}
\caption{Cells in the primary visual cortex}%
\label{columns}%
\end{figure}

From the mathematical point of view, it is thus assumed that V1 lifts the images $f(x,y)$ (i.e., functions of two position variables $x,y$ in the plane $\mathbb{R}^{2}$ of the image) to functions over the projective tangent bundle $PT\mathbb{R}^{2}.$
This bundle has as base $\mathbb{R}^{2}$ and as fiber over the point $(x,y)$
the set of directions of straight lines lying on the plane and passing through
$(x,y)$.

Consider for instance the simplest case in which the image is a smooth curve
$t\rightarrow(x(t),y(t))\in\mathbb{R}^{2}.$ Lifting this curve to
$PT\mathbb{R}^{2}$ means to add a new variable $\theta(t)$ that is the angle
of the vector $(\dot{x}(t),\dot{y}(t))$. Since we are obliged at some point to
go to certain stochastic considerations, it is convenient to write this lift
in the following \textquotedblleft control form\textquotedblright. We say that
$(x($\textperiodcentered$),y($\textperiodcentered$),\theta($%
\textperiodcentered$))$ is the lift of the curve $(x($\textperiodcentered
$),y($\textperiodcentered$))$ if there exist two functions $u($%
\textperiodcentered$)$ and $v($\textperiodcentered$)$ (called controls) such
that%
\begin{equation}
\left(
\begin{array}
[c]{c}%
\dot{x}(t)\\
\dot{y}(t)\\
\dot{\theta}(t)
\end{array}
\right)  =\left(
\begin{array}
[c]{c}%
\cos(\theta(t))\\
\sin(\theta(t))\\
0
\end{array}
\right)  u(t)+\left(
\begin{array}
[c]{c}%
0\\
0\\
1
\end{array}
\right)  v(t).\label{controlsys}%
\end{equation}
Here the control $u(t)$ plays the role of the modulus of the planar vector
$(\dot{x}(t),\dot{y}(t))$, but can take positive and negative values since the
angle $\theta(t)$ is defined modulo $\pi$.
The control $v(t)$ is just the derivative of $\theta(t)$.

\begin{remark}
\label{contact}. The vector distribution $N(x,y,\theta)$ $:=\spn\{F(x,y,\theta
),G(x,y,\theta)\}$, where $F(x,y,\theta)=\cos(\theta)\frac{\partial}{\partial
x}+\sin(\theta)\frac{\partial}{\partial y}$ and $G(x,y,\theta)=\frac{\partial
}{\partial\theta}$ is a vector distribution that endows $V1=PT\mathbb{R}^{2}%
$with the structure of a contact manifold. Indeed $N$ is completely
non-integrable (in the Frobenius sense) since $F$ and $G$ satisfy the
H\"{o}rmander condition: $\spn\{F,G,[F,G]\}=T_qPT\mathbb{R}^{2}$ for each $q\in
PT\mathbb{R}^{2}.$

Notice that the definition of vector field $F$ is not global over
$PT\mathbb{R}^{2}$ since it is not continuous at $\theta=\pi\sim0$: the
distribution $N$ $\ $is not trivializable and a correct definition of it would
require two charts. However, for sake of simplicity, we proceed with a single
chart with some abuse of notation. Notice, however, that\textbf{\ if we lift
the problem to the group SE(2) of rototranslations of the plane}, which is a
double covering of $PT\mathbb{R}^{2}$, the structure becomes trivializable and
the definition of $F$ becomes global. We do this often along the paper. We
will just underline places where the projectivization comes in and is important.
\end{remark}

In the model described by Petitot, when a curve is partially interrupted, it
is reconstructed by minimizing the energy necessary to activate the regions of
the visual cortex that are not excited by the image.

Since for an orientation column it is easy to activate another orientation column which is a \textquotedblleft first neighbor\textquotedblright\ either by
horizontal or by vertical connections, following Petitot, the energy necessary to activate a path $t\in
\lbrack0,T]\rightarrow(x(t),y(t),\theta(t))$ is given by
\begin{equation}%
\int_{0}^{t}\left(\dot{x}(t)^{2}+\dot{y}(t)^{2}+\frac{1}{\alpha}\dot{\theta}(t)^{2}\right)dt=
\int_{0}^{t}
\left(u(t)^{2}+\frac{1}{\alpha}v(t)^{2}\right)dt.\label{energy}%
\end{equation}
Here, the term $\dot{x}(t)^{2}+\dot{y}(t)^{2}$ is proportional to the
energy necessary to activate horizontal connections, while the term
$\dot{\theta}(t)^{2}$ is proportional to the energy necessary to activate
vertical connections. The parameter $\alpha>0$ is a relative weight.

To conclude, in V1, the problem of reconstructing a curve interrupted between
the boundary conditions $(x_{0},y_{0},\theta_{0})$ and $(x_{1},y_{1}%
,\theta_{1})$ becomes the optimal control
problem:%

\begin{align}
\left(
\begin{array}
[c]{c}%
\dot{x}(t)\\
\dot{y}(t)\\
\dot{\theta}(t)
\end{array}
\right)   & =\left(
\begin{array}
[c]{c}%
\cos(\theta(t))\\
\sin(\theta(t))\\
0
\end{array}
\right)  u(t)+\left(
\begin{array}
[c]{c}%
0\\
0\\
1
\end{array}
\right)  v(t)\label{optcont}\\
& =u(t)F(x,y,\theta)+v(t)G(x,y,\theta),\\%
\label{eq:energy}{\displaystyle\int\limits_{0}^{T}}
(u(t)^{2}+\frac{1}{\alpha}v(t)^{2})dt  & \rightarrow\min\\
(x(0),y(0),\theta(0))  & =(x_{0},y_{0},\theta_{0}),\text{ }(x_{1},y_{1}%
,\theta_{1})=(x(T),y(T),\theta(T)).
\end{align}

Finding the solution to this optimal control problem can be seen as the
problem of finding the minimizing geodesic for the sub-Riemannian structure
over $PT\mathbb{R}^{2}$ defined as follows: The distribution is $N$ and the
metric $g_{\alpha}$ over $N$ is the one obtained by claiming that the vector
fields $F$ and $\sqrt{\alpha}G$ form an orthonormal frame.

By construction this sub-Riemannian manifold is invariant under the action of
the group $SE(2)$ of rototranslations of the plane. Indeed, $\{(N
,g_{\alpha})\}_{\alpha\in]0,\infty\lbrack}$ are the only sub-Riemannian
structures over $PT\mathbb{R}^{2}$ which are invariant under the action of
$SE(2)$. See for instance \cite{AGR1}.

\begin{remark}
From the theoretical point of view, the weight parameter $\alpha$ is
irrelevant: for any $\alpha>0$ there exists a homothety of the $(x,y)$-plane
that maps geodesics of the metric with the weight parameter $\alpha$ to those
of the metric with $\alpha=1$. 
For this reason, in all theoretical considerations we fix $\alpha=1$. However
its role will be important in our image reconstruction algorithms.
\end{remark}

\begin{remark}
In the optimal control problem \eqref{optcont} the time $T$ should be fixed, but
changing $T$ changes only the parameterization of the solutions. For the same
reasons as in Riemannian geometry, minimizers of the sub-Riemannian energy \eqref{eq:energy} are the same as the minimizers of the
sub-Riemannian length 
\begin{equation}
\ell(x(\cdot),y(\cdot),\theta(\cdot))=\int_{0}^{T} \sqrt{(u(t)^{2}+\frac{1}{\alpha}v(t)^{2})}dt.
\end{equation}
\end{remark}

The history of this model goes back to the paper by Hoffman \cite{HOF} in
1989, who first proposed to regard the visual cortex as a manifold with a
contact structure. In 1998, Petitot \cite{PET, PET1} wrote the first version
of the model as a constrained minimization problem on the Heisenberg group and gave an enormous
impulse to the research on the subject. In 2006, Citti and Sarti \cite{CS}
required the invariance under rototranslations, and wrote the model on $SE(2)$, recognizing it as a sub-Riemannian structure and explicitly introducing the vector fields $F,G$.
In \cite{G1}, it was proposed to write the problems on $PT\mathbb{R}^{2}$ to
avoid some topological problems and to be more consistent with the fact that
the visual cortex V1 is sensitive only to directions (i.e., angles modulo
$\pi$) and not to directions with orientations (i.e., angles modulo $2\pi$).
The theory was wonderfully completed in Petitot's book \cite{PET1}.

The detailed study of geodesics was performed by Yuri Sachkov in a series of
papers \cite{SAC, SAC1}. For modifications of the model aimed to avoid the
presence of geodesics whose projection on the plane has cusps, see \cite{CS,
SG, BCR, BDR}. This model was also deeply studied by Duits \emph{et al.}\ in
\cite{Duits2014, DF}, with medical imaging applications in mind, and by Hladky
and Pauls \cite{HP}. Of course this model is closely related with the
celebrated model by Mumford \cite{MUM}.  See also \cite{August2003, Barbieri2015}.

The model described by Petitot was used to reconstruct smooth images by
Ardentov, Mashtakov and Sachkov \cite{AMS}. The technique developed by them
consists of reconstructing as minimizing geodesics the level sets of the image
where they are interrupted.

When applying the above strategy to reconstruct images with large corrupted parts, one is faced to the problem that it is not clear how to put in correspondence the non-corrupted parts of the same level set.
For this reason, in \cite{G1,DF}, was proposed the following method. In system \eqref{controlsys}, excite all possible admissible paths in a stochastic way, obtaining the SDE:
\begin{equation}
\left(
\begin{array}
[c]{c}%
dx_{t}\\
dy_{t}\\
d\theta_{t}%
\end{array}
\right)  =\left(
\begin{array}
[c]{c}%
\cos(\theta_{t})\\
\sin(\theta_{t})\\
0
\end{array}
\right)  du_{t}+\left(
\begin{array}
[c]{c}%
0\\
0\\
1
\end{array}
\right)  dv_{t},
\label{stoceq}%
\end{equation}
where $u_{t},v_{t}$ are two independent Wiener processes.
To this SDE is naturally associated a diffusion process (here we have fixed $\alpha$ $=1$):%
\begin{align}
\frac{\partial\Psi}{\partial t}  & =\frac{1}{2}\Delta\Psi,\text{
\ }\label{hypoel}\\
\Delta & =F^{2}+G^{2}=\left(\cos(\theta)\frac{\partial}{\partial x}+\sin
(\theta)\frac{\partial}{\partial y}\right)^{2}+\frac{\partial^{2}}{\partial
\theta^{2}}\nonumber
\end{align}
The operator $\Delta$ is not elliptic, but it is hypoelliptic (indeed it satisfies the H\"{o}rmander condition). By the Feynman--Kac formula, integrating Equation \ref{hypoel} with the corrupted image as the initial condition, one expects to reconstruct the most probable missing level curves (among admissible).

\begin{remark}
\label{unimodlap}Note that the operator $\Delta$ is the intrinsic
sub-Riemannian Laplace operator over the (unimodular) group $SE(2)$, as
defined in \cite{ABG}.
\end{remark}

To summarize, in this model, the process of reconstruction by V1 of corrupted
images is the following

\begin{itemize}
\item The plane image $f(x,y)$ is lifted to a certain \textquotedblleft
function\textquotedblright\ $f(x,y,\theta)$ on the bundle $PT\mathbb{R}^{2}$

\item The diffusion process \eqref{hypoel} with the initial condition
$\Psi_{|t=0}=f$ is integrated on the interval $t\in\lbrack0,T]$ for some $T>0$.

\item The resulting function $\bar{f}_{T}=\Psi_{|t=T}$ on the bundle
$PT\mathbb{R}^{2}$is projected down to a function $f_{T}(x,y)$, which
represents the reconstructed image.
\end{itemize}

The lifting procedure should be as follows: the image is assumed to be a
smooth function $f(x,y)$. Then, it can be naturally lifted to a surface $S$ in
$PT\mathbb{R}^{2}$ by lifting its level curves\footnote{It is widely accepted that the retina performs some smoothing, see \cite{MH}. After
such smoothing, the image $f(x,y)$ is in general a Morse function, and the lifted surface is a smooth surface. See \cite{G1}.}. At a point $(x,y,\theta)$ we
would like to set $f(x,y,\theta)=f(x,y)$ if $(x,y,\theta)\in S$ and
$f(x,y,\theta)=0$ elsewhere. But this would be nonsense since $S$ has zero
measure in $PT\mathbb{R}^{2}.$ Hence, $f(x,y)$ is lifted to a distribution
$f(x,y,\theta)$, supported in $S$, and weighted by $f(x,y)$. We refer to
\cite{G1} for details.

The idea of modeling the process of reconstruction of images in V1 as an
hypoelliptic diffusion was presented first in \cite{CS} and was implemented in
details with various modifications and different purposes by many authors. See, e.g., \cite{SG, G1, DvA, DF,Citti2016,Zweck2004} and references therein. A clever variant to the lifting process, which inspired the left-invariant lifts used in this work, was proposed in \cite{DF, DF1}.

\bigskip

It turns out that although good looking, Equation \eqref{hypoel} is not
easy to integrate numerically. In particular, the multiscale sub-Riemannian
effects are hidden inside. (The numerical literature for PDEs in
sub-Riemannian geometry appears to be very scarce.)

For the equation \eqref{hypoel} it is possible to compute the associated heat
kernel, see \cite{ABG, DvA, DF}. See also \cite{ZhangDuits2016} for a review on its numerical implementations. Moreover, the numerical integration starts to be a rather large problem, due to the number of points/angles in a reasonable image.

Some promising results about inpainting using hypoelliptic diffusion were obtained in \cite{G1}. In this paper a reasonable algorithm is presented. This algorithm is recalled here, with some improvements. See also \cite{Remizov2013}.

To end this section, we mention that the initial purpose of the aforementioned algorithms was the so-called ``modal'' reconstruction. (See Figure~\ref{fig:modal}.) That is, the reconstruction of contours that are actually perceived by the observer. However, the results obtained, especially those of \cite{Remizov2013}, show that these algorithms can be used also for ``amodal'' completions. This term designates the completion via a mental extrapolation of contours that are not visible by the observer. Since some of the results are uncanny with respect to what a human observer can do, this suggests that this model, although very efficient in practice, is not a completely realistic neurophysiological model.

\begin{figure}
  \begin{minipage}[b]{.45\linewidth}
    \centering\includegraphics[width=\textwidth, bb=0 0 1200 1280]{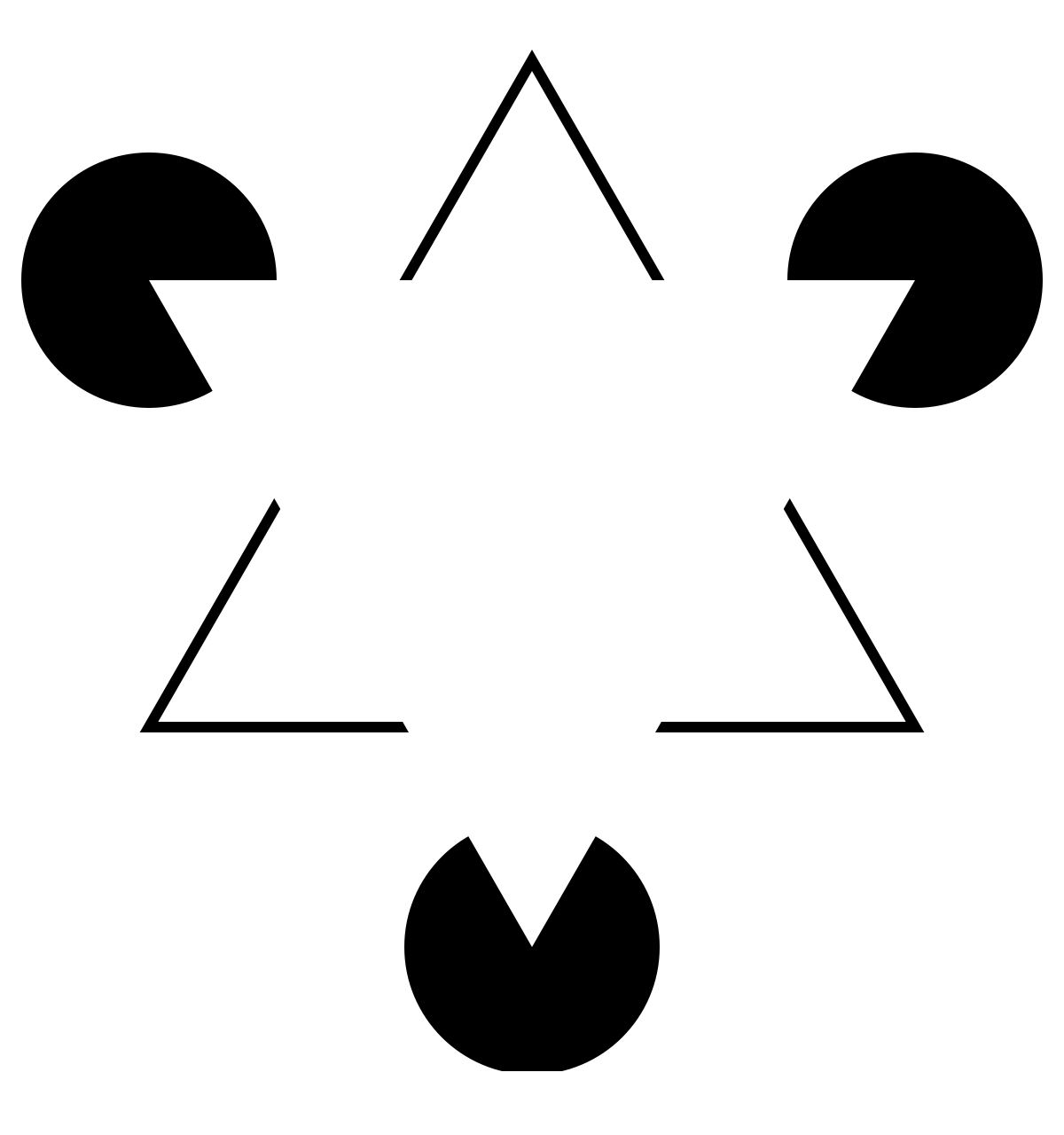}
    \subcaption{Modal completion (Kanisza triangle).}
  \end{minipage}
  \hfill
  \begin{minipage}[b]{.45\linewidth}
    \centering\includegraphics[width=\textwidth, bb=0 0 825 700]{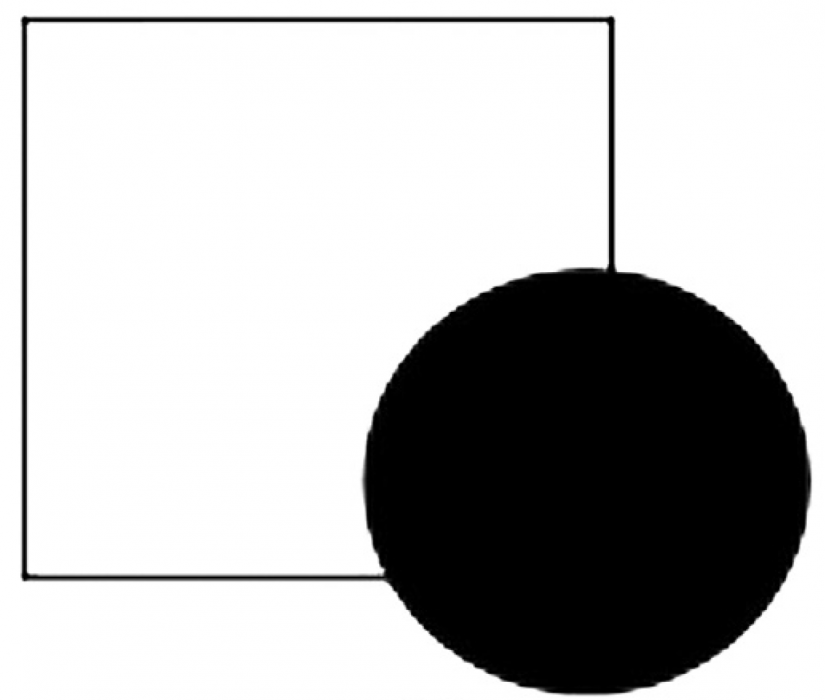}    
    \vspace{1em}
    \subcaption{Amodal completion.}
  \end{minipage}
  \caption{The two different neurophysiological types of contour completion.}
  \label{fig:modal}
\end{figure}

\section{Our semi-discrete model}

We shall use certain elementary and non elementary facts from representation
theory of locally compact groups.\ For general representation theory see the
book of Dixmier \cite{Dixmier1977} or \cite{Barut77theoryof, WAR, CK}. For a detailed treatment of
$SE(2)$ see \cite{VIL}. Almost periodic functions and Bohr compactification
will also be important here.\ See again \cite{Dixmier1977} and \cite{Weil, CGB}.
We also need standard results on duality theory for locally compact groups. Recall that compact groups are subject to Tannaka duality (an avatar of Pontryagin duality on abelian groups). Here, we will consider maximally almost-periodic (MAP) groups, which are subject to Chu duality, a generalization of
Tannaka duality. For Chu duality, see the foundational paper of Chu \cite{CHU} and
the book by H.\ Heyer \cite{HEY}.

In the Petitot model, we assume a finite number of columns in the
hypercolumns. Moreover, this number is assumed to be rather
small\footnote{Looking at the literature, discussing with Petitot and other
people, we could not get a clear answer about this number N.\ This estimate
(15 to 50) comes from our own numerical experiments in both image
reconstruction and pattern recognition.} (between 15 and 50).\ It corresponds
to a small number of directions on $PT\mathbb{R}^{2},$ and a small number $N$
of angles on the double-covering $SE(2).$ Then, we are naturally led to the
consideration of the (semi-discrete) subgroup of rigid motions, corresponding
to rotations with angles $\frac{2k\pi}{N}$ ($k$ integer). This group is
denoted by $SE(2,N).\ $It is the semi-direct product $SE(2,N)=\mathbb{Z}%
_{N}\ltimes\mathbb{R}^{2}$ of the finite abelian group $\mathbb{Z}%
_{N}=\mathbb{Z}/N\mathbb{Z}$ by $\mathbb{R}^{2}.$
It is a very interesting group that we will study in details, together with
some of its (finitely generated) subgroups.
 We also mention that this group has already been used in various image processing papers, e.g., \cite{Smach2008, CS, DuitsPhD}.

The semi-discrete group of Euclidean motions $SE(2,N)$ turns out to have quite a rich structure: although non-compact, it is a MAP group and, even more,
it is a Moore group, i.e. all its unitary irreducible representations are
finite dimensional. 
Since, just as $SE(2)$, it is a semi-direct product, these representations can
be computed by using Mackey's machinery (see \cite{mackey, Barut77theoryof}, but we shall recall
precisely what we need in the paper).

In the paper, we treat two questions that investigate the way V1
performs image reconstruction of corrupted images, and also (this is our
conjecture) the way V1 computes some important quantities that are used at a
higher level in the brain, to perform planar pattern recognition. For both
purposes the structure of $SE(2,N)$ is crucial.

We already explained how in the (continuous) Petitot model image
reconstruction is performed by lifting to $SE(2)$ and integrating a certain
hypoelliptic Laplace operator (see Remark \ref{unimodlap} above).

Usual Fourier transform diagonalizes the usual Laplace operator.\ Similarly,
the generalized Fourier transform \cite{Dixmier1977} over $SE(2)$ disintegrates our sub-Riemannian Laplace operator $\Delta$ into a continuous sum of Mathieu-type operators. At the level of $SE(2,N),$ the same scenario appear: A
semi-discrete Laplacian $\Delta_{N}$ comes in the picture, and it is again
disintegrated by Fourier transform over $SE(2,N)$ into a (continuous) sum of
finite dimensional (Mathieu-like) linear operators. In particular, the
corresponding heat kernel has a simple explicit expression.
From these considerations, we derived interesting numerical methods
for the problem of completion of images.

\section{Textures}\label{texture}

We are led to consider as models for textures, certain (sometimes finite
dimensional) subspaces \ $AP_{F}(\mathbb{R}^{2})$ of the space $B_2(\mathbb{R}%
^{2})$ of Besicovitch almost periodic functions (see \cite{CGB}).
Approximation and interpolation of usual images by elements of these
\textbf{texture spaces} (adapted to the structure of $SE(2,N)$) is an
interesting problem in itself, to which we provide a reasonable solution.

This representation is crucial in our work, since it is deeply used in
both problems treated herein, i.e image reconstruction and pattern recognition.
In fact, these very natural finite dimensional subspaces $AP_{F} (\mathbb{R}^{2})$ are (roto)translation invariant over $SE(2,N).$ Moreover,
the unitary irreducible representations of $SE(2,N)$ act on some of these spaces $AP_{F}(\mathbb{R}^{2})$. This leads to
very reasonable algorithms for both image reconstruction and pattern recognition.
In particular, over these spaces, \textbf{the semi-discrete diffusion becomes
just a linear ordinary differential equation}.

Moreover, we remark that such representation is of interest in itself due to its close relation with the Fast Fourier Transform algorithm, with the problem of the "polar" Fourier transform \cite{polarFT}, and more generally with the non-uniform FFT, about
which there is an important amount of literature. We give just a small
non-exhaustive list of references: \cite{BM, BWO, DR, LN, NL, OS, PI}. This is
in particular an important question in the fields of NMR and computed tomography.

\section{Triple convolution and bispectrum}

For real valued functions $f(x)$ over $\mathbb{R}$, the auto-correlation and
the triple correlation (or triple convolution) are
\begin{align*}
A(s)  & =\int_{\mathbb{R}}f(x)f(x+s)dx\\
T(s_{1},s_{2})  & =\int_{\mathbb{R\times R}}f(x)f(x+s_{1})f(x+s_{2})dx.
\end{align*}
Denoting by $\hat{f}(\lambda)$ the usual Fourier transform of $f$, it is
easily computed that the Fourier transform \ $\hat{A}(s)$ is just the power
spectrum $\hat{A}(s)=\hat{f}(\lambda)\hat{f}(\lambda)^{\ast}.$ (Here * stands
for conjugate).

Also, the Fourier transform (over $\mathbb{R\times R)}$ $\hat{T}(\lambda
_{1},\lambda_{2})$ of the triple correlation $T(s_{1},s_{2})$ is what is
called the \textbf{Bispectrum} of $f$:%
\[
\hat{T}(\lambda_{1},\lambda_{2})=\hat{f}(\lambda_{1})\hat{f}(\lambda_{2}%
)\hat{f}(\lambda_{1}+\lambda_{2})^{\ast}.
\]
The bispectrum (or equivalently the triple convolution) is
translation-invariant, and is used for long in many areas of signal processing
(\cite{HMM, GAM, KAK, JOH}). It was also used for texture discrimination of
music instruments (see \cite{Dubnov1997}) an it is suspected that the auditive cortex proceeds with bispectra.

These invariants already appeared in pattern recognition. They are
alternatively called Fourier descriptors (see \cite{G3, Smach2008, G5}). In these
papers, a natural abstract generalization on locally compact groups is
proposed. For $\bG$ a locally compact group and $f\in L^{2}(\bG)$, with $dg$
being the Haar measure, the bispectrum of $f$ is the operator valued map:
\begin{equation}\label{eq:bisp-intro}
\BS_{f}(\lambda_{1},\lambda_{2})=\hat{f}(\lambda_{1})\otimes\hat{f}(\lambda
_{2})\circ\hat{f}(\lambda_{1}\otimes\lambda_{2})^{\ast},%
\end{equation}
where \ $\lambda_{1},\lambda_{2}$ are unitary irreducible representations of
$\bG,$ and * denotes the adjoint operator. Bispectra are clearly invariant w.r.t
the action of translations of $\bG$. \textbf{The main fact is that they are
highly discriminating between functions up to translations}.

Let us say that a set of translation-invariants is \textbf{weakly
complete} if it discriminates between functions modulo the action of
translations, over a residual subset of $L^{2}(\bG).$ It turns out that, \textbf{when $\bG$ is abelian, compact (separable), or MAP, the bispectra are weakly complete.} This is the main fact,
and the residual subset of functions over which bispectra discriminate is just
the set of $f$ such that \textbf{the Fourier transform $\hat{f}$ is an invertible operator}.

Hence, the idea is very simple: to discriminate between images \textbf{on the
plane, }i.e. functions $f(x,y),$ let us, as for image reconstruction, lift the
functions $f(x,y)$ to functions $Lf(x,y,\theta)$ over $SE(2,N),$ and compute
the bispectral invariants of $Lf.$ 
(We mention that different pipelines are possible, see, e.g., \cite{Sifre2013}.)
After this step, we would like to feed a learning machine (such as an SVM
machine \cite{VAP}), or a deep learning machine, with the lifted bispectra
$\BS_{Lf}.$ We expect to get very good performances for pattern recognition
\cite{Smach2008}.

Our (plausible?) expectation is that the primary visual cortex feeds higher
level strata of the brain with such invariants.

Unfortunately, a very bad feature appears: if we require (which is natural)
the lift to be left-invariant, the Fourier transform $\hat{L}f$ of $Lf$ is
never invertible (on the contrary, it has always rank at most
one). Therefore, we do not know whether bispectral invariants are weakly
complete over $L^{2}(\mathbb{R}^{2})$ or not.

A long piece of this monograph has for purpose to overcome this
difficulty. Roughly speaking, we need the extra ingredient of "centering" the
images (which can be done with respect to the gravity center of the image for
$L^{2}(\mathbb{R}^{2})$ images, but which is not so easy in the case of almost
periodic functions). The purpose of this centering procedure is to eliminate
the effect of translations.
After this centering, we define quite naturally another set of invariants,
richer than the bispectrum, that we call \textbf{rotational bispectral
invariants}. They make sense for both cases of images in $L^{2}(\mathbb{R}%
^{2}),$ and for our texture spaces $AP_{F}(\mathbb{R}^{2}).$ As a last step,
we are able to derive weak completeness results.


\section{Organization of the paper}

The second chapter introduces all the technical tools and concepts we need,
i.e. general facts about locally compact groups, the non-commutative Fourier
transform and general Plancherel's theorem, Chu duality. We introduce our
centering operators, we recall the main facts of Mackey's imprimitivity
theory, together with the induction-reduction theorem for the decomposition of
tensor products of induced representations. We define weakly cyclic functions,
and we discuss a (more or less standard) version of the abstract wavelet
transform, that will be useful in our lifting process.

We define Bohr compactification, Bohr and Besicovitch almost periodic
functions and our relevant subspaces $AP_{F}(G)$ of almost periodic functions.
We discuss the process of centering almost periodic functions.

\bigskip

Chapter 3 discusses the lifting problem, mainly from $\mathbb{R}^{2}$ to
$SE(2)$ or $SE(2,N)$. In fact, we do it in a slightly more general
context.\ We show first that continuous left invariant lifts are essentially
wavelet transforms, and we characterize injectivity of the lifts.
We then show that Fourier transforms of (left-invariant) lifted functions have
always rank one. To finish, we define almost left-invariant lifts and cyclic
lifts, that will be enough for our discrimination purposes.

\bigskip

Chapter 4 deals with almost periodic interpolation and approximation.\ As
we said in the introduction (Section \ref{texture}), this chapter is
interesting in itself.\ It is strongly related with the problems of
non-uniform FFT and polar FFT.\ We define a generalized Fourier-Bessel
operator, and we prove a factorization theorem for this operator.\ This
factorization theorem is the key point for the algorithms of evaluation,
interpolation, approximation of functions in our texture spaces $AP_{F}(G).$

\bigskip

Chapter 5, Pattern Recognition, is the heart of the paper. Bispectral
invariants are defined in general by formula \eqref{eq:bisp-intro}. We prove the
main discrimination result over abelian, compact and Moore groups.
Next, we treat the problem of discriminating lifted functions. First, we define cyclic lifts and provide a proof of weak
completeness in that case. The remaining of the chapter studies our
rotational bispectral invariants in view of weak completeness.

\bigskip

In Chapter six, we focus on the image reconstruction problem, recalling
previous results and providing some improvements in the case of $SE(2,N).$ We
construct the heat operator in the case of $SE(2,N)$ via elementary stochastic
considerations and we recall the expressions of the heat Kernels in both cases
of $SE(2)$ and $SE(2,N)$. We present the basic completion algorithm, and we
show how it works in the case of our spaces $AP_{F}(\bG)$ of almost periodic
functions, in which case it just relies on integrating linear ordinary
differential equations (a finite number of them, provided that $F$ is finite). 

\bigskip

Chapter 7 is about applications.\ First we care about image reconstruction,
and we show how our basic algorithm can be substantially improved on by
certain natural heuristic considerations. We show some very convincing
reconstruction results.
  However these results provide very little improvement w.r.t.\ the state-of-the-art, which actually consists of extremely efficient algorithms. (See, e.g., \cite{Facciolo,cao}.)
  Therefore, our main contribution here is not over this practical area.
  In fact we just ``validate'' the Citti-Petitot-Sarti model and our semi-discrete improvement of it. Moreover, we reduce the diffusion to ODE's which is conceptual gain only.

Second, we care about pattern recognition. 
  On the contrary, in this area we get very interesting practical results, with several advantages w.r.t.\ some other standard methods. 
The basic idea, as we said, is to feed a learning machine with our bispectral invariants. These invariants have a number of good features for discrimination: they are continuous invariants (which is absolutely necessary) and they rely on the consideration of a set $F$ of basic frequencies (over the frequency plane).
Depending on the application, it is quite an easy routine to select
properly this set $F$, that in some situations has a clear frequency interpretation.
Moreover, this approach allows rather easily to pass to 3D\ pattern
recognition. Roughly speaking, it is enough to feed the learning step with a
number of pictures of the object under consideration, taken under several
distinct points of view.

To test our invariants, we chose to use the SVM learning machine by V.\ Vapnik \cite{VAP}, and some of its improvements. We remark, however, that we could have chosen a more fashionable deep learning machine. 
On the web, one can find easily a series of test data, results and procedures
in order to perform comparisons with other methods. We present some of these
comparisons, mostly from our papers \cite{Smach2008, G3}.

  To finish, we insist on the fact that the main interest (from the image processing point of view) of this semi-discrete model (i.e., the lift of images to functions over $SE(2,N)$) is not image completion but pattern recognition.

%% file: preliminaries.tex
This chapter introduces the concepts that are the main subject of the rest of this work, along with the essential tools that are needed. After a brief introduction on harmonic analysis in non-commutative groups, we introduce the general setting considered in this work, alongside with some essential facts on its representation theory. Afterwards, we recall some basic notions on almost-periodic functions and and a precise construction that allows to select some relevant subspaces. Finally we present our models for natural images (compactly supported functions of $L^2(\bR^2)$) and textures (properly selected finite-dimensional subspaces of almost-periodic functions in the plane).

\section{Prerequisites} 
\label{sec:prerequisites}

In the following, we briefly recall some well-known fact in commutative and non-commutative harmonic analysis. In particular, we introduce the Fourier transform and the Plancherel theorems that will be the basis of our work.

\subsection{Conventions}\label{sec:conventions}

Scalar products on complex-valued vectors or functions of an Hilbert space $\cH$ are always assumed to be linear in the \emph{second} variable. According to this convention, the tensor product of $u,v\in\cH$ is the linear operator 
\begin{equation}
	u\otimes v(w) = u\langle \bar v, w\rangle, \qquad \forall w\in\cH.
\end{equation}
That is, $(u\otimes v)_{i,j} = u_i v_j$.

\subsection{Harmonic analysis on locally compact abelian groups} 
\label{sec:foruier_transform_on_locally_compact_abelian_groups_and_pontryiagin_duality}

Let $\bG$ be a locally compact abelian group with additive notation.
A \emph{character} of $\bG$ is a continuous group homomorphism $\lambda:\bG \to \bC$ such that $|\lambda(a)|=1$ for any $a\in\bG$.
Defining the product of two characters as the point-wise multiplication and the inverse as the complex conjugation, the set
\begin{equation}
  \hat\bG = \{\lambda \mid \lambda \text{ is a character of } \bG\},
\end{equation}
endowed with the topology of uniform convergence on compact sets, is a locally compact abelian group, called the \emph{(Pontryagin) dual group} of $\bG$.

It is straightforward to check that $\Omega:\bG \to \widehat{\widehat{\bG}}$ defined by $\Omega_x(\lambda) \coloneqq \lambda(x)$, is a continuous group homomorphism. In particular,  $\bG\subset \widehat{\widehat\bG}$.

\begin{theorem}[Pontryagin duality]
  \label{thm:pontryagin}
  The map $\Omega$ is a group isomorphism, and thus $\bG$ is canonically isomorphic to the dual of ${\widehat\bG}$.
\end{theorem}

The Fourier transform allows to carry the above isomorphism to the level of complex-valued functions defined on $\bG$ and $\widehat\bG$.
Namely, endow $\bG$ with its Haar measure and for any $f\in L^2(\bG)\cap L^1(\bG)$ define its Fourier transform $\hat f\in L^2(\widehat\bG)$ by
\begin{equation}
  \hat f(\lambda) \coloneqq \int_\bG f(x) \bar\lambda(x) \, dx.
\end{equation}
Observe, in particular, that letting $\avg f = \int_\bG f(x)\,dx$ it holds $\avg f = \hat f(\hat o)$, where $\hat o (\cdot) \equiv 1$ is the identity of $\widehat \bG$.
We have the following.

\begin{theorem}[Plancherel Theorem]
  \label{thm:plancherel}
  There exists a unique measure $d\lambda$ on $\widehat \bG$, called \emph{Plancherel measure}, such that the above defined Fourier transform can be extended to an isometry $\cF: L^2(\bG) \to L^2(\widehat\bG)$.
  In particular, whenever $f\in L^2(\bG)\cap L^1(\bG)$ and $\hat f\in L^2(\widehat\bG)\cap L^1(\widehat\bG)$, it holds that
  \begin{equation}
    \cF^{-1} (\hat f)(x) = \int_{\widehat\bG} \hat f(\lambda) \lambda(x)\, d\lambda.
  \end{equation}
\end{theorem}

\begin{remark}
  When $\bG =\bR$ the above procedure yields the classical Fourier transform. 
  Indeed, the Haar measure of $\bR$ is the Lebesgue measure, $\widehat \bR\cong\bR$ can be realized as the set of $x\mapsto e^{2\pi i \lambda x}$ for $\lambda\in\bR$, and the Plancherel measure is the normalized Lebesgue measure.
\end{remark}
  
The left regular representation of $\bG$ is the map $x\mapsto \tau_x\in \cU(L^2(\bG))$ defined as $\tau_x f(y) = f(y-x)$. Then, the fundamental property of the Fourier transform, at least for our purposes, is the following.

\begin{theorem}
  For any $f,g\in L^2(\bG)$ and any $x\in\bG$ it holds that
  \begin{equation}
    f = \tau_x g \iff \hat f(\lambda) = \bar \lambda(x)\, \hat g(\lambda) \quad\forall \lambda\in\widehat\bG.
  \end{equation}
\end{theorem}


\subsection{Fourier transform on locally compact non-commutative groups}\label{sec:fourier-non-commutative}

Let $\bG$ be a locally compact unimodular group, not necessarily abelian.
A \emph{unitary representation} $T$ of $\bG$ is a continuous\footnote{With respect to the strong topology of $\cU(\cH_T)$. Recall that this is not the norm topology, w.r.t.\ which irreducible representations are not in general continuous.} homomorphism $T:\bG\to\cU(\cH_T)$, where $\cH_T$ is a complex (possibly infinite dimensional) Hilbert space.
A representation $T$ is \emph{irreducible} if no nontrivial closed subspace of $\cH_T$ is invariant for all $T(a)$, $a\in\bG$.
Two representations $T$ and $T'$ are \emph{equivalent} if there exists a linear invertible operator $A:\cH_T\to \cH_{T'}$ such that $A\circ T = T'\circ A$.
In this case we write $T\cong T'$.

The \emph{dual set of $\bG$} is the set $\widehat \bG$ of all equivalence classes of unitary irreducible representations of $\bG$. Although, for $\bG$ abelian, the only irreducible representations are the characters, and this set coincides with the Pontryagin dual, in the general case it has no group structure.
The Fourier transform of a function $f\in L^2(\bG)\cap L^1(\bG)$ is then defined by
\begin{equation}
  \label{eq:non-comm-ft}
  \hat f(T) = \int_\bG f(a)\,T(a)^{-1}\,da, \qquad \forall T\in\widehat\bG.
\end{equation}
Observe that $\hat f(T)$ is a Hilbert-Schmidt operator on $\cH_T$.

\begin{remark}
  The same formula can be used to define the values of $\hat f$ on not necessarily irreducible unitary representation of $\bG$.
\end{remark}

We have the following generalization of Theorem~\ref{thm:plancherel}.

\begin{theorem}[Unimodular non-commutative Plancherel Theorem]\label{thm:plancherel}
  Let $\bG$ be a locally compact separable unimodular group.
  Then, there exists a (unique) Plancherel measure $\hat \mu_\bG$ on $\widehat \bG$ such that the above definition can be extended to an isometry $\cF: L^2(\bG) \to L^2(\widehat \bG,\hat\mu_\bG)$.
  In particular, the following inversion formula holds
  \begin{equation}
    f(a) = \int_{\widehat\bG} \tr\left(\hat f(T) \circ T(a) \right)\,d\hat \mu_\bG(T).
  \end{equation}
\end{theorem}

As in the abelian case, the Fourier transform has a nice behavior w.r.t.\ to the action of the left regular representation $\Lambda:\bG\mapsto \cU(L^2(\bG))$, defined by $\Lambda(a)f(b) \coloneqq f(a^{-1}b)$, as shown in the following.

\begin{theorem}[Fundamental property w.r.t.\ the action of the left regular representation]
  \label{thm:FT-fund-prop}
  For any $f,g\in L^2(\bG)$ and any $a\in\bG$ it holds
  \begin{equation}
    f = \Lambda(a)g \iff \hat f(T) = \hat g(T) \circ T^{-1}(a) \quad\forall T\in\widehat\bG.
  \end{equation}
\end{theorem}

\subsection{Chu Duality}\label{sec:chu}

Chu duality is an extension of the dualities of Pontryagin (see Theorem~\ref{thm:pontryagin}) and Tannaka (for compact groups) to a class of more general groups. 
In particular, it applies to \emph{Moore groups}, i.e., those groups whose unitary irreducible representations are all finite dimensional.
Here, the difficulty is to find a suitable notion of bi-dual, carrying a group structure. 
See \cite{Heyer1973}.

Let $\repr_n(\bG)$ denote the set of continuous unitary representations of $\bG$ over $\bC^n$.
Taking as a basis of neighborhoods at $T\in\repr_n(\bG)$ the sets
\begin{equation}
  W(T,K,\varepsilon) \coloneqq \left\{ \rho\in\repr_n(\bG) \mid \|T(a)-\rho(a)\|_{\text{HS}} \le \varepsilon \quad \forall a\in K \right\},
\end{equation}
for $\eps>0$ and $K\subset\bG$ compact,
the set $\repr_n(\bG)$ is a topological space which turns out to be locally compact since $\bG$ is so.
The \emph{Chu dual} of $\bG$ is the topological sum
  \begin{equation}
    \repr(\bG) \coloneqq \bigcup_{n\ge 1} \repr_n(\bG).
  \end{equation}

A \emph{quasi-representation} of $\bG$ is a continuous map $Q$ from $\repr(\bG)$ to $\bigcup_{n\ge 1} \cU(\bC^n)$ such that for any $T\in \repr_{n(T)}(\bG)$, $T'\in\repr_{n(T')}(\bG)$, and $U\in\cU(\bC^{n(T)})$ it holds
  \begin{enumerate}
    \item $Q(T)\in\cU(\bC^{n(T)})$;
    \item $Q(T\oplus T') = Q(T)\oplus Q(T')$;
    \item $Q(T\otimes T') = Q(T)\otimes Q(T')$;
    \item $Q(U\circ T\circ U^{-1}) = U\circ Q(T)\circ U^{-1}$;
  \end{enumerate}

The set of quasi-representations of $\bG$ is denoted by $\repr(\bG)^{\vee}$ and is called the \emph{Chu quasi-dual}.
Setting $E(T)\coloneqq\idty_{n(T)}$ and $Q^{-1}(T)=Q(T^{-1})$, the Chu quasi-dual is an Hausdorff topological group with identity $E$.
Finally, we can define the continuous group homomorphism $\Omega:\bG\mapsto \repr(\bG)^{\vee}$ as
\begin{equation}
  \Omega_a(T) \coloneqq T(a).
\end{equation}

\begin{definition}
  A locally compact group $\bG$ has the \emph{Chu duality property} if $\Omega$ is a topological group isomorphism.
\end{definition}

The main result is then the following.

\begin{theorem}
  \label{thm:tannaka}
  Any Moore group has the Chu duality property.
\end{theorem}

Observe that, since all abelian and/or compact groups are Moore, Chu duality contains both Pontryagin and Tannaka duality.


\section{General setting} 
\label{sec:repr-semidirect}

We now present the general setting that we will consider for most of this work, that of certain semi-direct product groups. We also recall some well-known facts regarding the corresponding representation theory.

Let us consider the semi-direct product $\bG = \bK\ltimes \bH$,  obtained thanks to the action $k\in\bK\mapsto \phi(k)\in\operatorfont{Aut}(\bH)$ . We will always assume the following:
\begin{itemize}
  \item $\bH$ is an abelian separable connected locally compact group.
  \item $\bK$ is an abelian finite group of cardinality $N$.
  \item The restriction of the action $k\mapsto \phi(k)$ to $\bH\setminus\{o\}$ is free.
  \item The Haar measure of $\bH$ is invariant under the $\phi(k)$'s.
\end{itemize}

The above assumptions guarantee that $\bG$ is unimodular \cite[Ch. II, Prop. 28]{Nachbin1965}.
Note that $\bG$ is also automatically post-liminal.
Later on we will explicitly compute the unitary irreducible representations of $\bG$, which will be finite dimensional, thus proving that $\bG$ is a Moore group.

\begin{remark}
  The freeness assumption on the action of $\bK$ could probably be removed. 
  However, this would yield to a more complicated description of the representations of $\bG$ and it is outside the scope of this work, whose main motivation is $SE(2,N)=\bZ_N\ltimes \bR^2$.
\end{remark}

Additive notation is used for $\bH$ and multiplicative one for $\bK$.
We denote the identity of $\bH$ by $o$ and that of $\bK$ by $e$.
The letters $x,y,z$ are reserved for elements of $\bH$, while $k,h,\ell,\alpha,\beta$ are elements of $\bK$.
Elements of the Pontryagin duals $\widehat\bH$ and $\widehat\bK$ are denoted, respectively, as $\lambda,\mu\in\widehat\bH$ and $\hat k, \hat h,\ldots\in\widehat\bK$.
The identities of the Pontryagin duals are $\hat o\in\widehat\bH$ and $\hat e\in\widehat\bK$.
Elements of $\bG$ are denoted either by $a,b\in\bG$ or as couples $(k,x),(h,y)$.

The action of $\bK$ on $\bH$ induces a contragredient action of $\bK$ on $\widehat \bH$, still denoted $k\mapsto \phi(k)$ and defined by $\phi(k)\lambda(x) = \lambda(\phi(k^{-1})x)$.
The left regular representations of $\bH$ and $\bK$ are called \emph{translation} and \emph{shift} operators and denoted by $x\mapsto \tau_x\in\cU(L^2(\bH))$ and $k\mapsto S(k)\in\cU(L^2(\bK))$, respectively.
Their actions on $f\in L^2(\bH)$ and $v\in L^2(\bK)\simeq\bC^N$ are given by
\begin{equation}
  \tau_x f(y) \coloneqq f(y-x)\quad\text{ and }\quad S(k).v(h) = v(k^{-1}h).
\end{equation}
When $\bK$ is cyclic, i.e.\ $\bK\simeq\bZ_N$, the shift operator is completely determined by $S=S(1)$ via $S(k) = S^k = S\circ \cdots \circ S$, $k$ times.

The left regular representation of $\bG$ is denoted by $\Lambda:\bG\to \cU(L^2(\bG))$, and its action on $f\in L^2(\bG)$ is $\Lambda(a)f(b) = f(a^{-1}b)$.
Exploiting the semi-direct product structure of $\bG$, we can consider the quasi-regular representation of $\bG$, denoted by $\pi:\bG\to \cU(L^2(\bH) )$ and whose action on $f\in L^2(\bH)$ is 
\begin{equation}
	\pi(k,x) f(y) = f((k,x)^{-1}y) = f(\phi(k^{-1})(y-x)),\qquad \forall (k,x)\in \bG,\, y\in \bH.
\end{equation}
The quasi-regular representation is far from being irreducible, see, e.g., \cite{Fuhr2002a}. Indeed, to any measurable $\bK$-invariant $U\subset\widehat\bH$, is associated the following closed invariant subspace 
\begin{equation}
  \label{eq:invariant-subspace}
  \cA = \cA_U = \{ f\in L^2(\bH) \mid \supp \hat f \subset U\}.
\end{equation}
By conjugating the quasi-regular representation with the Fourier transform on $\bH$ we obtain the representation $\hat\pi$ on $L^2(\widehat\bH)$. Since $\phi(k)\bar\lambda=\overline{\phi(k)\lambda}$, $\hat\pi$ is given by
\begin{equation}\label{eq:quasi-reg-fourier}
  \hat\pi(x,k)\hat f(\lambda)= \cF({\pi(x,k)f})(\lambda) = \bar\lambda(x)\hat f(\phi(k^{-1})\lambda).
\end{equation}

Throughout the paper we will be interested in quotienting out the effect of the action of $\bH$, or one of its subsets, on $L^2_{\bR}(\bH)$. 

\begin{definition}
  \label{def:centering}
  Let $\cA\subset L^2(\bH)$ be invariant under the action of $\pi$ and let $U\subset \bH$.
  A \emph{centering of $\cA$ w.r.t.\ $U$} is an operator $\Phi:\cA\to \cA$ that acts by $\Phi(f)=\tau_{c(f)}f$, where $c:\cA\to \bH$ is such that, for any $f,g\in\cA$,
  \begin{equation}
    \Phi(f)=\Phi(g) \iff \exists x\in U \text{ s.t.\ } f = \tau_x g.
  \end{equation}
\end{definition}

Observe that the above implies that for any $k\in\bK$ it holds
\begin{equation}
  \Phi(f)=\phi(k)\Phi(g) \iff \exists x\in U \text{ s.t.\ } f = \pi(x,k) g.
\end{equation}
It is then clear that $\Phi(\cA)\subset \cA$ is invariant under the action of $\bK$.

\subsection{Representation theory}

A complete description of the unitary irreducible representations of $\bG$ can be obtained via Mackey machinery, see e.g., \cite[Ch. 17.1, Theorems 4 and 5]{Barut77theoryof}.
We recall it in the following.

\begin{theorem}[Representations of semidirect products]
  \label{thm:repr-semidir}
  To any $\hat k\in\widehat\bK$ corresponds the unitary representation of $\bG$ defined by  $T^{\hat o\times \hat k}=\hat k$ and acting on $\bC$.
  On the other hand, to any $\lambda\in\widehat\bH\setminus\{\hat o\}$ corresponds the unitary representation $T^\lambda$ acting on $L^2(\bK)$ and defined by
  \begin{equation}
    T^\lambda(x,k) = \diag_h (\phi(h)\lambda(x))\, S(k).
  \end{equation}

  Moreover, the dual set $\widehat\bG$ is the union of the set of the nontrivial orbits in $\widehat \bH$ under the action of $\bK$ and of $\{\hat o\}\times \widehat\bK$.
  Indeed, for any $\ell\in\bK$ it holds that $T^{\phi(\ell)\lambda}\circ S(\ell) = S(\ell)\circ T^\lambda$ and hence $T^{\lambda_1}$ is equivalent to $T^{\lambda_2}$ whenever $\lambda_1,\lambda_2$ belong to the same orbit.
  Finally, the Plancherel measure $\hat\mu_\bG$ is supported outside of $\{\hat o\}\times \widehat\bK$.
\end{theorem}

\begin{remark}
  When $\bG = SE(2,N)$, the set of nontrivial orbits can be identified with the ``slice of Camembert'' $\cS\subset\widehat{\bR^2}\cong \bR^2$, defined by 
  \begin{equation}
  	\cS = \{\xi e^{i\omega}\mid \xi\in\bR_+,\, \omega\in[0,2\pi/N)\}.
  \end{equation}
\end{remark}

Let $\{e_k\}_{k\in\bK}$ be the canonical basis of $L^2(\bK)$ given by
\begin{equation}
	e_k(h) = 
	\begin{cases}
		1 & \text{ if } k=h,\\
		0 & \text{otherwise.}
	\end{cases}
\end{equation}
Then, the coefficients of $T^\lambda(k,x)$ w.r.t.\ this basis are
\begin{equation}\label{eq:T_components}
	T^\lambda(k,x)_{i,j}= \langle T^\lambda(k,x)e_j,e_i\rangle = \phi(kj)\lambda(x) \delta_{i,kj},
\end{equation}
where $\delta_{i,j}$ denotes the Kronecker's delta.
In Chapter~\ref{ch:ap_interp} we will explicitly compute the coefficients of $T^\lambda$ with respect to the dual basis $\widehat \bK$ and we will relate these to Bessel functions.


\begin{proposition}
  \label{prop:FT-semidirect}
  Let $f\in L^2(\bG)$.
  Then, for any $\lambda\in\widehat \bH\setminus\{\hat o\}$, the components of $\hat f(T^\lambda)$ w.r.t.\ the canonical basis $\{e_k\}_k\in\bK$ of $L^2(\bK)$ are
  \begin{equation}
  	\hat f(T^\lambda)_{i,j} = \cF(f(i^{-1}j,\cdot))(\phi(j)\lambda).
  \end{equation}
  Moreover, for any $f\in L^1(\bG)\cap L^2(\bG)$ and for any $\hat k\in\widehat\bK$ it holds
  \begin{equation}
    \hat f(T^{\hat o\times \hat k}) = \widehat{\avg_f}(\hat k).
  \end{equation}
  Here, we let $\avg_f\in L^2(\bK)$ be defined as 
	\begin{equation}
		\avg_f(k)\coloneqq \avg f(\cdot,k) := \int_\bH f(x,k)\,dx.
	\end{equation}
\end{proposition}

\begin{proof}
	Since $(k,x)^{-1}=(k^{-1},-\phi(k^{-1})x)$, the first statement follows by \eqref{eq:T_components}:
	\begin{equation}
		\begin{split}
			\hat f(T^\lambda)_{i,j} 
			&= \sum_{k\in\bK} \int_\bH f(k,x)\,T^\lambda(k,x)^{-1}_{i,j}\,dx \\
			&= \sum_{k\in\bK} \cF(f(k,\cdot))(\phi(j)\lambda)\delta_{i,k^{-1}j}\\
			&= \cF(f(i^{-1}j,\cdot))(\phi(j)\lambda).
		\end{split}
	\end{equation}
    On the other hand, to prove the second statement it suffices to compute
    \begin{equation}
      \hat f(T^{\hat o\times \hat k}) 
      = \sum_{\ell\in\bK} \int_{\bH} f(x,\ell) \hat k(-\ell) \,dx  
      = \sum_{\ell\in\bK} \avg_f(\ell) \,\overline{\hat k}(\ell)  = \widehat{\avg_f}(\hat k).
    \end{equation}
  \end{proof}

\subsubsection{Induction-Reduction theorem}

Throughout the paper, we will use a well-known fact on tensor product representations: the Induction-Reduction Theorem.
(See \cite{Barut77theoryof}.)
This theorem allows to decompose the tensor products of representations $T^{\lambda_1}\otimes T^{\lambda_2}$, acting on $\bC^N\otimes \bC^N\cong \bC^{N\times N}$, to an equivalent representation acting on $\bigoplus_{h\in\bK} \bC^N$, which is a block-diagonal operator whose block elements are of the form $T^{\lambda_1+R_k\lambda_2}$.
  In particular, the Induction-Reduction Theorem plays the role of the Klebsch-Gordan decomposition in the non-compact case.
    We use it several times in our technical computations, Section 3.2 in particular, and in the proof of our main results.
  

\begin{theorem}[Induction-Reduction Theorem]
  \label{thm:ind-reduction}
  For any $\lambda_1,\lambda_2\in \bH\setminus\{\hat o\}$ it holds 
  \begin{equation}
    T^{\lambda_1} \otimes T^{\lambda_2} \cong \bigoplus_{h\in\bK} T^{\lambda_1+\phi(h)\lambda_2}.
  \end{equation}
  In particular, the unitary equivalence $A:L^2(\bK\times\bK)\to \bigoplus_{h\in\bK}L^2(\bK)$ is given by
  \begin{equation}
  	A = \bigoplus_{h\in\bK} A_h, \qquad A_h = \cP\circ\big(\idty\otimes S(h^{-1})\big),
  \end{equation}
  where $\cP:L^2(\bK\times\bK)\to L^2(\bK)$ is the operator $P\varphi(\ell)= \varphi(\ell,\ell)$, $\varphi\in L^2(\bK\times\bK)$.
\end{theorem}

\begin{proof}
  Let $\{e_i\}_{i\in\bK}$ be the canonical basis of $L^2(\bK)$.
	It is clear that to prove the theorem it suffices to show that, for all $(k,x)\in\bG$ and $i,j\in\bK$, it holds
	\begin{equation}
		(T^{\lambda_1}(k,x)\otimes T^{\lambda_2}(k,x)).e_i\otimes e_j = \left(A^*\circ \bigoplus_{h\in\bK}T^{\lambda_1+\phi(h)\lambda_2}\circ A\right). e_i\otimes e_j.
	\end{equation}
	By \eqref{eq:T_components}, the right hand-side computes to
	\begin{equation}\label{eq:ind-red-1}
		(T^{\lambda_1}(k,x)\otimes T^{\lambda_2}(k,x)).e_i\otimes e_j = \phi(ik)\left[\lambda_1+\phi(i^{-1}j)\lambda_2\right](x)\,e_{ik}\otimes e_{jk}.
	\end{equation}
	On the other hand, for all $h\in\bK$, it holds
	\begin{equation}
		(T^{\lambda_1+\phi(h)\lambda_2}\circ A_h). e_i\otimes e_j = \delta_{i,h^{-1}j} \, \phi(ik)\left[\lambda_1+\phi(h)\lambda_2\right](x) e_{ik}.
	\end{equation}
	Since a simple computation shows that $A^*((v_h)_{h\in\bK})(r,s) = v_{r^{-1}s}(r)$, this yields
	\begin{equation}
		\left(A^*\circ \bigoplus_{h\in\bK}T^{\lambda_1+\phi(h)\lambda_2}\circ A\right). e_i\otimes e_j(r,s) = \delta_{i,s^{-1}rj} \, \phi(ik)\left[\lambda_1+\phi(r^{-1}s)\lambda_2\right](x) e_{ik}.
	\end{equation}
	Finally, it is easy to check that the above coincides with \eqref{eq:ind-red-1} for all  $r,s\in\bK$.
\end{proof}


The action of linear operators $\cB:\bigoplus_{h\in\bK}L^2(\bK)\to \bigoplus_{h\in\bK}L^2(\bK)$, can be block-decomposed as
\begin{equation}
	(\cB \psi)_h = \sum_{\ell\in\bK} \cB_{h,\ell} \psi_\ell \qquad\forall \psi=\left(\psi_h\right)_{h\in\bK}\in\bigoplus_{h\in\bK}L^2(\bK).
\end{equation}
To be precise, $\cB_{h,\ell} = p_h\circ \cB\circ p^*_\ell$, where $p_h:\bigoplus_{h\in\bK}L^2(\bK)\to L^2(\bK)$ is the projection on the on the $h$-th component and $p_h^*$ its adjoint.
%
Direct computations yield the following.

\begin{proposition}
  \label{prop:equivalence-prop}
  Let $A$ be the equivalence in Theorem~\ref{thm:ind-reduction}.
  Then, the following holds
  \begin{itemize}
    \item For any linear operator $\cT: L^2(\bK\times\bK)\to L^2(\bK\times\bK)$ with components $\cT=(\cT_{i,j,r,s})_{i,j,r,s}$, the operator $A\circ \cT \circ A^*$ has $k,\ell$ block component:
    \begin{equation}
      \begin{split}
      	(A\circ\cT\circ A^*)_{h,l} 
      	&= P\circ (\idty\otimes S(h^{-1}))\circ\cT\circ(\idty\otimes S(\ell))\circ P^* \\
      	&= (\cT_{i,ih,j,j\ell})_{i,j\in\bK}.
      \end{split}
    \end{equation}
    In particular, for a couple of linear operators $B,C: L^2(\bK)\to L^2(\bK)$ it holds
    \begin{equation}\label{eq:equiv_tensor}
      (A\circ (B\otimes C) \circ A^*)_{h,l} = P\circ(B\otimes S(h^{-1}) C S(\ell))\circ P^* = (B_{i,j}C_{ik,j\ell})_{i,j\in\bK}.
    \end{equation}

    \item Let $\tilde S(k):\bigoplus_{k\in\bK} L^2(\bK)\to \bigoplus_{k\in\bK} L^2(\bK)$ be defined by $\tilde S(k)(\psi_h)_h = (\psi_{k^{-1}h})_h$ for any $\psi=(\psi_h)_{h\in\bK}$ and $k\in\bK$.
    Then, 
    \begin{equation}
      A\circ S(h)\otimes S(\ell)\circ A^* = \tilde S(\ell^{-1}h) \circ \bigoplus_{k\in\bK} S(h).
    \end{equation}
  \end{itemize}
\end{proposition}


\subsection{Weakly cyclic functions} 
\label{sub:weakly_cyclic_functions}

In this section we present the space of weakly-cyclic functions. These are functions of $L^2(\bH)$ whose Fourier transform, evaluated on a.e.\ orbit in $\widehat \bH$ w.r.t.\ $\bK$ yields a well-behaved element of $L^2(\bK)$. Since all our results in pattern recognition of Chapter~\ref{ch:pattern} will apply only to this class of functions, we will later show, in Section~\ref{sec:compactly_real_value}, that the generic image is indeed represented by a weakly-cyclic function.

  A vector $v\in L^2(\bK)$ is \emph{cyclic} if $\{S(k) v\}_{k\in\bK}$ is a basis for $L^2(\bK)$.
  If $\bK$ is cyclic and finite with $N$ elements, recalling that $S(k)=S^k$ for $S=S(e)$, this is equivalent to the following circulant operator (see Appendix~\ref{app:circulant}) being invertible
  \begin{equation}
    \Circ v = \left(
    v,\, Sv,\, \ldots,\, S^{N-1}v
    \right).
  \end{equation}

  For $f\in L^2(\bH)$ we denote by $\hat f\in L^2(\widehat\bH)$ the abelian Fourier transform on $\bH$.
  From the action $\phi$ of $\bK$ on $\widehat\bH$ we obtain a contragredient action on $L^2(\widehat\bH)$, still denoted by $\phi$, letting $\phi(k)\hat f(\lambda)=\hat f(\phi(k^{-1})\lambda)$.
  Finally, for $\lambda\in\widehat\bH$ we let the vector $\hat f_\lambda\in L^2(\bK)$ be the evaluation at $\lambda$ of the (inverse) $\bK$-orbit of $\hat f$, that is
  \begin{equation}
    \label{eq:omegaPsi}
    \hat f_\lambda(k) = \phi(k^{-1})\hat f( \lambda) \qquad\forall k\in\bK.
  \end{equation}
  Observe that $S(k)\hat f_\lambda = \hat f_{\phi(k^{-1})\lambda}$, since $S(k)\hat f_\lambda(h) = \phi(k).\phi(h^{-1})\hat f(\lambda)$.

  Since $\hat f_o = (\hat f(o),\ldots, \hat f(o))$ the vector $\hat f_\lambda$ cannot be cyclic for every $\lambda\in\widehat \bH$, thus motivating the following definition.

  \begin{definition}
    \label{def:cc-weakly-cyclic}
    A function $f\in L^2(\bH)$ is \emph{weakly cyclic} if $\hat f_\lambda$ is cyclic for a.e.\ $\lambda\in\widehat{\bH}$.
    We denote by $\cC\subset L^2(\bH)$ the set of weakly cyclic functions.
  \end{definition}

\subsubsection{Real valued functions} 
\label{ssub:real_valued_functions}

  Our arguments in the following are heavily based on exploiting the weak-cyclicity property.
    However, we now show that, if $\bK$ satisfy an ``evenness'' condition, no real-valued function can be weakly-cyclic. The rest of this section is then devoted to define the concept of $\bR$-weak-cyclicity, which will be exploited in Section~\ref{sub:real_valued_functions} to circumvent this problem.

  \begin{definition}
    \label{def:even}
    The action of $\bK$ on $\bH$ is \emph{even} if there exists $k_0\in\bK$ such that $\phi(k_0) = -\idty$.
  \end{definition}

  A necessary condition for the action to be even is that $k_0=k_0^{-1}$.
  The example to keep in mind is that of the natural action of $\bZ_N$ on $\bR^2$ when $N$ is even, in which case $k_0=N/2$.

  \begin{proposition}
    \label{prop:omega-real}
    Let $\bK$ be acting evenly on $\bH$ and define the following proper $\bR$-linear subspace of $L^2(\bK)$
    \begin{equation}
      \cX = \left\{ v\in L^2(\bK) \mid v(h) = \overline{v({k_0h})}\quad \forall h\in\bK \right\}.
    \end{equation}  
    Then, $\hat f_\lambda\in\cX$ for any $f\in L^2_\bR(\bH)$ and any $\lambda\in\widehat{\bH}$.
    In particular, $\hat f_\lambda$ is never cyclic.
  \end{proposition}

  \begin{proof}
    From the evenness of the action, it follows that $\phi(h) \lambda = -\phi(k_0h)\lambda$ for any $\lambda\in\widehat{\bH}$ and $h\in\bK$.
    Using that $\hat f(\lambda)=\overline{\hat f(-\lambda)}$, the statement follows from 
    \begin{equation}
      \hat f_\lambda(h) = \hat f(\phi(h)\lambda) = {\hat f(-\phi(k_0h)\lambda)} = \overline{\hat f(\phi(k_0h)\lambda)} = \overline{\hat f_\lambda}(k_0h), \quad\forall h\in\bK. 
    \end{equation}
  \end{proof}

  Observe that $\cX$ is invariant under the action of the shift operator.
  We then say that $w\in \cX$ is \emph{$\bR$-cyclic} if $\spn\{S(k)w\}_{k\in\bK} = \cX$, and pose the following.

  \begin{definition}
    \label{def:cc-weakly-cyclic-real}
    If the action of $\bK$ on $\bH$ is even, a real valued function $f\in L^2_{\bR}(\bH)$ is \emph{weakly $\bR$-cyclic} if $\hat f_\lambda$ is $\bR$-cyclic for a.e.\ $\lambda\in\widehat\bH$.
    On the other hand, if $\bK$ is not acting evenly on $\bH$, $f\in L^2_\bR(\bH)$ is weakly $\bR$-cyclic if it is weakly cyclic in the sense of Definition~\ref{def:cc-weakly-cyclic}.

    We denote by $\cC_\bR\subset L^2_\bR(\bH)$ the set of weakly $\bR$-cyclic functions.
  \end{definition}




  Let $\bK$ be acting evenly on $\bH$. Consider $V=\{e,k_0\}\triangleleft \bK$ and fix any section $\sigma:\bK/V\to \bK$.
  Define the map $B:L^2(\bK/V)\to \cX$ as
  \begin{equation}
    B w(h) = 
    \begin{cases}
      {w([h])} & \text{ if } h\in\range\sigma, \\
      \overline{w([h])} & \text{ otherwise}. \\
    \end{cases}
  \end{equation}
  Obviously $B$ is invertible and $\bR$-linear, and thus it endows $L^2(\bK/V)$ with the structure of a real vector space.
  When $\bK=\bZ_N$ with $N$ even, the natural choice for $\sigma$ leads to identify $L^2(\bK/V)$ with the first $N/2$ components of vectors in $L^2(\bK)\cong \bC^N$.

  Since $\cX$ is invariant under the shifts, $S_\bR(\ell) = B^{-1}\circ S(\ell) \circ B$ is a representation of $\bK$ acting on $L^2(\bK/V)$.
  Its action can be described explicitly:
  \begin{equation}
    S_\bR(\ell) v([h]) = 
    \begin{cases}
      {v([\ell^{-1}h])} & \text{ if } \ell^{-1}\sigma([h]) \in\range\sigma \\
      \overline{v([\ell^{-1}h])} & \text{ otherwise.} \\
    \end{cases} 
  \end{equation}
  It is then immediate to see that $v\in \cX$ is $\bR$-cyclic if and only if $\spn\{S_\bR(\ell) B^{-1}v\}_{\ell\in\bK/V}=L^2(\bK/V)$. 
  We can then translate the $\bR$-cyclicity property to vectors of $L^2(\bK/V)$.
  In particular, when $\bK$ is cyclic with $N$ elements, $\bR$-cyclicity of $w\in L^2(\bK/V)\cong\bC^{N/2}$ is equivalent to the invertibility of the following ``even circulant'' matrix
  \begin{equation}
    \Circ_\bR w = \left(w, S_\bR w, \ldots , S^{N/2-1}_\bR w\right).
  \end{equation}



\subsection{Wavelet transform}
  \label{sec:wavelet}

  We now introduce the concept of (continuous) wavelet transform. 
  In Chapter~\ref{ch:lift}, we will observe that, under some reasonable assumptions, the operators lifting functions from $L^2(\bH)$ to $L^2(\bG)$ can always be seen as wavelet transforms associated with the quasi-regular representation $\pi$.
  The following results are well-known and can be found, e.g., in \cite{Fuhr2002a}. (See also \cite{Fuhr2005}.)
  
  Let $T$ be a strongly continuous unitary representation of a locally compact group $\bG$ on the Hilbert space $\cH = \cH_T$.
  Given a vector $\Psi\in\cH$, the wavelet transform of $\varphi\in\cH$ w.r.t.\ the \emph{wavelet} $\Psi$ is the linear bounded operator $W_\Psi:\cH \to C_b(\bG)$ defined by
  \begin{equation}
    \label{eq:wavelet-transform}
    W_\Psi \varphi (a) = \langle T(a)\Psi, \varphi \rangle_{\cH} \qquad\forall a\in\bG.
  \end{equation}
%
  We call $\Psi$ \emph{admissible} if $W_\Psi$ is an isometry from $\cH$ onto $L^2(\bG)$ and \emph{weakly admissible} if $W_\Psi$ is a bounded one-to-one mapping onto $L^2(\bG)$.

  If $\bG=\bK\ltimes \bH$ is a semi-direct product, it is natural to consider the wavelet transform w.r.t.\ the quasi-regular representation $\pi$, acting on $L^2(\bH)$. Recall that $\Psi^*(x):= \bar\Psi(-x)$.
  Straightforward computations then show that $W_\Psi\varphi(k,\cdot) = \varphi\star(\phi(k)\Psi^*)$, where $\star$ denotes the usual convolution product in $L^2(\bG)$.
  We will need the following observation,
  \begin{equation}
    \label{eq:FT-wavelet}
    \cF({W_\Psi \varphi(k,\cdot)})(\lambda) =  \phi(k)\widehat {\Psi^*}(\lambda)\,\hat \varphi(\lambda), \qquad \forall\lambda\in\widehat\bH.
  \end{equation}
  As an immediate consequence, see \cite{Fuhr2002a}, we have that
  \begin{equation}
    \label{eq:wavelet-transform-norm}
    \|W_\Psi \varphi\|^2_{L^2(\bG)} = \int_{\widehat\bH}|\hat \varphi(\lambda)|^2 \|\widehat{\Psi^*}_\lambda\|_{L^2(\bK)}\,d\lambda.
  \end{equation}
  Here, $\widehat{\Psi^*}_\lambda$ is the vector defined in \eqref{eq:omegaPsi}.
%
  We then have the following.

  \begin{theorem}
    \label{thm:admissible-vect}
    Let $\Psi\in L^2(\bH)$.
    Then,
    \begin{itemize}
      \item $\Psi$ is weakly admissible $\iff$  $\lambda\mapsto\|\widehat{\Psi^*}_\lambda\|_{L^2(\bK)}$ is strictly positive and belongs to $L^\infty(\widehat\bH)$;
      \item $\Psi$ is  admissible $\iff$ it holds $\|\widehat{\Psi^*}_\lambda\|_{L^2(\bK)} = 1$ for a.e.\ $\lambda\in\widehat\bH$. 
    \end{itemize}
  \end{theorem}

  In the sequel, we will frequently need to know the non-commutative Fourier transform of a wavelet transform $W_\Psi\varphi$, which is given in the following. 
  
  \begin{proposition}\label{prop:ft-lift}
    Let $\Psi\in L^2(\bH)$. Then, it holds
    \begin{equation}
      \widehat{W_\Psi \varphi}(T^\lambda) = \widehat{\Psi^*}_\lambda\otimes \hat \varphi_\lambda \in \HS(L^2(\bK)), \qquad \forall \varphi\in L^2(\bH), \,\lambda\in\widehat\bH.
    \end{equation}
    Moreover, if $\Psi\in L^1(\bH)\cap L^2(\bH)$ it holds
    \begin{equation}
      \widehat{W_\Psi \varphi}(T^{\hat o\times \hat k}) = 
      \begin{cases}
      \avg(\varphi)\,\avg(\bar\Psi) &\qquad\text{if } \hat k = \hat e,\\
      0 & \qquad \text{otherwise.}
      \end{cases}\qquad\forall \varphi\in L^2(\bH)\cap L^1(\bH),\, \hat k\in\widehat\bK.
    \end{equation}
  \end{proposition}
  
  \begin{proof}
    By Proposition~\ref{prop:FT-semidirect} and \eqref{eq:FT-wavelet}, we have
    \begin{equation}
      \begin{split}
       	\widehat{W_\Psi \varphi}(T^{\lambda})_{i,j} 
       	&= \cF\left(W_\Psi \varphi(i^{-1}j,\cdot)\right)(\phi(j)\lambda) \\
    	&= \phi(i^{-1})\widehat{\Psi^*}(\lambda)\,\hat \varphi_\lambda(\phi(j)\lambda)\\
    	&= \widehat{\Psi^*}_\lambda(i)\,\hat \varphi_\lambda(j)\\
    	&= \left(\widehat{\Psi^*}_\lambda\otimes \hat \varphi_\lambda\right)_{i,j}\\
      \end{split}
    \end{equation}
    This completes the proof of the first part of the statement.

    On the other hand, observe that $\avg_{W_\Psi \varphi}\equiv \avg \varphi \avg\Psi^*$.
    Indeed,
    \begin{equation}
      \int_{\bH} \pi(k,x)\bar\Psi(y) \,dx = \avg\Psi^* \qquad\forall k\in\bK,\, y\in\bH.
    \end{equation}
    Thus, $\widehat{\avg_{W_\Psi\varphi}}(\hat e)=\avg \varphi \avg\bar\Psi$ and $\widehat{\avg_{W_\Psi\varphi}}(\hat k)=0$ for $\hat k\neq\hat e$, and the second part of the statement follows from Proposition~\ref{prop:FT-semidirect}.
  \end{proof}
  
  As a consequence of the Induction-Reduction Theorem we also obtain this result.
  
  \begin{corollary}
    \label{cor:ft-lift-tensor}
    Let $\Psi\in L^2(\bH)$
    Then, for any $\lambda_1,\lambda_2\in\widehat \bH\setminus\{\hat o\}$ and any $\varphi\in L^2(\bH)$, we have 
    \begin{equation}
      A\circ\widehat{W_\Psi\varphi}(T^{\lambda_1}\otimes T^{\lambda_2})\circ A^{*} = 
      \widehat{\Psi^*}_{\lambda_1+\phi(k)\lambda_2} \otimes \hat \varphi_{\lambda_1+\phi(k)\lambda_2}
    \end{equation}
    Here, $A$ is the equivalence from $L^2(\bK\times\bK)$ to $\bigoplus_{k\in\bK}L^2(\bK)$ given by the Induction-Reduction Theorem. (See Theorem~\ref{thm:ind-reduction}.)
  \end{corollary}

  \begin{proof}
    The statement follows from the Induction-Reduction Theorem, the fact that  the Fourier transform commutes with equivalences and direct sums, and Proposition~\ref{prop:ft-lift}.
  \end{proof}

\section{Almost periodic functions and MAP groups}
\label{sssec:MAP-and-AP}\label{sec:MAP-and-AP}

  As we already mentioned, we consider almost-periodic functions as a model for textures. In this section we introduce both Bohr and Besicovitch almost-periodic functions from the group-theoretic point of view, i.e., as the pull-back of certain functional spaces on the Bohr compactification of $\bG$. We then proceed to introduce a reasonable concept of action of $\bG$ on Besicovitch almost-periodic functions, under which we will show in Chapter~\ref{ch:pattern} how to discriminate. Finally, in Section~\ref{sub:finite_dimensional_subspaces_of_almost_periodic_functions}, we introduce the finite-dimensional spaces of almost-periodic functions which will be the main object of interest in Chapter~\ref{ch:ap_interp}.
For the results in this section we refer to \cite[Ch. 16]{Dixmier1977}.

\begin{definition}
  The \emph{Bohr compactification} of a topological group $\bG$ is the universal object $(\bG^\flat, \sigma)$ in the category of diagrams $\sigma':\bG\mapsto \bK$ where $\sigma'$ is a continuous homomorphism from $\bG$ to a compact group $\bK$.
\end{definition}

When $\bG$ is abelian, one can construct $\bG^\flat$ in the following way: 
Let $\widehat\bG_d$ be the Pontryagin dual $\widehat\bG$ endowed with the discrete topology.
Then, its dual is a compact abelian group and it holds $\bG^\flat=\widehat{\widehat\bG_d}$.
Moreover, $\sigma$ is the continuous homomorphism whose dual  $\hat\sigma: \widehat\bG_d \to \widehat\bG$ is the identity map.

As a consequence of the definition, $\sigma(\bG)$ is dense in $\bG^\flat$ and $\bG^\flat=\bG$ whenever $\bG$ is compact.
In any case, $\sigma$ induces a bijection between the finite-dimensional continuous unitary representations of $\bG$ and those of $\bG^\flat$. 

\begin{definition}
  If the map $\sigma:\bG\to \bG^\flat$ is injective, the group $\bG$ is said to be \emph{maximally almost periodic} (MAP).
\end{definition}

The group $\bG$ is MAP if and only if the continuous finite-dimensional unitary representations of $\bG$ separate the points. 
A connected locally compact group is MAP if and only if it is the direct product of a compact group by $\bR^n$.
In particular, the Euclidean group of rototranslations $SE(2)$ is \emph{not} MAP.
Indeed, letting $\bG_o$ be the connected component of the identity $o$, a locally compact group $\bG$ such that $\bG/\bG_o$ is compact is MAP if and only if it is the semi-direct product of a compact subgroup $\bK$ and of a normal subgroup $\bH\cong\bR^n$ such that every element of $\bH$ commutes with the component of $\bK$ containing the identity \cite[16.5.3]{Dixmier1977}.
We will only be interested in MAP groups satisfying this property, as  $SE(2,N)$.

\begin{definition}\label{def:bohr-ap}
  The set $AP(\bG)$ of \emph{Bohr almost-periodic functions} over $\bG$ is the pull-back through $\sigma$ of the continuous functions over $\bG^\flat$.
  On the other hand, the set $B_2(\bG)$ of \emph{Besicovitch almost-periodic functions} over $\bG$ is the pull-back through $\sigma$ of $L^2(\bG^\flat)$.
\end{definition}

It can be shown that Bohr almost-periodic functions are exactly the uniform limits over $\bG$ of linear combinations of coefficients of finite-dimensional unitary representations of $\bG$.
In particular, when $\bG$ is abelian, this amounts to say that $f \in AP(\bG)$ if and only if it is the uniform limit of characters, that is
\begin{equation}
  f(x) \sim \sum_{\lambda\in \cK_f} a_f(\lambda) \lambda(x),
\end{equation}
where $\cK_f\subset \widehat{\bG}$ is a countable set.
If $\bG$ is MAP then $AP(\bG)$ is dense in $C(\bG)$, in the topology of uniform convergence over compact subset.

On the other hand, $f\in B_2(\bG)$ if and only if it can be written as a square integrable linear combination of coefficients of finite-dimensional unitary representations of $\bG$.
In particular, if the finite-dimensional unitary irreducible representations of $\bG$ are uncountable, $B_2(\bG)$ is a non-separable space.
In the abelian case, this amounts to say that $f\in B_2(\bG)$ if and only if
\begin{equation}
  \label{eq:B2}
  f(x) = \sum_{\lambda\in \cK_f} a_f(\lambda) \lambda(x) \quad\text{ s.t.\ } \sum_{\lambda\in \cK_f} |a_f(\lambda)|^2<+\infty.
\end{equation}

We now show the connection between $AP(\bG)$ and $B_2(\bG)$ in a more explicit way. 
Consider the set $C_b(\bG)$ of continuous bounded functions over $\bG$ endowed with the supremum norm.
It can be shown that $f\in C_b(\bG)$ is in $AP(\bG)$ if and only if $\{\Lambda(a) f\}_{a\in \bG}$ is a relatively compact subset of $C_b(\bG)$.
In this case, the convex hull $K$ of $\{\Lambda(a) f\}_{a\in \bG}$ in $C_b(\bG)$ contains exactly one constant function, whose value is called the mean value of $f$ and is denoted by $M(f)$.
In the case $\bG=\bR$, it holds 
\begin{equation}
  \label{eq:meanB2}
  M(f) = \lim_{T\rightarrow+\infty} \frac 1 {2T} \int_{-T}^T f(x) \,dx.
\end{equation}

Let $f\in AP(\bG)$ and denote by $f'$ the function of $C(\bG^\flat)$ such that $f = f'\circ\sigma$.
Then, it holds that $M(f)=\int_{\bG^\flat} f'(s) \, ds$, where the integration is taken w.r.t.\ the Haar measure of total mass equal to $1$ on $\bG^\flat$.
Endowing $AP(\bG)$ with the sesquilinear form $(f|g)\coloneqq M(f \bar g)=\langle f',g'\rangle$ we obtain a pre-Hilbert space which is canonically isomorphic to $C(\bG^\flat)$ regarded as a subspace of $L^2(\bG^\flat)$.
Since continuous functions over compact spaces are dense in $L^2$, the closure of $AP(\bG)$ w.r.t.\ the induced norm is then $B_2(\bG)$.

The above shows, in particular, that the pull-back $\sigma^*:L^2(\bG^\flat)\to B_2(\bG)$ is indeed an isomorphism of Hilbert spaces.
Thus, in the abelian case, characterization \eqref{eq:B2} is an immediate consequence of $L^2(\widehat {\bG^\flat}) = L^2(\widehat\bG_d)$.

\begin{remark}
  Observe that many non-zero functions over $\bG$ are the pull-back of a.e.\ zero functions in $\bG^\flat$.
  In particular it can be proved, and it is a trivial consequence of \eqref{eq:meanB2} in the case $\bG=\bR$, that for any $f\in C_c(\bG)$ any function $f':\bG^\flat\to \bC$ such that $f=f'\circ \sigma$ has to be zero a.e. on $\bG^\flat$.
  Due to this fact, functions in $B_2(\bG)$ represent indeed equivalence classes of functions $\bG\to\bC$.
\end{remark}

Let $f\in B_2(\bG)$ be expressed as in \eqref{eq:B2}, then the following Parseval equality holds
\begin{equation}
  (f|f) = \sum_{\lambda\in \cK_f} |a_f(\lambda)|^2.
\end{equation}
As a consequence, the usual diagonalization of the convolution takes place w.r.t.\ the scalar product $(\cdot|\cdot)$:
\begin{equation}
  f\star_{AP} g (x)= \sum_{\lambda\in \cK_f\cap \cK_g} a_f(\lambda)a_g(\lambda) e^{2\pi i\langle \lambda,x\rangle}.
\end{equation}

To conclude the section, let us consider the case of $\bG = \bK\ltimes\bH$ under the hypotheses introduced at the beginning of the paper.
Then, $\bG$ is a MAP group and $\bG^\flat = \bK\ltimes \bH^\flat$, where the action of $\bK$ on $\bH^\flat$ is obtained through the injection $\sigma_\bH:\bH\to \bH^\flat$ (see \cite{Deleeuw}).
Observe that functions $f\in B_2(\bG)$ are exactly those such that $f(k,\cdot)\in B_2(\bH)$ for any $k\in\bK$.
We will denote by $\sigma_\bG$ the injection of $\bG$ in $\bG^\flat$.
With abuse of notation, since $\sigma_\bG(x,k)=(k,\sigma_\bH(x))$, we will omit the subscript when no confusion arises.

Since whenever $\bH$ is non-compact we have that $B_2(\bH)\cap L^2(\bH)=\{0\}$, the quasi-regular representation is of no use to distinguish the action of $\bG$ on $B_2(\bH)$.
Indeed, to define the action of $\bG$ of $B_2(\bH)$ we need to work on the Bohr compactified, as follows.
Let $\pi^\flat$ be the quasi-regular representation of $\bG^\flat$ in $L^2(\bH^\flat)$.
Since the pull-back $\sigma^*:L^2(\bH^\flat) \to B_2(\bH)$ is an isomorphism, we can consider the representation $\pi_\sharp^\flat$ of $\bG^\flat$ on $B_2(\bH)$.
Finally, we pose the following.

\begin{definition}
  \label{def:B2-quasi-reg}
  The \emph{$B_2(\bH)$-quasi regular representation} of $\bG$ is $\pi_{B_2}=\pi_\sharp^\flat\circ \sigma_{\bG}$.
\end{definition}

Since it can be shown that $\pi_{B_2}(k,x) f(y) = f(\phi(k^{-1}(y-x))$ for any $f\in B_2(\bH)$, $(k,x)\in\bG$, and $y\in\bH$, the $B_2(\bH)$-quasi regular representation is the correct way to consider the action of $\bG$ on Besicovitch almost-periodic functions.

\subsection{Subspaces of almost periodic functions} 
\label{sub:finite_dimensional_subspaces_of_almost_periodic_functions}

Let $\bG=\bK\ltimes\bH$ satisfy the assumptions of Section~\ref{sec:repr-semidirect}.
Recall that Bohr (resp.\ Besicovitch) almost-periodic functions are the uniform (resp. $\ell^2$) limits of linear combinations of coefficients of unitary irreducible representations. In particular, each $N$-dimensional irreducible representation determines an $N$-dimensional subspace of almost-periodic functions.

For the remaining of the section, let us denote by $\widehat\cS\subset\widehat \bH$ any representative of $\widehat\bH/\bK$. Then, for a given set $F\subset\widehat\cS$ we let  
\begin{equation}
	\AP_F(\bG) = \spn\left\{ \langle T^\lambda(\cdot)e_m,e_n \rangle \mid n,m\in\bK \right\}\subset B_2(\bG).
\end{equation}
Namely, $f\in\AP_F(\bG)$ if and only if there exists $\cK_f\subset F$ such that
\begin{equation}
	f(a) = \sum_{\lambda\in \cK_f} \sum_{n,m\in \bK}\langle T^{\lambda}(a)e_n,e_m\rangle\,a_{n,m}(\lambda), \text{ with } \sum_{\lambda\in\cK_f}\sum_{n,m\in\bK}|a_{n,m}(\lambda)|^2<+\infty.
\end{equation}
Observe that, if $F$ is countable $\AP_F(\bG)$ is a separable subspace of $B_2(\bG)$, while if $F$ is finite $\AP_F(\bG)$ is finite dimensional and $\AP_F(\bG)\subset B_2(\bG)\cap \AP(\bG)$.

Direct computations, using the explicit expression \eqref{eq:T_components} for the coefficients of $T^\lambda$, yield the following.

\begin{proposition}\label{prop:ap-simpl}
	For any $f\in\AP_F(\bG)$ and any $(k,x)\in\bG$, it holds,
	\begin{equation}
		f(k,x) = \sum_{\lambda\in \cK_f} \sum_{n\in\bK} \phi(nk)\lambda(x) \,\hat f(k, n, \lambda),\qquad\text{where } \hat f(k,n,\lambda) = a_{n,nk}(\lambda).
	\end{equation}
	In particular, there exists a (linear) bijection $\cF_{\AP}:\AP_F(\bG)\to \bC^\bK\otimes\bC^\bK\otimes \ell^2(F)$, mapping $f$ to $\hat f$.
\end{proposition}

Recall that $f\in\AP(\bG)$ if and only if $f(k,\cdot)$ is an almost-periodic function over the abelian group $\bH$, for all $k\in\bK$.
For this reason, there exists a natural embedding of $\AP(\bH)$, the set almost-periodic functions over $\bH$, in $\AP(\bG)$ that acts by lifting $\psi\in\AP(\bH)$ to $\Psi\in\AP(\bG)$ given by $\Psi(k,x) = \delta_0(k)\,\psi(x)$. 
As an immediate consequence of this fact and of Proposition~\ref{prop:ap-simpl}, we get the following.

\begin{corollary}\label{cor:ap-simpl-H}
	Let $\AP_F(\bH)$ be the subspace of $\AP(\bH)$ of almost-periodic functions on $\bH$ that lift to $\AP_F(\bG)$. Then, for any $\psi\in\AP_F(\bH)$ we have $a_{n,m}=0$ if $n\neq m$. Moreover, for any $x\in\bH$ it holds,
	\begin{equation}
		\psi(x) = \sum_{\lambda\in F} \sum_{n\in\bK} \phi(n)\lambda(x) \,\hat \psi(n, \lambda),\qquad\text{where } \hat \psi(n,\lambda) = a_{n,n}(\lambda).
	\end{equation}
	In particular, there exists a (linear) bijection $\cF_{\AP}:\AP_F(\bH)\to \bC^\bK\otimes \ell^2(F)$, mapping $\psi$ to $\hat\psi$.
\end{corollary}

From the above, it follows that whenever $|\widehat\cS\setminus F|> 0$, no $f\in\AP_F(\bH)$ is weakly cyclic or weakly $\bR$-cyclic. This motivates the following definition.

\begin{definition}
  \label{def:APweakly-cyclic}
	A function $f\in \AP_F(\bH)$ is \emph{$\AP$-weakly cyclic} if the vector $\hat f_\lambda:=\hat f(\cdot,\lambda)$ defined in Corollary~\ref{cor:ap-simpl-H} is cyclic for a.e.\ $\lambda\in F\setminus\{\hat o\}$. 
	Similarly, a real valued function $f\AP_F(\bH)$ is \emph{$\AP$-weakly $\bR$-cyclic} if $\hat f_\lambda$ is $\bR$-cyclic for a.e.\ $\lambda\in F\setminus\{\hat o\}$.
	
	The sets of AP-weakly cyclic and AP-weakly $\bR$-cyclic functions are denoted respectively by $\cC^{\text{AP}}$ and $\cC^{\text{AP}}_\bR$.
\end{definition}

In order that the bispectral invariants defined in Part~\ref{bispectrum} make sense on $\AP_F(\bH)$, we need some assumptions on the set $F$.
Namely, let $\tilde F = \bigcup_{k\in \bK}\phi(k)F\subset \widehat\bH$ and 
\begin{equation}
	I^\otimes = \left\{ (\lambda_1,\lambda_2)\in\tilde F\times\tilde F \mid \lambda_1+\phi(k)\lambda_2\in\tilde F \text{ for any } k\in\bK \right\}.
\end{equation}
Then, we pose the following.

\begin{definition}
  \label{def:bispectrally-admissible-set}
  The set $F\subset \widehat\cS$ is \emph{bispectrally admissible}  if $\tilde F= \tilde F_1\cup \tilde F_2\subset \widehat\bH$ with $\tilde F_1\times \tilde F_1\subset I^\otimes $ and such that for any $\lambda\in \tilde F_2$ it holds $\lambda=\lambda_1+ \phi(k) \lambda_2$ for some $\lambda_1,\lambda_2\in \tilde F_1$ and $k\in \bK$.
\end{definition}

For a practical algorithm that generates bispectrally admissible sets, and some theoretical results on this subject, see Appendix~\ref{app:bispectral}.




\section{Functional spaces under consideration}
\label{sec:functional-spaces-under-consideration}

In this section we introduce the two functional spaces we are interested with: compactly supported real-valued square-integrable functions on the plane, which model natural images, and Besicovitch almost-periodic functions on the plane, which model textures.
%

\subsection{Compactly supported square-integrable functions on the plane}
\label{sec:square-integrable-functions-on-the-plane}



Let  $D_R\subset\bR^2$ be the compact disk of radius $R>0$.
For fixed $R>0$, the size of the screen,
images are elements of
\begin{equation}
   \cV(D_R) = \{f\in L^2_{\bR}(\bR^2)\mid \exists c\in\bR^2 \text{ s.t.\ } \supp f\subset c+D_R \text{ and } \avg f\neq 0\}.
\end{equation} 

Recall that the set $\cC_\bR$ is the set of weakly $\bR$-cyclic $L^2_\bR(\bR^2)$ functions, defined in Section~\ref{sub:weakly_cyclic_functions}.
The following can be proved by using the same argument of the third case in Theorem~\ref{thm:genericity-cG}.

\begin{theorem}
  \label{thm:genericity-weakly-cyclic}
  For any $R>0$ the set $\cC_\bR \cap \cV(D_R)$ is open dense in $\cV(D_R)$.
\end{theorem}

  We now define a {centering} operator for images, in the sense of Definition~\ref{def:centering}, which acts by translating the geometric center of the image into the origin. More precisely, let
  \begin{equation}
    \cI = \left\{ f\in L^2_\bR(\bR^2)\cap L^1_\bR(\bR^2) \mid \avg f\neq 0 \right\}.
  \end{equation}
  This is a closed subspace of $L^2_\bR(\bR^2)\cap L^1_\bR(\bR^2)$ with open and dense complement.
  Then, for $f\in\cI$, the \emph{geometric center of $f$} is the point $\cent(f) = (x_1^f,x^f_2)$, where
  \begin{equation}
    x_i^f = \frac 1 {\avg (f)}\int_{\bR^2} x_i\,f(x) \,dx \qquad i = 1,2 .
  \end{equation}
  The centering operator $\Phi_c:\cI\rightarrow \cI$ is then defined by $\Phi_c(f)= \tau_{\cent(f)}f$, so that the geometric center of $\Phi_c(f)$ is always the origin.

  Since $\cV(D_R)\subset \cI$, using this centering we obtain the following identification
  \begin{equation}
    \cV(D_R)\cong(\bR^2 \oplus \big\{ f\in L^2_{\bR}(D_R)\mid \cent(f)=0 \big\} )/ \sim,
  \end{equation}
  where $(c_1,f)\sim (c_2,g)$ if and only if either $f=g=0$ or $f=g$ and $c_1=c_2$.
  That is, a couple $(c,f)\in \cV(D_R)$ is composed of the actual image $f$ and its center $c$. 

\subsection{AP functions on the plane}\label{ap-functions-on-the-plane}

The space $B_2(\bG)$ of Besicovitch almost-periodic functions on a topological group $\bG$ has been introduced in Section~\ref{sssec:MAP-and-AP}.
We will consider $B_2(\bR^2)$ as a model of textures.
Recall that to distinguish the action of $SE(2,N)$ on $B_2(\bR^2)$ we have to use the $B_2(\bR^2)$-quasi regular representation, introduced in Definition~\ref{def:B2-quasi-reg}

  When considering textures, due to the finiteness of the screen, we will restrict ourselves to certain subsets of $B_2(\bR^2)$. 
  This is achieved by considering the space $\AP_F(\bR^2)$, where $F\subset\widehat\cS$ is a bispectrally admissible set, introduced in Section~\ref{sub:finite_dimensional_subspaces_of_almost_periodic_functions}.
  Let us denote by $\AP_{F,\bR}(\bR^2)\subset \AP_F(\bR^2)$ the set of real valued functions in $\AP_F(\bR^2)$.
  Observe that $\AP_{F,\bR}(\bR^2)\neq\varnothing$ only if 
  \begin{equation}
  	\tilde F = -\tilde F,\qquad \text{where}\qquad \tilde F = \bigcup_{k\in\bZ_N}R_kF.
  \end{equation}

  \begin{theorem}
    \label{thm:genericity-AP-weakly-cyclic}
    When $F$ is finite, the set $\cC^{\AP}$ is open and dense in $\AP_F(\bR^2)$.
    Moreover, when $F$ is countable the set $\cC^{\AP}$ is residual.
    The same results are true for the set $\cC^{\AP}_\bR$ w.r.t.\ $\AP_{F,\bR}(\bR^2)$, when $\tilde F=-\tilde F$.
  \end{theorem}

  \begin{proof}
    We start by claiming that the set of cyclic vectors in $\bC^N$ is open and dense.
    The openness follows from the fact that circulant matrices are diagonalized by the discrete Fourier transform (unitary)  matrix $\cF_N$.
    Indeed, this yields
    \begin{equation}
      \{v\in\bC^N\mid v \text{ is cyclic}\} 
      = \cF_N^* \left( \bigcap_{j=0}^{N-1} \{\hat v\in\bC^N\mid \hat v_j\neq 0 \} \right),
    \end{equation}
    which proves that $\{v\in\bC^N\mid v \text{ is cyclic}\}$ is open since it is the inverse image under an isometry of a finite intersection of open sets.
    The density follows by observing that $\Circ(v+w)=\Circ v +\Circ w$ and that, due to the analyticity of $\eps\mapsto \det(\eps A+ B)$, if $A,B$ are two matrices with $A$ invertible and $B$ not invertible, then $\eps A+ B$ is invertible for all $\eps>0$ sufficiently small.
    This completes the proof of the claim.

    Then, by definition,
    \begin{equation}
      \cC^{\AP} = \bigcap_{\lambda\in F} \{ f \mid \hat f_\lambda \text{ is cyclic.} \}.
    \end{equation}
    Since by the previous claim the sets on the r.h.s.\ are open and dense, this completes the proof of the statement regarding $\cC^{\text{AP}}$.
    The proof of the statement regarding $\cC_\bR^{\text{AP}}$ follows from similar arguments.
  \end{proof}

\subsubsection{Centering almost periodic functions} 
\label{ssec:centering_almost_periodic_functions}

In the final part of the paper, in order to be able to restrict only to the action of $\bZ_N$, we will need to quotient out the effect of translations on $B_2(\bR^2)$.
Here, given a \emph{finite} set $F\subset\widehat{\cS}$ and a compact subset $K$ of $\bR^2$, we define a centering of an appropriate subset of $\AP_F(\bR^2)\subset B_2(\bR^2)$ w.r.t.\ $K$, in the sense of Definition~\ref{def:centering}.
This centering is obtained exploiting the fact that functions of $\AP_F(\bR^2)$ are restrictions of periodic functions on a bigger space.

Let us denote by $Q$ the cardinality of $F$, so that 
\begin{equation}
	\tilde F = \left\{\lambda_k^j:=R_k\lambda^j\mid k\in\bZ_N,\, j=0,\ldots,Q-1\right\} \subset\widehat{\bR^2}.
\end{equation}
By Corollary~\ref{cor:ap-simpl-H}, this allows to identify $\AP_F(\bR^2)$ with $\bC^N\otimes\bC^Q$.

Let $\bT$ be the torus, endowed with the multiplicative group law given by the embedding $\bT\subset \bC$. 
Then, $\AP_F(\bR^2)$ is isomorphic to the linear space $\cV_F(\bT)$ of {linear expressions} on $\bT^{QN}$, via the 
  following isomorphism:
  \begin{equation}
    \Gamma : f(x) = \sum_{k\in\bZ_N}\sum_{j=0}^{Q-1} a_{j,k} \lambda_{k}^j(x) \in \AP_F(\bR^2) \mapsto 
    \tilde f(z) = \sum_{k\in\bZ_N}\sum_{j=0}^{Q-1} a_{j,k} z_k^j \in \cV_F(\bT).
  \end{equation}
Clearly, also $\tilde f\in\cV_F(\bT)$ can be identified with its components on $\bC^N\otimes \bC^Q$.
Moreover, letting $\Lambda:\bR^2\to \bT^{QN}$ be defined by $\Lambda(x) = (\lambda_k^j(x))_{k,j}$, we have $\Gamma^{-1}\tilde f =  \tilde f\circ \Lambda$ for any $\tilde f\in\cV_F(\bT)$. In this sense, functions in $\AP_F(\bR^2)$ are restrictions to $\bR^2$ of periodic functions in a $QN$-dimensional space.

The following proposition, which is a direct consequence of the definition of $\Gamma$, shows how the rotation and translation operator are modified by $\Gamma$.

\begin{proposition}
  \label{prop:Gamma-propr}
  For all $h\in\bZ_N$ and $y\in \bR^2$, the images under $\Gamma$ of the rotation and translation operators on $\AP_F(\bR^2)$ are given by:
  \begin{equation}
  	\tilde R_h = S_{-h} \otimes \idty \qquad\text{and}\qquad \tilde \tau_y = \tau_{\Lambda(y)}.
  \end{equation}
\end{proposition}

We now restrict our attention to the following space of ``admissible textures'':
\begin{equation}
  \cI = \left\{ f\in\AP_F(\bR^2)\mid \exists! m(f)\in\bT^{QN} \text{ s.t.\ } \max \Re\circ\Gamma f = \Re\circ\Gamma f(m(f))\right\}.
\end{equation}
Here, for any $z\in\bC$, we let $\Re z$ denote its real part.
In Lemma~\ref{lem:unique-max} we will prove that $\cI = \{ f\in\AP_F(\bR^2) \mid f_k^j\neq 0,\, \forall k,j \}$, and hence that admissible textures are open dense in $\AP_F(\bR^2)$.

For $f\in\cI$, by the previous proposition it is clear that $m(\tau_y f) = m(f)-\Lambda(y)$.
A first guess for centering functions on $\cI$ would then be to translate them by $y$ such that $\Lambda(y)=m(f)$. Unfortunately, this is usually impossible.

\begin{proposition}
  The set $\Lambda(\bR^2)$ is a $2$-dimensional sub-manifold of $\bT^{QN}$ whenever $N>2$ and $F\neq\{0\}$, or $N=2$ and $F$ contains at least two elements at the same distance from the origin or belonging to the same straight line through the origin.
\end{proposition}

\begin{proof}
Observe that $\Lambda$ is a Lie group homomorphism. Letting $\{e_1,e_2\}$ be the canonical basis for $\bR^2$ and $\Lambda_*:\bR^2\to \bR^{QN}$ the corresponding Lie algebra homomorphism, for any $\xi\in\bR^2$ we have
\begin{multline}\label{eq:homo}
	\Lambda(\xi) = \Lambda\circ\exp(\xi_1e_1+\xi_2e_2) \\
	= \exp\circ\Lambda_*(\xi_1e_1+\xi_2e_2) 
	= \exp(\xi_1\Lambda_*(e_1))\exp(\xi_2\Lambda_*(e_2)).
\end{multline}
Here, we exploited the fact that $[\Lambda_*(e_1),\Lambda_*(e_2)] = \Lambda_*([e_1,e_2]) = 0$. 

A simple computation shows that $\Lambda_*(e_i) = (\lambda_k^j(e_i))_{k,j}$, $i=1,2$. Since the two exponentials on the r.h.s.\ of \eqref{eq:homo} commutes, to prove that $\Lambda(\bR^2)$ is a $2$-dimensional sub-manifold of $\bT^{QN}$ it suffices to show that $\Lambda_*(e_1)$ and $\Lambda_*(e_2)$ are linearly independent. To this aim, let us denote $\lambda_k^j = \xi_j e^{i(\omega_j+2\pi k/N)}$. Then, simple computations show that the linear dependence of $\Lambda_*(e_1)$ and $\Lambda_*(e_2)$ is equivalent to the existence of $\alpha\in\bR$ such that whenever $\xi_j\neq 0$ it holds
\begin{equation}
  \tan\left( \omega_j-\frac{2\pi}N k \right) = \frac{\alpha - \xi_j}{\xi_j},\qquad \forall k\in\{0,\ldots,N-1\}.
\end{equation}
Since $\tan\theta = c$ has two solutions for $\theta\in[0,2\pi)$, this is clearly impossible whenever $N>2$, thus proving this case. On the other hand, if $N=2$, this reduces to
\begin{equation}
  \tan\omega_j = \frac{\alpha - \xi_j}{\xi_j}, \qquad \text{for all } j \text{ s.t.\ }\xi_j\neq 0,
\end{equation}
which cannot be satisfied under the given assumptions.
\end{proof}

Since $\Lambda(\bR^2)$ is a two-dimensional manifold, we have that in general $m(f)\not\in \Lambda(\bR^2)$.
Even worst,
the function $\xi\in \Lambda(\bR^2)\mapsto \|m(f)-\xi\|$ could not even attain a minimum.
For this reason, we need to restrict our attention to a bounded set of translations.

Fix a compact $K\subset \bR^2$.
Since $\{\Lambda(y)\mid y\in K\}$ is a closed subset of $\bT^{QN}$ the value $\min_{y\in K}\| m(f)-\Lambda(y) \|$ is attained in at least one point.
We then let
\begin{equation}
  \cR_K \coloneqq \left\{ f\in \cI \mid \min_{y\in K}\| m(f)-\Lambda(y) \| \text{ is realized at exactly one point of } \bR^2 \right\} \subset \cT(E).
\end{equation}

\begin{definition}
  The \emph{center of $f\in\cR_K$} is the point $\cent(f)\in\bR^2$ such that 
  \begin{equation}
    \|m(f)-\cent(f)\| = \min_{y\in K}\| m(f)-\Lambda(y) \|.
  \end{equation}
  The \emph{centering} of $\cR_K$ w.r.t.\ $K$ is then the function $\Phi:\cR_K\to \cR_K$ defined by $\Phi(f) = \tau_{\cent(f)}f$.
\end{definition}

\begin{remark}
  Here we centered w.r.t.\ the maximum of the real part of the function $\Gamma f$ since we cannot define the geometric center of a function in $\AP_F(\bR^2)$.
  Indeed, it is easy to see that all functions in $\cV_F(\bT)$ have zero average.
\end{remark}


The following follows immediately from the definition.

\begin{proposition}
  The set $\cR_K$ is residual in $\AP_F(\bR^2)$.
\end{proposition}

We conclude the section with the following characterization of $\cI$.

\begin{lemma}
  \label{lem:unique-max}
    Let any $f\in AP_F(\bR^2)$ be represented as
    \begin{equation}
      f(x) = \sum_{k\in\bZ_N}\sum_{j=0}^{Q-1} \rho_{j,k} e^{2\pi \theta_{j,k} \lambda_{k}^j(x)}.
    \end{equation}
    Then, $\cI = \{f\in \AP_F(\bR^2)\mid \rho_{j,k}\neq 0,\, \forall j,k\}$. Moreover, we have that $m(f) = (-\theta_k^j)_{k,j}$. 
\end{lemma}

\begin{proof}
  From the definition of $\Gamma$ and the addition formula of the cosine, follows that
  \begin{equation}
    \begin{split}
      \Re\circ\Gamma f (z) 
      &= \sum_{j=0}^{M-1}\sum_{k\in\bZ_N} {\left(\Re(a_{j,k}) \cos\left(2\pi z_k^j\right) -\Im(a_{j,k}) \sin\left(2\pi z_k^j\right)\right)}\\
      &= \sum_{j=0}^{M-1}\sum_{k\in\bZ_N} \underbrace{\rho_{j,k} \cos\left( 2\pi (\theta_{j,k}+z_k^j)\right)}_{\eqqcolon \varphi_{j,k}(z_k^j)}.
    \end{split}
  \end{equation}
  Since all the $\varphi_{j,k}$ depend of different variables, it is clear that the maximum is realized at those points $\bar z=(\bar z_k^j)$ such that 
  \begin{equation}
    \varphi_{j,k}(\bar z_k^j) = \max_{\xi\in\bT} \varphi_{j,k}(\xi), \quad \forall j\in\{0,\ldots, M-1\},\,k\in\bZ_N.
  \end{equation}
  Clearly, if $\rho_{j,k}=0$ for some $k,j$ then $\varphi_{j,k}\equiv 0$, proving the only if part of the statement.
  On the other hand, if $\rho_{j,k}\neq 0$ then $\varphi_{j,k}$ has exactly one maximum in $\bar z_k^j=-\theta_{j,k}$, completing the proof of the lemma.
\end{proof}


%% file: lifts.tex

Let $\bG = \bK\ltimes\bH$ be a semi-direct product, as considered in Section~\ref{sec:repr-semidirect}.
Here, we are interested in operators $L:L^2(\bH)\to L^2(\bG)$, which we call \emph{lift operators} for obvious reasons.
Observe that, via the isomorphism $\sigma^*:B_2(\bH)\to L^2(\bH^\flat)$ any lift $L:L^2(\bH^\flat)\to L^2(\bG^\flat)$ induces a lift $L':B_2(\bH)\to B_2(\bG)$ of Besicovitch almost periodic functions.

We are mainly interested in identifying the action of the quasi-regular representation on $f\in L^2(\bH)$ by analyzing the Fourier transform of the lift $Lf\in L^2(\bG)$.
Thus, the first, and more natural, requirement on the lift operation is to intertwine the quasi-regular representation acting on $L^2(\bH)$ with the left regular representations on $L^2(\bG)$. We call these type of lifts \emph{left-invariant}. We show that, under some mild regularity assumptions on $L$, left-invariant lifts coincide with wavelet transforms, as defined in Section~\ref{sec:wavelet}. These kind of lifts have been extensively studied in, e.g., \cite{DF}, and related works.

Unfortunately, left-invariant lifts have a huge drawback for our purposes: they never have an invertible non-commutative Fourier transform $\widehat{Lf}(T^\lambda)$.
The second part of this chapter is then devoted to the generalization of the concept of \emph{cyclic lift}, introduced in \cite{Smach2008} exactly to overcome the above problem. In this general context, we will present a cyclic lift as a combination of an \emph{almost-left-invariant lift} and a centering operation, as defined in Definition~\ref{def:centering}. As a consequence, we obtain a precise characterization of the invertibility of $\widehat{Lf}(T^\lambda)$ for these lifts.

\section{Left-invariant lifts}
  \label{sec:right-invariant-lifts}

    In this section we introduce the most natural class of lift operators and we show that, under mild regularity assumptions, these lifts coincide with wavelet transforms. This will allow us to prove the non-invertibility of the Fourier transform $\widehat{Lf}(T^\lambda)$.

  \begin{definition}
  A lift operator $L:L^2(\bH)\to L^2(\bG)$ is \emph{left-invariant} if 
  \begin{equation}
  	\Lambda(a)\circ L = L\circ \pi(a),\quad\text{ for any }a\in\bG.
  \end{equation}
  \end{definition}

  It is clear from the definition that for an \emph{injective} left-invariant lift it holds
  \begin{equation}
    \label{eq:left-inv-lift}
    Lf = \Lambda(a)Lg \iff f = \pi(a) g.
  \end{equation}
  In the sequel we will thus be mainly interested in injective left-invariant lifts.

  Obviously, any wavelet transform via the quasi-regular representation induces a left-invariant lift operator.
  As presented in Section~\ref{sec:wavelet} the injectivity of these lift is equivalent to the existence of a weakly admissible vector.
  Later on we will call these \emph{regular left-invariant lift operators}.

  It is readily seen from Theorem~\ref{thm:admissible-vect}, that if $\widehat\bH$ is not compact, no admissible vector exists for the quasi-regular representation of $\bG$.
  Thus, in this case no regular left-invariant lift can be an isometry.

  \begin{remark}
  Closed invariant subspaces of $\pi$ are characterized in \eqref{eq:invariant-subspace}, and in particular they are of the form $\cA=\cA_U$ where $U\subset \widehat\bH$ is a $\bK$ invariant measurable set.
  From Theorem~\ref{thm:admissible-vect} it follows the existence of admissible vectors for any sub-representation $\pi_U$ of $\pi$ with $|U|<+\infty$.
  By considering restrictions of regular left-invariant lift to these subsets, it is then possible to obtain isometric lifts.
  However, due to the Paley-Wiener Theorem, none of these $\cA_U$ contains the compactly supported functions when $\widehat \bH$ is non-compact.
  It is worth to mention that this is the approach chosen in \cite{DF}, while in \cite[Appendix B]{Bekkers2014} the authors circumvent the problem by requiring additional regularity, that is, by considering lift operators of the form $L:H^\ell(\bH)\to L^2(\bH)$.
  \end{remark}

  We now characterize those left-invariant lifts that come from wavelet transforms, showing that most ``reasonable'' injective left-invariant lifts are of this type.
 We remark that, as a consequence, it is possible to derive a characterization of the range of these lifts. See, e.g., \cite[Theorem 4]{Duits2007}.

  \begin{theorem}
    \label{thm:left-inv-form}
    Let $L:L^2(\bH)\to C(\bG)\cap L^2(\bG)$ be a \emph{linear} left-invariant lift such that $f\mapsto Lf(e,o)$ is a continuous function from $\cA$ to $\bC$.
    Then, there exists $\Psi\in L^2(\bH)$ such that $\lambda\mapsto\|\hat\Psi^\lambda\|_{L^2(\bK)}$ is essentially bounded on $\widehat{\bH}$ and
    \begin{equation}
      \label{eq:left-inv-form}
      Lf(a)= W_\Psi f(a) \, (= \langle \pi(a)\Psi, f\rangle) \qquad\forall a\in \bG.
    \end{equation} 

    Moreover, $L$ is injective if and only if $\lambda\mapsto\|\hat\Psi^\lambda\|_{L^2(\bK)}$ is a.e.\ strictly positive.
  \end{theorem}  

  \begin{proof}
    By the assumptions, $f\mapsto Lf(e,o)$ is an element of the dual $L^2(\bH)^*$, which by Riesz Theorem can be identified with $L^2(\bH)$.
    Thus, there exists $\Psi\in L^2(\bH)$ such that $Lf(e,o)=\langle f,\Psi \rangle$.
    Formula \eqref{eq:left-inv-form} is then obtained from the left-invariance of $L$.
    Indeed, by this and the unitarity of $\pi$, for any $a\in\bG$ it holds
    \[
    Lf(a) = \Lambda(a^{-1}) Lf(e,o) = L(\pi(a^{-1})f)(e,o) = \langle \pi(a^{-1})f,\Psi \rangle = \langle f, \pi(a)\Psi \rangle. 
    \]

    Finally, the fact that $\lambda\mapsto\|\hat\Psi^\lambda\|_{L^2(\bK)}$ belongs to $L^\infty(\widehat{\bH})$ is a consequence of the discussion following \eqref{eq:wavelet-transform-norm}, while the last statement is a consequence of Theorem~\ref{thm:admissible-vect}.
  \end{proof}

  \begin{remark}
    In the above theorem, we could have assumed the function $f\mapsto Lf(e,o)$ to be continuous from $C_c(\bH)$ (or $C_0(\bH)$) to $\bC$.
    Due to the characterization of the dual of $C_c(\bH)$ (or $C_0(\bH)$), this would have yield the same result with the wavelet $\Psi$ being a finite (or locally finite) Radon measure on $\bH$.
    
    The trivial lift considered in \cite{Smach2008} is obtained in a similar way, choosing $\Psi=\delta_{o}$, the Dirac delta mass centered at the identity of $\bH$.
    Observe that this choice does not guarantee $\range L\subset C(\bG)$.
  \end{remark}

  \begin{definition}
    A left-invariant lift $L$ is \emph{regular} if it satisfies the assumptions of Theorem~\ref{thm:left-inv-form} and is injective.
  \end{definition}

  We then have the following result, which proves the non-invertibility of the Fourier transforms of regular lifts.
  
  \begin{corollary}\label{cor:non-inver-li}
    Let $L:L^2(\bH)\to L^2(\bG)$ be a regular lift.
    Then, it holds that
    \begin{equation}
      \rank \widehat{Lf}(T^\lambda) \le 1, \qquad \text{for any }\lambda\in\widehat \bH\setminus\{o\}.
    \end{equation}
  \end{corollary}
  
  \begin{proof}
    The result is an immediate consequence of Theorems~\ref{prop:ft-lift} and \ref{thm:left-inv-form}.
  \end{proof}

  \section{Cyclic lift}
  \label{sec:cyclic-lift}

  In \cite{Smach2008} to overcome the difficulties presented by non-invertible Fourier transforms, a different lift operator (called cyclic lift) is considered.
  In this section we put those ideas in a more general context.
  Indeed, close analysis of the cyclic lift of \cite{Smach2008}, shows that it is the composition of two operators, that we will discuss in the following sections.

  \subsection{Almost left-invariant lifts} 
  \label{ssub:almost_left_invariant_lifts}

  The first problem to overcome when building a lift that can yield invertible Fourier transforms, is to avoid left-invariance.

  \begin{definition}
    An operator $L:L^2(\bH) \to L^2(\bG)$ is \emph{almost left-invariant} if 
    \begin{equation}
      \Lambda(h,y)Lf(k,x) = L(\pi(h^2,\phi(k) y)f)(k,x),\qquad \text{ for all }(k,x),(h,y)\in\bG.
    \end{equation}
  \end{definition}

  From the definition, it immediately follows that injective almost left-invariant lifts satisfy
  \begin{gather}
    \label{eq:almost-left-invariance}
  	Lf = \Lambda(h,o) Lg \iff f = \pi(h^2,0)g \\
  	Lf(e,\cdot) = \Lambda(e,y) Lg(e,\cdot) \iff f = \pi(e,y) g.
  \end{gather}
  Observe that the second equivalence holds only for $Lf(e,\cdot)$.
  The fact that it cannot be extended to $k\neq e$ implies that the invariants of almost left-invariant lifts cannot separate the action of translations on $L^2(\bH)$.
  To overcome this problem we will later introduce cyclic lifts.

	The following theorem (similar to Theorem~\ref{thm:left-inv-form}) justifies the above definition.

  \begin{theorem}
  	\label{thm:almost-left-inv-form}
    Let $L:L^2(\bH)\to C(\bG)\cap L^2(\bG)$ be a linear almost left-invariant lift such that $f\mapsto Lf(e,o)$ is a continuous function from $L^2(\bH)$ to $\bC$.
    Then, there exists $\Psi\in L^2(\bH)$ satisfying
    \begin{equation}
      \label{eq:almost-left-inv-ess-bound-Psi}
      \lambda\mapsto \sum_{k\in\bK}|\widehat{\Psi^*}_\lambda(k^2)|^2 \text{ is essentially bounded on }\widehat{\bH},
    \end{equation}
    and such that
    \begin{equation}
        \label{eq:almost-left-inv-form}
    	Lf(k,x) = \left\langle \phi(k)\pi(k,x)\Psi, f \right\rangle \qquad\forall(k,x)\in \bG.
    \end{equation}

    Moreover, $L$ is injective if and only if the function in \eqref{eq:almost-left-inv-ess-bound-Psi} is strictly positive.
  \end{theorem}

  \begin{proof}
    As in Theorem~\ref{thm:left-inv-form}, from the Riesz representation theorem follows immediately that $Lf(e,o)=\langle \Psi, f\rangle$ for some $\Psi\in L^2(\bH)$.
    Then,  \eqref{eq:almost-left-inv-form} follows by writing $Lf(k,x)=\Lambda(k,x)^{-1}Lf(e,o)$, using the definition of almost-left invariance, and the unitarity of $\pi$. Indeed, this yields,
    \begin{equation}
    	\begin{split}
			Lf(k,x)
			&=\Lambda(k^{-1},-\phi(k^{-1})x)Lf(e,o)\\
			&=\left\langle \Psi, \pi(k^{-2},-\phi(k^{-1})x)f \right\rangle\\
			&=\left\langle \pi(k^2,\phi(k)x)\Psi,f\right\rangle\\
			&=\left\langle \phi(k)\pi(k,x)\Psi,f\right\rangle\\
    	\end{split}
    \end{equation}

    To prove \eqref{eq:almost-left-inv-ess-bound-Psi}, observe that by \eqref{eq:almost-left-inv-form} we have
    \begin{equation}
    	\label{eq:FT-almost-left}
      \cF(Lf(k,\cdot)) (\lambda) = \hat f (\phi(k)\lambda)\,\widehat{\Psi^*}(\phi(k^{-1})\lambda).
    \end{equation}
    This allows to compute, via the Parseval identity,
    \begin{equation}
		\begin{split}
			\|Lf\|_{L^2(\bG)} 
			&= \sum_{k\in\bK} \int_{\widehat\bH} |\hat f(\phi(k)\lambda)|^2 |\widehat{\Psi^*}(\phi(k^{-1})\lambda)|^2\,d\lambda\\
			&= \int_{\widehat\bH} |\hat f(\phi(k)\mu)|^2 \sum_{k\in\bK} |\widehat{\Psi^*}(\phi(k^{-2})\mu)|^2\,d\mu\\
			&= \int_{\widehat\bH} |\hat f(\phi(k)\mu)|^2 \sum_{h\in\bK} |\widehat{\Psi^*}_\mu(h^2)|^2\,d\mu.
		\end{split}
    \end{equation}
    Since $(L^1(\widehat\bH))^*=L^\infty(\widehat\bH)$, this implies \eqref{eq:almost-left-inv-ess-bound-Psi}.
    Moreover, it also implies that $\ker L=\{0\}$ if and only if the function in \eqref{eq:almost-left-inv-ess-bound-Psi} is positive.
    By linearity of $L$ this implies the last statement.
  \end{proof}

  \begin{remark}
    \label{rmk:psi-measure-almost-left-inv}
    As in the case of Theorem~\ref{thm:left-inv-form}, in the above we could have assumed the function $f\mapsto Lf(e,o)$ to be continuous from $C_c(\bH)$ (or $C_0(\bH)$) to $\bC$.
    This would have yielded similar results with $\Psi$ being a finite (or locally finite) Radon measure on $\bH$.
  \end{remark}


  \begin{definition}
    An almost left-invariant lift is \emph{regular} if it satisfies the assumptions of Theorem~\ref{thm:almost-left-inv-form} and is injective.
  \end{definition}

  Let us observe that the conditions on $\Psi$ for an almost left-invariant lift to be regular coincide with those for left-invariant lifts (obtained in Theorem~\ref{thm:left-inv-form}) if and only if $\bK^2:=\{k^2\mid k\in\bK\}$, which a priory is only a subgroup of $\bK$, satisfies $\bK^2\simeq \bK$.
  If $\bK=\bZ_N$, this is equivalent to $N$ being odd.

  \subsection{Fourier transform of lifted functions} 
  \label{ssub:fourier_transform_of_lifted_functions}
  
  We present the analog for almost left-invariant lifts of Proposition~\ref{prop:ft-lift} and Corollary~\ref{cor:ft-lift-tensor}, which allows us to compute the non-commutative Fourier transform of cyclic lifts. Then, as a consequence, we present a characterization of the rank of such Fourier transform, thus showing that, under appropriate assumptions on the lifted function, it can be invertible.
    
  \begin{proposition}
    \label{prop:ft-almost-left-lift}
    Let $L:L^2(\bH)\to L^2(\bG)$ be a regular almost left-invariant lift and let $f\in L^2(\bH)$.
  	Then, for any $\lambda\in\widehat \bH\setminus\{\hat o\}$, it holds that $\widehat {Lf}(T^\lambda)\in \operatorname{HS}(L^2(\bK))$ has matrix elements
  	\begin{equation}
  		\widehat {Lf}(T^\lambda)_{i,j} = \widehat{\Psi^*}_\lambda(i)\,\hat f_\lambda(i^{-1}j^2).
  	\end{equation}

  	Moreover, if $\Psi\in L^1(\bH)$ it holds
    \begin{equation}
      \widehat{Lf}(T^{\hat o\times \hat k}) = 
      \begin{cases}
      \avg(f)\,\avg(\bar\Psi) &\qquad\text{if } \hat k = \hat e,\\
      0 & \qquad \text{otherwise,}
      \end{cases}\qquad \forall f\in L^2(\bH)\cap L^1(\bH),\,\hat k\in\widehat\bK.
    \end{equation}
  \end{proposition}

  \begin{proof}
    The second part of the statement can be proved exactly as in Proposition~\ref{prop:ft-lift}.
  	On the other hand, by Proposition~\ref{prop:FT-semidirect} and \eqref{eq:FT-almost-left}, we have
  	\begin{equation}
  		\widehat{Lf}(T^\lambda)_{i,j} = \hat f(\phi(i^{-1}j)\phi(j)\lambda)\,\widehat{\Psi^*}(\phi(j^{-1}i)\phi(j)\lambda) = \widehat{\Psi^*}_\lambda(i)\,\hat f_\lambda(i^{-1}j^2).
  	\end{equation}
    This proves the first part of the statement, completing the proof of the proposition.
  \end{proof}
  
   In Section~\ref{sub:bispectral_invariants_for_lifts} we will need the following consequence of the Induction-Reduction Theorem, which can be proved as Corollary~\ref{cor:ft-lift-tensor}.
  \begin{corollary}
    \label{prop:ft-almost-left-lift-tensor}
    Let $f\in L^2(\bH)$.
    Then, for any $\lambda_1,\lambda_2\in\widehat \bH$ it holds that 
    \begin{equation}
        A\circ\widehat{Lf} (T^{\lambda_1}\otimes T^{\lambda_2}) \circ A^{*}
        = \bigoplus_{k\in\bK} \left(\widehat{\Psi^*}_\lambda(i)\,\hat f_\lambda(i^{-1}j^2) \right)_{i,j\in\bK}.
    \end{equation}
	Here, $A$ is the equivalence from $L^2(\bK\times\bK)$ to $\bigoplus_{k\in\bK}L^2(\bK)$ given by the Induction-Reduction Theorem. (See Theorem~\ref{thm:ind-reduction}.)  
	\end{corollary}
  
  As a consequence of Proposition~\ref{prop:ft-almost-left-lift}, we obtain the following description of the conditions for $\widehat{Lf}$ to be invertible. 

  \begin{proposition}
    \label{prop:almost-inv-lift-ft-inv}
    Let $L$ be a regular almost left-invariant lift.
    Then, for any $\lambda\in\widehat \bH\setminus\{o\}$ such that $\hat\Psi(\phi(k)\lambda)\neq0$ for all $k\in\bK$, it holds that
    \begin{equation}
      \rank \widehat{Lf}(T^\lambda) = \dim\spn\{S(k)\hat f_\lambda\mid k\in\bK\}.
    \end{equation}
  \end{proposition}
    
  \begin{proof}
    Let $D=\diag \widehat {\Psi^*}_\lambda$ and $F\in L^2(\bK)\otimes L^2(\bK)$ be $F_{i,j}=\hat f_\lambda(i^{-1}j^2)$.
    By Theorem~\ref{prop:ft-almost-left-lift} we have $\widehat {Lf}(T^\lambda) =  DF$.
    Hence, since $D$ is invertible by assumption, 
    \begin{equation}
      \rank \widehat {Lf}(T^\lambda) =  \rank F.
    \end{equation}
    Consider the invertible operator $B$ 
    such that $B(e_{i}\otimes e_j) = e_{i^{-1}j}\otimes e_j$ for all $i,j\in\bK$ and observe that
    \begin{equation}
      (BF)_{i,j} = F_{ij,j} = \hat f_{\lambda}(i^{-1} j) = S(i) \hat f_{\lambda}(j).
    \end{equation}
    Thus, the $i$-th row of $BF$ is exactly $S(i)\hat f_\lambda$ and hence the statement follows from
    \begin{equation}\label{eq:rank}
      \rank F = \rank BF = \dim\spn\{S(k)\hat f_\lambda\mid k\in\bK\}. 
    \end{equation}
  \end{proof}

In the following chapters we will exploit this consequence of the above.

\begin{corollary}
  Let $L$ be a regular almost left-invariant lift. Then, for any $f\in L^2(\bR^2)$ that is weakly-cyclic, $\widehat{Lf}(T^\lambda)$ is invertible for a.e.\ $\lambda\in \widehat\bH$.
  
  If the action of $\bK$ on $\bH$ is \emph{not} even, in the sense of Definition~\ref{def:even}, the same is true for any $f\in L^2_\bR(\bR^2)$ which is real-valued and $\bR$-weakly-cyclic. On the other hand, if the action of $\bK$ on $\bH$ is even, for any such function it holds $\dim\rank\widehat{Lf}(T^\lambda)=N/2$ for a.e.\ $\lambda\in\widehat\bH$.
\end{corollary}

\begin{proof}
  The first statement is an immediate consequence of Proposition~\ref{prop:almost-inv-lift-ft-inv}, as is the first part of the second. 
  To complete the proof, recall that by Proposition~\ref{prop:omega-real} if $\bK$ acts evenly on $\bH$, $\dim\spn\{S(k)\hat f_\lambda\mid k\in\bK\}\le N/2$
  Then, letting $D$ and $F$ be as in the proof of Proposition~\ref{prop:almost-inv-lift-ft-inv}, by \eqref{eq:rank} we have that
  \begin{equation}
    \rank \widehat{Lf}(T^\lambda)\le \min\{\rank D,\rank F\}\le N/2. 
  \end{equation}
\end{proof}

  \subsection{Cyclic lifts}
  \label{ssub:cyclic-lifts}

  Cyclic lifts are obtained by composing almost left-invariant lifts with the centering operators defined in Definition~\ref{def:centering}, in order to quotient out the action of $\bH$ from $L^2(\bH)$. 

  \begin{definition}
    Let $\cA\subset L^2(\bH)$ be invariant under the action of $\pi$ and $U\subset\bH$.
  	A lift operator $L:\cA \to L^2(\bG)$ is a \emph{cyclic lift} if there exist a centering $\Phi$ of $\cA$ w.r.t.\ $U$ and an almost left-invariant lift $P$ such that $L=P\circ \Phi$.
  \end{definition}

  From the definition of centering and from \eqref{eq:almost-left-invariance}, it immediately follows that, whenever $P$ is injective,
  \begin{equation}
    \label{eq:cyclic-lift-inv}
    Lf = \Lambda(k,0)Lg \iff f = \pi(k^2,x)g \text{ for some } x\in U \subset \bH.
  \end{equation}
  In particular, if $U=\bH$ and $\bK\simeq\bK^2$, a cyclic lift can  be used to separate translations and rotations.
  Together with Proposition~\ref{prop:almost-inv-lift-ft-inv} this is the second reason why, when $\bK=\bZ_N$, we will need to assume its cardinality to be odd.

  \begin{definition}
    A cyclic lift $L=P\circ \Phi$ is \emph{regular} if $P$ is a regular almost-invariant lift.
  \end{definition}

  The following is immediate, from Theorem~\ref{thm:almost-left-inv-form}.

  \begin{corollary}
    \label{thm:cyclic-form}
    Let $L=P\circ \Phi: \cA\to C(\bG)\cap L^2(\bG)$ be a regular cyclic lift.
        Then, there exists $\Psi\in L^2(\bH)$ satisfying
    \begin{equation}
      \label{eq:cyclic-Psi-ess-bound}
      \lambda\mapsto \sum_{k\in\bK}|\widehat{\Psi^*}_\lambda(k^2)|^2 \text{ is essentially bounded on }\widehat{\bH},
    \end{equation}
    and such that
    \begin{equation}
        \label{eq:cyclic-form}
    	Lf(k,x) = \left\langle \phi(k)\pi(k,x)\Psi, \Phi f \right\rangle \qquad\forall(k,x)\in \bG.
    \end{equation}

    Moreover, $L$ is injective if and only if the function in \eqref{eq:cyclic-Psi-ess-bound} is a.e.\ positive.
  \end{corollary}

  \begin{remark}
    The cyclic lift considered in \cite{Smach2008} is obtained by choosing $\Phi$ to be the centering discussed in Section~\ref{sec:square-integrable-functions-on-the-plane} defined for $\cA = \cV(D_R)$ and letting $\Psi=\delta_o$, the Dirac delta mass centered at the identity of $\bH$ (see Remark~\ref{rmk:psi-measure-almost-left-inv}).
  \end{remark}

%% file: ap_interp.tex

In this chapter, following \cite{ap_interp}, we present a method to interpolate or approximate a given function $f:\bG\to \bC$ (or $F:\bH\to\bC$) by an AP functions in $\AP_F(\bG)$, i.e., AP functions whose Fourier transform is supported in a given discrete and finite set $F\subset\widehat\bG$. (See Section~\ref{sub:finite_dimensional_subspaces_of_almost_periodic_functions}.) 
In order to do this, we generalize the well-known decomposition of the 2D Fourier transform on the plane in polar coordinates, via the Fourier-Bessel operator that we recall briefly below.

The Fourier Transform $\hat f$ of a function $f\in L^2(\bR^2)$ can be obtained by, firstly, developing $f$ in a multi-pole series $f(\rho e^{i\theta}) = \sum_{n\in\bZ} f_n(\rho) e^{in\theta}$,  and then applying the Hankel transform on the $f_n$'s. More precisely, polar coordinates allow to identify $L^2(\bR^2)\simeq L^2(\bS^1)\otimes L^2(\bR_+,\rho\,d\rho)$. Then, letting $\cF:L^2(\bS^1)\to L^2(\bZ)$ be the Fourier transform on $\bS^1$ and $\cF_{\bR^2}:L^2(\bR^2)\to L^2(\bR^2)$ be the one on $\bR^2$, we have
\begin{equation}\label{eq:hankel}
	\cF_{\bR^2} = (\cF^*\otimes \idty)\circ \cJ \circ (\cF\otimes \idty).
\end{equation}
Here, we implicitly identified $L^2(\bZ)\otimes L^2(\bR_+,\rho\,d\rho)\simeq \bigoplus_{k\in \bZ} L^2(\bR_+,\rho\,d\rho)$, and let $\cJ=  \bigoplus_{n\in \bZ} \cJ_{n}$ be the so-called \emph{Fourier-Bessel operator}. Namely, $\cJ_n$ is a re-normalized version of the $n$-th Hankel transform operator:
\begin{equation}
	\cJ_n \varphi(\lambda) = (-i)^n\int_{\bR_+} \varphi(\rho)J_n(\lambda\rho)\,\rho d\rho, \qquad \varphi:\bR_+\to\bR,
\end{equation}
where $J_n$ is the $n$-th Bessel function of the first kind, which appears as the matrix coefficients of representations of $SE(2)$, as shown in \cite{ap_interp, VIL, vilenkin}.

In the first part of the chapter, exploiting the deep connection between \eqref{eq:hankel} and the group of rototranslations $SE(2)$, we generalize the former to $\AP_F(\bG)$ functions. In particular, we show how a discrete operator that we call the \emph{generalized Fourier-Bessel} operator plays a crucial role in this generalization. We then consider the problem of interpolating functions $\psi:\bG\to\bC$ on $\bK$-invariant finite sets $\tilde E\subset \bG$ via $\AP_F(\bG)$ functions.

The last part of the chapter is devoted to particularize (and slightly generalize) the above results to the relevant case for image processing, i.e., $\bG=SE(2,N)$. Indeed we present numerical algorithms for the (exact) evaluation, interpolation, and approximation of $\AP_F(SE(2,N))$ functions on finite sets of spatial samples $\tilde E\subset \bR^2$, invariant under the action of $\bZ_N$. This is an instance of a very general problem, and can be seen as a generalization of the discrete Fourier Transform and its inverse, that act on regular square grids, i.e., invariant under the the action of $SE(2,4)$.

\section{Generalized Fourier-Bessel operator}

Recall that the matrix coefficients of $T^\lambda$, with respect to the basis $\widehat\bK$ of $L^2(\bK)\simeq\bC^\bK$, are the functions,
\begin{equation}
	t^\lambda_{\hat m, \hat n}(g) := \sum_{\ell\in\bK} T^\lambda(g)\hat n(\ell)\,\overline{\hat m(\ell)},\qquad \forall g\in \bG, \, \hat n, \hat m\in \widehat\bK.
\end{equation}
Since, in the case of $SE(2)$, Bessel functions appear inside in these coefficients, we now compute them in order to obtain a coherent generalization of Bessel functions to this context. 

In order to do so, let us mimic the polar coordinates construction, by choosing a bijection of $\bH/\bK\times\bK$ to $\bH$. To this aim, fix any section $\sigma: \bH/\bK\to \bH$, that we do not assume to have any regularity. Indeed, the arguments that follow work even for non-measurable $\sigma$'s. Then, $\Xi:f\mapsto f\circ \sigma$ is a bijection between functions on $\bG$ and functions on $\bK\times\bK\times\bH/\bK$. More precisely, if $f:\bG\to \bC$, then $\Xi f(k,h,y) := f(k,\phi(h)\sigma(y))$. 

\begin{proposition}
	The matrix elements of $T^\lambda$, $\lambda\in \widehat\cS$, with respect to the basis $\widehat \bK$ of $L^2(\bK)$ are
	\begin{equation}
		\Xi\, t^\lambda_{\hat m, \hat n} (k,h,y) = 
		\overline{\hat n}(k)\, [\hat n - \hat m](h) \sum_{\ell\in\bK}\lambda(\phi(\ell)\sigma(y))\overline{[\hat n-\hat m]}(\ell)
	\end{equation}
\end{proposition}

\begin{proof}
	The statement follows by direct computations,
	\begin{equation}
		\begin{split}
			t^\lambda_{\hat m, \hat n}(k,\phi(h)\sigma(y)) 
			&= \sum_{\ell\in\bK} T^\lambda(k,\phi(h)\sigma(y))\hat n(\ell)\,\overline{\hat m(\ell)}\\
			&= \sum_{\ell\in\bK} \lambda(\phi(\ell^{-1}h)\sigma(y))\hat n(k^{-1}\ell)\,\overline{\hat m(\ell)}\\
			&= \overline{\hat n}(k) \sum_{\ell\in\bK} \lambda(\phi(\ell^{-1}h)\sigma(y))[\hat n-\hat m](\ell)\\
			&= \overline{\hat n}(k)[\hat n-\hat m](h) \sum_{r\in\bK} \lambda(\phi(r)\sigma(y))\overline{[\hat n-\hat m]}(r). 
		\end{split}
	\end{equation}
\end{proof}

The above proposition justifies the following.

\begin{definition}\label{def:gen-bessel}
	The \emph{generalized Bessel function of parameters $\hat n\in\widehat\bK$} is the function defined by
	\begin{equation}
		J_{\hat n}:\widehat\cS \times \bH/\bK\to \bC, \qquad J_{\hat n}(\lambda, y) = \sum_{r\in\bK} \lambda(\phi(r)\sigma(y))\overline{\hat n(r)}.
	\end{equation}
	The \emph{generalized Fourier-Bessel operator} is the operator 
	\begin{equation}
		\cJ:\bigoplus_{k\in\bK,\hat n\in \widehat \bK}\bC^F\to \bigoplus_{k\in\bK,\hat n\in \widehat \bK} C(\bH/\bK), \qquad \cJ = \bigoplus_{k\in\bK,\hat n\in \widehat \bK} \overline{\hat n(k)} \cJ_{\hat n},
	\end{equation}
	where $C(\bH/\bK)$ is the set of continuous functions on $\bH/\bK$, and $\cJ_{\hat n}$ is the operator with kernel $J_{\hat n}$. That is,
	\begin{equation}
		\cJ_{\hat n}\varphi(y) := \sum_{\lambda\in F} J_{\hat n}(\lambda, y)\varphi(\lambda),\qquad \forall \varphi\in \bC^F.
	\end{equation}
\end{definition}

\begin{remark}
	The generalized Bessel functions depend on the choice of the section $\sigma$.
	Namely, if a different section $\sigma':\bH/\bK\to \bH$ is fixed we have that $\sigma'(y)=\phi(r)\sigma(y)$, and hence
	\begin{equation}
		J^{\sigma'}_{\hat n}(\lambda, y) = \hat n(r)\, J^{\sigma}_{\hat n}(\lambda, y).
	\end{equation}
\end{remark}

Let $\cF:\bC^\bK\to \bC^{\widehat\bK}$ be the Fourier transform over $\bK$, defined for $v\in \bC^\bK$ as
\begin{equation}
	\cF v(\hat k) = \frac1{\sqrt N} \sum_{h\in\bK} \overline{\hat k(h)} v(h), \qquad \forall\hat k\in\widehat\bK.
\end{equation}
Moreover, for any vector space $V$ let $\cP_V: \bC^\bK \otimes \bC^{\widehat \bK}\otimes V \to \bigoplus_{k\in\bK,\hat n\in \widehat \bK} V$ be the bijection defined by 
\begin{equation}\label{eq:P}
	\cP_V(a\otimes b\otimes v) 
	= (a_k b_{\hat n} v)_{k\in\bK,\hat n\in\widehat\bK}, \qquad \forall a\in \bC^\bK,\, \, b\in\bC^{\widehat \bK},\, v\in V.
\end{equation}
We then have the following.

\begin{theorem}\label{thm:general}
	The bijection $\cF_{\AP}^{-1}:\bC^\bK\otimes\bC^\bK\otimes\bC^F\to \AP_F(\bG)$ admits the following decomposition
	\begin{equation}
		\cF_{\AP}^{-1} = \Xi^{-1}\circ(\idty \otimes\cF^*\otimes\idty)\circ \cP^{-1}_{C(\bH/\bK)}\circ\cJ\circ\cP_{\bC^F}\circ(\idty\otimes\cF\otimes\idty).
	\end{equation}
	In particular, the Fourier-Bessel operator is a bijection onto its range.
\end{theorem}

\begin{proof}
	It is clear that it suffices to prove the statement for a basis of $\bC^\bK\otimes\bC^\bK\otimes\bC^F$ as, for example, 
	\begin{equation}
		B = \left\{ \delta_k\otimes \hat n\otimes \varphi\mid  k\in\bK,\,  \hat n\in\widehat\bK,\,\varphi\in\bC^F\right\}.
	\end{equation}
	Observe that $\cF\hat n = \delta_{\hat n}$ and that $\cP_{\bC^F}(\delta_k\otimes\delta_{\hat n}\otimes\varphi) = \delta_k\delta_{\hat n}\varphi$.
	Thus, 
	\begin{equation}
		[\cJ\otimes\cP_{\bC^F}\circ(\idty\otimes\cF\otimes\idty).(\delta_k\otimes\hat n\otimes\varphi)]_{h,\hat m} = 
		\begin{cases}
			\overline{\hat n(k)}\, \cJ_{\hat n}\varphi & \text{ if } k=h,\, \hat n=\hat m,\\
			0& \text{ otherwise}.
		\end{cases}
	\end{equation}
	Then, considering the inverse actions $\cF^*$ and $\cP^{-1}_{C(\bH/\bK)}$, we have
	\begin{equation}\label{eq:part1}
		(\idty \otimes\cF^*\otimes\idty)\circ \cP^{-1}_{C(\bH/\bK)}\circ\cJ\otimes\cP_{\bC^F}\circ(\idty\otimes\cF\otimes\idty).(\delta_k\otimes\hat n\otimes\varphi) = \delta_k\otimes\hat n\otimes (\overline{\hat n(k)}\cJ_{\hat n}\varphi).
	\end{equation}

	Let us compute, by Proposition~\ref{prop:ap-simpl},
	\begin{equation}\label{eq:part2}
		\begin{split}
			\Xi\circ\cF_{\AP}^{-1}[\delta_k\otimes \hat n\otimes \varphi] (h,r,y)
			&=\cF_{\AP}^{-1}[\delta_k\otimes \hat n\otimes \varphi] (h,\phi(r)\sigma(y))\\
			&= \sum_{\lambda\in F} \sum_{\ell\in\bK} \phi(\ell h r^{-1})\lambda(\sigma(y))\delta_k(h) \hat n(\ell) \varphi(\lambda)\\
			&= \delta_k(h) \overline{\hat n}(h) \hat n(r) \sum_{\lambda\in F} \sum_{s\in\bK} \phi(s^{-1})\lambda(\sigma(y)) \overline{\hat n}(s) \varphi(\lambda)\\
			&=  [\delta_k\otimes \hat n \otimes (\overline{\hat n(k)}\cJ_{\hat n} \varphi)](h,r,y),
		\end{split}
	\end{equation}
	where we applied the change of variables $s = \ell h r^{-1}$. Together with \eqref{eq:part1}, this completes the proof.
\end{proof}

Via the lift procedure described in the Section~\ref{sub:finite_dimensional_subspaces_of_almost_periodic_functions}, the above yields a similar result on $\AP_F(\bH)$.

\begin{corollary}\label{cor:general-H}
	Let us consider the restriction of the Fourier-Bessel operator given by 
	\begin{equation}
		\cJ_\bH: \bigoplus_{\hat n\in\widehat\bK}\bC^F\to \bigoplus_{\hat n\in\widehat\bK}C(\bH/\bK),
		\qquad
		\cJ_\bH = \bigoplus_{\hat n\in\widehat \bK}\cJ_{\hat n}.
	\end{equation}
	Then, the bijection $\cF_{\AP}^{-1}:\bC^\bK\otimes\bC^F\to \AP_F(\bH)$ admits the following decomposition
	\begin{equation}
		\cF_{\AP}^{-1} = \Xi^{-1}\circ(\cF^*\otimes\idty)\circ \cP^{-1}_{C(\bH/\bK)}\circ\cJ_{\bH}\circ\cP_{\bC^F}\circ(\cF\otimes\idty),
	\end{equation}
	where $\cP_{\bC^F}:\bC^{\widehat\bK}\otimes \bC^F\to \bigoplus_{\hat n\in\widehat\bK}\bC^F$ and $\cP_{C(\bH/\bK)}:\bC^{\widehat\bK}\otimes C(\bH/\bK) \to \bigoplus_{\hat n\in\widehat \bK}C(\bH/\bK)$ are the appropriate restrictions of the corresponding operators given by \eqref{eq:P}.
\end{corollary}

\begin{proof}
	It suffices to check the statement on the basis of $\bC^\bK\otimes\bC^F$ given by $\{\hat n\otimes \varphi\mid\hat n\in\widehat\bK,\,\varphi\in\bC^F\}$. Then, if $\hat\psi = \hat n\otimes\varphi$ corresponds to $\psi=\cF_{\AP}^{-1}\hat\psi\in \AP_F(\bH)$, and letting $\Psi\in\AP_F(\bG)$ be the lift of $\psi$, we have that $\hat\Psi = \delta_0\otimes\hat n\otimes\varphi$.
	Then, the statement follows by \eqref{eq:part1} and \eqref{eq:part2}.
\end{proof}

\begin{remark}\label{rmk:4}
	If a different lift from $\AP_F(\bH)$ to $\AP_F(\bG)$ is considered, the above corollary cannot be recovered. 
	This is easy to check, e.g., for the (left-invariant) lift $\Psi(k,\phi(k)x)=\psi(x)$. Indeed, in this case, if $\hat\psi = \hat n\otimes\varphi$ we have $\hat\Psi = \frac1N\sum_{k\in\bK} \delta_k\otimes\hat n\otimes\varphi$. Thus, by \eqref{eq:part2}
	\begin{multline}
		\Xi\circ\cF_{\AP}^{-1}\hat\psi(r,y) = 
		\sum_{h\in\bK} \Xi\circ\ev\hat\Psi(h,hr,y) \\
		=\sum_{h\in\bK}\frac1N\sum_{k\in\bK} \delta_k(h)\,\hat n(hr)\, \overline{\hat n(k)}\cJ_{\hat n}\varphi(y) = 
		[\hat n\otimes \cJ_{\hat n}\varphi](r,y).
	\end{multline}
	However, by \eqref{eq:part1},
	\begin{equation}
		\begin{split}
		(\cF^*&\otimes\idty)\circ  \cP^{-1}_{C(\bH/\bK)}  \circ\cJ_{\bH}\circ\cP_{\bC^F}  \circ(\cF\otimes\idty)\hat\psi(r,y)\\
		&=  \frac1N \sum_{h,k\in\bK}(\idty \otimes\cF^*\otimes\idty)\circ \cP^{-1}_{C(\bH/\bK)}\circ\cJ\circ\cP_{\bC^F}\circ(\idty\otimes\cF\otimes\idty).(\delta_k\otimes\hat n\otimes\varphi)(h,r,y) \\
		&= \frac1N\sum_{h,k\in\bK}[\delta_k\otimes\hat n\otimes (\overline{\hat n(k)}\cJ_{\hat n}\varphi)](h,r,y)\\
		&= \frac1N\left(\sum_{k\in\bK}\hat n(k)\right) [\hat n\otimes\cJ_{\hat n}\varphi](r,y).
		\end{split}
	\end{equation}
	Since $\frac1N\left(\sum_{k\in\bK}\hat n(k)\right) = \delta_0(\hat n)$, the above proves Corollary~\ref{cor:general-H} for functions of $\bC^\bK\otimes\bC^F$ independent on the first variable only. 

	The same reasoning shows that the approach used above cannot be extended to the case $\bG=SE(2)$, where $\bK$ is non-discrete.
\end{remark}

\subsection{Almost periodic interpolation}\label{sec:ap-interp}

In this section we apply (and slightly generalize) the results of the previous section to the problem of interpolating and approximating functions between two fixed grids in $\cS$ and $\bG$, respectively.
In particular, we are interested in finite sets $\tilde E\subset \bG$ that are invariant under the action of $\bK$ both on $\bG$ and on the $\bH$ component of $\bG$.
These sets are completely determined by finite sets $E\subset \bH/\bK$ in the following way:
\begin{equation}
	g\in \tilde E \iff \exists y\in E, h,k\in\bK \text{ s.t.\ } g = (k,\phi(h)\sigma(y)),
\end{equation}
where $\sigma:\bH/\bK\to\bH$ is a fixed section.
This identification allows to decompose $\bC^{\tilde E}\simeq\bC^\bK\otimes\bC^\bK\otimes\bC^E$.
Then, we let the \emph{sampling operator} $\sampl:\bC^\bG\to \bC^\bK\otimes\bC^\bK\otimes\bC^E$ to be
\begin{equation}
	\sampl \psi(k,h,y) = \psi(k,\phi(h)\sigma(y)).
\end{equation}
Finally, the \emph{evaluation operator} $\ev:\bC^\bK\otimes\bC^\bK\otimes\bC^F\to \bC^\bK\otimes\bC^\bK\otimes\bC^E$ is defined as $\ev = \sampl\circ\cF_{\AP}^{-1}$. That is, $\ev$ is the operator associating to each $\hat f$ the sampling on $\tilde E$ of the corresponding $\AP_F(\bG)$ function.

\begin{definition}
	Let $F\subset \widehat\cS$ and $E\subset \bH/\bK$ be two finite sets. The \emph{almost-periodic (AP) interpolation} of a function $\Psi:\bG\to \bC$ on the couple $(E,F)$ is the function $f\in \AP_F(\bG)$ such that $\ev \hat f = \sampl\Psi$.
	We say that the AP interpolation problem on $(E,F)$ is \emph{well-posed} if to each $\Psi:\bG\to \bC$ corresponds exactly one AP interpolation $f\in\AP_F(\bG)$.
\end{definition}

In practice, even if the AP interpolation problem is well-posed, one has to pay some attention. Indeed, the AP interpolation $f\in\AP_F(\bH)$ of $\psi:\bH\to\bC$ can oscillate wildly in between points of $E$. This can be observed in Section~\ref{sec:ap_interp_numerical}, where it is shown that this function behaves very badly w.r.t.\ small translations in space. (To this effect, see Figure~\ref{fig:ap-approx} in Section~\ref{sec:ap_interp_numerical}.) Thus, we introduce also the following weighted version of the AP interpolation problem.

\begin{definition}
	Fix a vector $d\in\bR^\bK\otimes \bR^\bK\otimes \bR^F$.
	The \emph{AP approximation} of a function $\psi:\bG\to \bC$ on the couple $(E,F)$ is the function $f\in\AP_F(\bG)$ such that $\hat f\in\bC^\bK\otimes \bC^\bK\otimes \bC^F$ satisfies
	\begin{equation}\label{eq:ap-approx}
		\hat f =  \arg\min_{v\in\bC^\bK\otimes \bC^\bK\otimes\bC^F} \langle d, v\rangle + \| \sampl \psi - \ev v \|^2.
	\end{equation}
\end{definition}

It is clear that, if the AP interpolation problem problem is well-posed, the AP interpolation coincides with the AP interpolation with $d=0$.

To apply the results of the previous section to this setting, let us introduce the discretization of the generalized Fourier-Bessel operator.

\begin{definition}\label{def:discr-bessel}
	The \emph{discrete Fourier-Bessel operator on the couple $(E,F)$} is the operator 
	\begin{equation}
		\cJ^{E}:\bigoplus_{k\in\bK,\hat n\in \widehat \bK}\bC^F\to \bigoplus_{k\in\bK,\hat n\in \widehat \bK} \bC^E,
		\qquad
		\cJ^E =\left[ \bigoplus_{k\in\bK,\hat n\in\widehat\bK}\Pi \right] \circ\cJ,
	\end{equation}
	where $\Pi:C(\bH/\bK)\to\bC^E$ is the sampling operator $\Pi\varphi = (\varphi(y))_{y\in E}$.
%
\end{definition}

The following is the main result of the paper.

\begin{theorem}\label{thm:discrete}
	The operator $\ev:\bC^\bK\otimes\bC^\bK\otimes\bC^F\to \bC^\bK\otimes\bC^\bK\otimes\bC^E$ decomposes as follows.
	\begin{equation}\label{eq:ev}
		\ev = (\idty \otimes\cF^*\otimes\idty)\circ \cP^{-1}_{\bC^E}\circ\cJ^E\circ\cP_{\bC^F}\circ(\idty\otimes\cF\otimes\idty).
	\end{equation}
	In particular, the AP interpolation problem on $(E,F)$ is well-posed if and only if $\cJ^E$ is invertible.
\end{theorem}

\begin{proof}
	It follows directly from the respective definitions that
	\begin{equation}
		\sampl\circ\,\Xi^{-1} = \idty\otimes \idty\otimes \Pi.
	\end{equation}
	Thus, by definition of $\ev$, of $\cJ^E$, and Theorem~\ref{thm:general}, the statement is equivalent to
	\begin{equation}
		(\idty\otimes\idty\otimes\Pi)\circ(\idty\otimes\cF^*\otimes \idty)\circ \cP_{C(\bH/\bK)}^{-1} = 
		(\idty \otimes\cF^*\otimes\idty)\circ \cP^{-1}_{\bC^E}\circ\bigoplus_{k\in\bK,\hat n\in\widehat\bK}\Pi.
	\end{equation}
	Since, up to changing the identity operators, $(\idty\otimes\idty\otimes\Pi)$ and $(\idty\otimes\cF^*\otimes \idty)$ commute, this reduces to
	\begin{equation}
		(\idty\otimes\idty\otimes\Pi)\circ \cP_{C(\bH/\bK)}^{-1} = 
		 \cP^{-1}_{\bC^E}\circ\bigoplus_{k\in\bK,\hat n\in\widehat\bK}\Pi.
	\end{equation}
	Finally, the above holds, as can be easily seen by testing it on functions of the type $(\delta_h(k) \delta_{\hat m}(\hat n)\varphi)_{k\in\bK,\hat m\in\widehat\bK}$, for $h\in\bK$, $\hat n\in\widehat\bK$, and $\varphi\in C(\bH/\bK)$.
\end{proof}

It is clear that restricting the evaluation and sampling operators on vectors of the form $\delta_0\otimes v\otimes w$ allows to define the AP interpolation and approximation of functions $\psi:\bH\to\bC$.
In particular, the same arguments used in Corollary~\ref{cor:ap-simpl-H}, allow to prove the following.

\begin{corollary}\label{cor:discrete-H}
	Let us consider the restriction of the discrete Fourier-Bessel operator given by 
	\begin{equation}
		\cJ_\bH^E: \bigoplus_{\hat n\in\widehat\bK}\bC^F\to \bigoplus_{\hat n\in\widehat\bK}\bC^E,
		\qquad
		\cJ^E_\bH =\left[ \bigoplus_{\hat n\in\widehat\bK}\Pi \right]\circ \cJ_\bH.
	\end{equation}
	Then, $\ev:\bC^\bK\otimes\bC^F\to \bC^\bK\otimes\bC^E$ admits the following decomposition
	\begin{equation}
		\ev = (\cF^*\otimes\idty)\circ \cP^{-1}_{\bC^E}\circ\cJ_{\bH}^E\circ\cP_{\bC^F}\circ(\cF\otimes\idty),
	\end{equation}
	where $\cP_{\bC^F}:\bC^{\widehat\bK}\otimes \bC^F\to \bigoplus_{\hat n\in\widehat\bK}\bC^F$ and $\cP_{\bC^E}:\bC^{\widehat\bK}\otimes \bC^E \to \bigoplus_{\hat n\in\widehat \bK}\bC^E$ are the appropriate restrictions of the corresponding operators given by \eqref{eq:P}.

	In particular, the AP interpolation problem on $(E,F)$ is well-posed if and only if $\cJ_{\bH}^E$ is invertible.
\end{corollary}

\section{Application to image processing}\label{sec:image}

Here, we particularize the results of the previous section to the almost-periodic interpolation of functions $\psi:\bR^2\to \bC$ on a spatial grid $\tilde E$ and a frequency grid $\tilde F$. These grids are assumed to be invariant under the action of $\bZ_N$ on $\bR^2$, given by the rotations $\{R_{\frac{2\pi}Nk}\}_{k\in\bZ_N}$. This is indeed a particular case of Corollary~\ref{cor:discrete-H}.

In this setting, we can naturally identify $\bR^2/\bZ_N$ with the slice $\cS_N=\{\rho e^{i\alpha}\mid \rho>0,\, \alpha \in[0,2\pi/N)\}$, thus fixing a choice for the section $\sigma$ and the map $\Xi$ introduced in Section~\ref{sub:finite_dimensional_subspaces_of_almost_periodic_functions}.
Clearly, the same is true for the set $\widehat\cS$ of frequencies with trivial stabilizer subgroup.
Since $\tilde E$ is rotationally invariant under discrete rotations in $\bZ_N$, we represent any element of $x\in \tilde E$ as a couple $(n,y)\in \bZ_N\times E$, where $E\subset \cS_N$, by letting
\begin{equation}
	x  = R_{\frac{2\pi}N n}y, \qquad y\in E, \, n\in\bZ_N.
\end{equation}
The same can be done for any $\Lambda\in\tilde F$, with $(n,\lambda)\in \bZ_N\times F$ and $F\subset \cS_N$.
Moreover, considering polar coordinate $y=\rho e^{i\alpha}$ and $\lambda=\xi e^{i\omega}$, letting $P,Q\in\bN$ be the respective cardinalities of $E$ and $F$, we will exploit the identifications 
\begin{gather}\label{eq:identification}
	E = \left\{(m,\rho_j e^{i\alpha_j})\mid n=0,\ldots, N,\, j=0,\ldots, P \right\},\\
	F = \left\{(n,\xi_k e^{i\omega_k})\mid m=0,\ldots, N,\, k=0,\ldots, Q \right\}.
\end{gather}

As in the previous section, the sampling of a function $\varphi:\bR^2\to \bC$ on $\tilde E$ is given by the sampling operator, $\sampl: \varphi\mapsto\sampl\varphi \in \bC^N\otimes\bC^E$. On the other hand, the evaluation operator $\ev:\bC^N\otimes\bC^F\to \bC^N\otimes\bC^E$ associates to $\hat f\in\bC^N\otimes\bC^E$ the sampling on $\tilde E$ of the AP function $f:\bR^2\to \bC$ of the form 
\begin{equation} \label{eq:interpol}
	f(x) = \sum_{\lambda\in \tilde F} \hat f(\lambda) e^{i\langle \lambda, x\rangle}, \qquad \forall x\in \tilde E.
\end{equation}

Recall that $\bZ_N\simeq\widehat{\bZ_N}$ and $\bR^2\simeq\widehat{\bR^2}$. In the following we will let $\hat n(k) = e^{i\frac{2\pi}N \hat n k}$ and $\lambda(x) = e^{i\langle\lambda, x\rangle}$. In particular, in polar coordinates the latter becomes
\begin{equation}
	\lambda(x) = e^{i\lambda\rho\cos(\alpha-\omega)},\qquad \text{if }x = \rho e^{i\alpha} \text{ and } \lambda=\xi e^{i\omega}.
\end{equation}
Then, direct computations yield.

\begin{proposition}
	The generalized Bessel function on $SE(2,N)=\bZ_N\ltimes\bR^2$ of parameter $\hat n\in\widehat\bZ_N$ is
	\begin{equation}
		J_{\hat n}(\lambda,y) = \sum_{r= 0}^{N-1} e^{i\xi\rho \cos\left(\alpha-\omega+\frac{2\pi}Nr\right)-\frac{2\pi}N\hat n r}, \qquad \text{if }x = \rho e^{i\alpha}\in \cS_N \text{ and } \lambda=\xi e^{i\omega}\in\cS_N.
	\end{equation}
	Moreover, the restriction of the discrete Fourier-Bessel operator $\cJ^E_{\bR^2}$ is the block-diagonal operator $\cJ^E_{\bR^2} = \bigoplus_{\hat n=0}^{N-1}\cJ_{\hat n}$, where $\cJ_{\hat n}: \bC^F\to \bC^E$ is given by the matrix,
	\begin{equation}
		 (\cJ_{\hat n})_{k,j} = J_{\hat n}(\xi_k e^{i\omega_k},\rho_j e^{i\alpha_j}).
	\end{equation}
\end{proposition}

\begin{remark}
	Generalized Bessel functions on $SE(2,N)$ only depend on the product $\xi\rho$ and on the difference $\alpha-\omega$. 
	Since $\alpha,\omega\in[0,2\pi/N)$, it is clear that, for $N\rightarrow+\infty$, the generalized Bessel functions converge to the usual ones:
	\begin{equation}
		J_{\hat n}(\xi,\rho) \xrightarrow{N\to+\infty} 2\pi i^{\hat n} J_{\hat n}(\xi\rho).
	\end{equation}
\end{remark} 

As a consequence of the above result and Corollary~\ref{cor:discrete-H}, we have the following.
\begin{corollary}\label{cor:ap-r}
The sampling of $f\in\AP_F(\bR^2)$ is connected with $\hat f$ by
\begin{equation}
\big((\cF\otimes\idty)\sampl f\big)_{\hat n,j} = \sum_{k=0}^{Q-1} J_{\hat n}(\lambda_k,y_j)\big((\cF\otimes\idty)\hat f\big)_{\hat n, k}.
\end{equation}
In particular, the AP interpolation problem on $(E,F)$ is well-posed if and only if all the matrices $\cJ_{\hat n}^E$ are invertible. 
\end{corollary}

\begin{proposition}\label{prop:least-squares}
	Let $P\ge Q$. Then, for a given weight vector $d\in\bR^N\otimes\bR^E$, the AP approximation of a function $\psi:\bR^2\to \bC$ on the couple $(E,F)$ is the function $f\in\AP_F(\bR^2)$ such that $\hat f = (\cF^*\otimes\idty) w$, where
	\begin{equation}
		\left(\cJ_{\hat n}^*\circ\cJ_{\hat n}- \diag_{i} d_{\hat n,i}^2\right) w_{\hat n,\cdot}  = \cJ^*_{\hat n} [(\cF\otimes \idty)\sampl\psi]_{\hat n,\cdot}
		\quad\text{for any}\quad \hat n=0,\ldots,N-1.
	\end{equation}
\end{proposition}

\begin{proof}
	From Corollary~\ref{cor:ap-r}, the definition of AP approximation, and the fact that $\cF\otimes\idty$ is an isometry, we have that
	\begin{equation}\label{eq:ap-approx}
		(\cF\otimes\idty) \hat f 
			=  \arg\min_{v\in\bC^N\otimes\bC^F} \langle d, v\rangle + \| \cP_{\bC^E}\circ(\cF\otimes \idty) \sampl \psi - \cJ \circ \cP_{\bC^F} v \|^2.
	\end{equation}
	In particular, this decomposes for $\hat n\in\{0,\ldots,N-1\}$ as
	\begin{equation}
		(\cF\otimes\idty) \hat f_{\hat n,\cdot} = \arg\min_{v_{\hat n}\in\bC^F} \langle d_{\hat n,\cdot},v_{\hat n}\rangle + \|[(\cF\otimes\idty)\sampl\psi]_{\hat n,\cdot} - \cJ_{\hat n} v_{\hat n}\|^2.
	\end{equation}
	The statement then follows by the standard formula for solving complex least-square problems.
\end{proof}

\begin{remark}
	In numerical experiments, we always found the matrix conditioning of the matrices $\cJ_{\hat n}$ to be very good.
	Moreover, these same experiment seem to suggest this conditioning to be connected with the smallest distance between elements in $E$ and in $F$. 
\end{remark}

\subsection{Computational cost}

\begin{proposition}\label{prop:comp-cost}
	Given $\sampl\psi$, after a prefactorization of $\cJ$ of computational cost $\cO(N Q^3)$, the computational cost of the AP approximation is 
	\begin{equation}\label{eq:comp-cost}
		\cC = N\left(\frac{Q^2}2 + 10\max\{P,Q\}\log N\right).
	\end{equation}
	Moreover, this operation can be parallelized on $N$ processors, yielding an effective cost of
	\begin{equation}\label{eq:comp-cost-parallel}
		\cC_P = \frac{Q^2}2 + 10\max\{P,Q\} N\log N.
	\end{equation}
\end{proposition}

\begin{proof}
	Let $v\in\bC^N\otimes \bC^E$ and denote $\tilde v_{\hat n} = (\cF\otimes\idty)v_{\hat n, \cdot}$. Similarly, for $w=\ev(v)$ let $\tilde w_{\hat n,\cdot} = (\cF\otimes\idty)w_{\hat n,\cdot}$.  The computational cost to evaluate $(\cF\otimes\idty)v$ and to pass from the $\tilde w_{\hat n}$'s to $w$ is of $5 P N\log N$ and $5 QN\log N$ FLOPs, respectively. By Proposition~\ref{prop:least-squares}, solving \eqref{eq:ap-approx} amounts to solve, the following problems
	\begin{equation}
		(\cJ_{\hat n}^*\circ\cJ_{\hat n}-  \diag_{i}d_{\hat n,i}^2) \tilde v_{\hat n} = \cJ^*_{\hat n} \tilde w_{\hat n} \qquad \hat n = 0,\ldots,N-1.
	\end{equation}
	Up to a prefactorization of the matrices $\cJ_n^*\circ\cJ_n-\diag_{i}d_{\hat n,i}^2$, with a computational cost of $\cO(Q^3)$, solving each of the above systems has a computational cost of $Q^2/2$ FLOPs. All together this yields the (non-parallelized) final cost of \eqref{eq:comp-cost}
	Since the solution of the systems is independent for each $n$, it can be parallelized, yielding to the cost \eqref{eq:comp-cost-parallel}, for $N$ processors.
\end{proof}

\begin{corollary}
	Let $G = \{ \rho_j e^{i \frac{2\pi}K k}\mid j=1,\ldots,R ,\, k = 1,\ldots,K \}$ be a fixed polar grid. Then, for $\tilde E=\tilde F=G$, the best choice for AP approximation is $N=\sqrt{{|G|}/{10}}$. This yields a prefactorization complexity of $\cO(|G|^{3/2})$ and a computational complexity of
	\begin{equation}\label{eq:polar-comp-cost}
		\cC =  \cO(|G|^{3/2}\log |G|)
		\quad\text{and}\quad
		\cC_P = 5|G|\left( 1 + \log\frac{|G|}{10}\right).
	\end{equation}
\end{corollary}

\begin{proof}
	Clearly, $P=Q$. 	Then, a simple computation, using that $|G|=|E|=QN$, yields
	\begin{equation}
		\cC_P = \cC_P(M) = \frac{|G|^2}{2N^2} + 10 |G|\log N.
	\end{equation}
	The above expression attains its minimum at $N=\sqrt{|G|/10}$, which gives the cost in \eqref{eq:polar-comp-cost}. To complete the proof for $\cC$ it suffices to observe that $\cC = N\,\cC_P$.
\end{proof}

\begin{remark}
	The above shows that, once parallelized, the complexity of the AP approximation is the same as the polar Fourier transform algorithm presented in \cite{polarFT}.
\end{remark}

%% file: bispectrum.tex

In this chapter we present a framework for pattern recognition on groups, based on Fourier invariants. Our aim is to give an effective procedure for discriminate functions  up to the action of the left-regular representation of some group.

Let $\bG$ be an unimodular group.
The maps $f\mapsto I_f$ are called \emph{invariants} for $\bG$ if $I_f = I_{\Lambda(a)f}$ for any $a\in\bG$.
A choice of invariants is \emph{complete} if it separates the orbits of $\Lambda$.
That is, if for any $f,g\in L^2(\bG)$ we have
\begin{equation}
  I_f=I_g \iff f= \Lambda(a) g \text{ for some }a\in\bG.
\end{equation}
A choice of invariants is \emph{weakly complete} if the above is generically true on $L^2(\bG)$, i.e., if it holds for some residual subset of square-integrable functions.

In the following, we will first present the simplest Fourier-based invariants that we will focus on: the \emph{power spectrum} and the \emph{bispectrum} invariants. Although it is easy to show that, even in the simplest case where the group $\bG$ is abelian,  the power spectrum invariants are not weakly complete, the aim of the first part of the chapter is to prove that, when considered together with bispectral invariants, they are weakly complete. In particular, we show their completeness on the residual subset $\cG\subset L^2(\bG)$ of functions whose Fourier transform is invertible on an open and dense subset of $\widehat\bG$. For pedagogical purposes, we present the proof of the completeness first in the case where $\bG$ is abelian, which exploits Pontryagin duality, then in the case where $\bG$ is compact, exploiting Chu (or Tannaka) duality, and finally in the most general case of a semi-discrete product, as introduced in Section~\ref{sec:repr-semidirect}.

In the second part of the chapter we consider the problem of discriminating functions in $L^2(\bH)$ under the action of the semi-direct product $\bG=\bK\ltimes \bH$, as given by its quasi-regular representation $\pi$. The natural idea here is to fix an (injective) lift operator $L:L^2(\bH)\to L^2(\bG)$ and, given two functions $f,g,\in L^2(\bH)$, to compare the invariants for their lifts $Lf,Lg\in L^2(\bG)$. If the lift intertwines correctly the quasi-regular representation on $L^2(\bH)$ with the regular representation on $L^2(\bG)$, this is enough to solve the discrimination problem. 

We first show that, if the lift $L$ is left-invariant or cyclic, the computation of the equality of these invariants can be reduced to computations based only on the abelian Fourier transform of $f$ and $g$ on $\bH$. This allows to observe that it is indeed enough to compare the traces of the bispectral invariants. Later, we prove that bispectral invariants are indeed weakly complete for regular cyclic lifts, while this is not the case for left-invariant lifts. Indeed, if $L$ is left-invariant, Corollary~\ref{cor:non-inver-li} shows that $Lf$ can never be in the completeness set $\cG$ identified before.

The above observation yield us to consider stronger invariants, the \emph{rotational power spectrum} and \emph{rotational bispectrum} invariants. We then prove the main theorem of the chapter: Theorem~\ref{cor:reduced-rot-inv} that is, that these invariants, up to a centering operation, are weakly complete on lifts of functions in $L^2(\bH)$. We also show how this result can be extended to functions in $L^2_\bR(\bH)$, and how it can be strengthened if $\bH=\bR^2$ and one is interested only with compactly supported functions. 

We conclude the chapter by presenting the extension of this theory to almost-periodic functions. In particular, we prove that, if $\bG$ is non-compact, the bispectral invariants are never weakly complete already for Besicovitch almost-periodic functions on $\bG$. On the other hand, we show that rotational bispectral invariants are weakly-complete on the separable subspaces $\AP_F(\bH)$, for $F\subset\widehat \bH$ countable, introduced in Section~\ref{sub:finite_dimensional_subspaces_of_almost_periodic_functions}.

\section{Power spectrum and bispectral invariants}

  The simplest invariants that one can consider are the following.

  \begin{definition}
    The \emph{(power) spectrum invariants} of $f\in L^2(\bG)$ is the set $\PS_f = \{ \PS_f(T) \mid T\in \supp \mu_{\hat \bG} \}$, where
    \begin{equation}
      \PS_f(T) = \widehat f(T)\circ \hat f (T)^*.
    \end{equation}
  \end{definition}

  The power spectrum invariants are not weakly complete even in the simple case of $\bG=\bR$.
  In this case $\PS_f(\lambda)=|\hat f(\lambda)|^2$ for any $\lambda\in \supp \mu_{\hat \bG}=\widehat \bR$, and it is easy to build a counterexample.
  Indeed, it suffices to fix some $\phi:\widehat \bR\to\bR$ and consider the function $g = \cF^{-1}(e^{i\phi(\lambda)}\hat f(\lambda))$.
  Clearly, $g$ is such that $\PS_f=\PS_g$ but $f=\Lambda(a)g$ if and only if $\phi(\lambda)=a\lambda$.

  Thus, we need to consider richer sets of invariants, as the following.

  \begin{definition}
    The \emph{(power) bispectral invariants} of $f\in L^2(\bG)$ is the set $\BS_f = \{ \BS_f(T_1,T_2) \mid T_1,T_2\in \supp \mu_{\hat \bG} \}$, where
    \begin{equation}
      \label{eq:bispectral-inv}
      \BS_f(T_1,T_2) = \widehat f(T_1)\otimes \hat f(T_2) \circ \hat f (T_1\otimes T_2)^*.
    \end{equation}
  \end{definition}
  
  \begin{remark}
    As shown in \cite{Kakarala2009}, the bispectral invariants can be derived as the Fourier transform of the triple correlation function of $f\in L^2(\bG)$, which is
    \begin{equation}
      A_3(a,b) = \int_\bG \overline{f(g)}f(ga)f(gb)\,dg.
    \end{equation}
    Triple correlation is useful in signal processing and in music theory \cite{Dubnov1997}.
  \end{remark}

  A priori, to insure weak completeness, one needs to consider both  power spectrum and bispectral invariants, although we will see that in most cases, and in particular in the case $\bG = SE(2,N)$, we have that $\BS_f\supset \PS_f$.

\section{Weak completeness of the spectral invariants}

  In this section we will prove the weak completeness of the bispectral invariants in three different cases.
  In all these situations, we indeed prove that the bispectral invariants are complete on the following (residual) subset of $L^2(\bG)$:
  \begin{equation}
    \label{eq:general-def-G}
    \cG \coloneqq \left\{ f\in L^2(\bG) \mid  
    \begin{array}{l}
          f\text{ has compact support and } \hat f(T) \text{ is invertible on}\\
          \text{an open and dense subset of } \supp\hat\mu_\bG
    \end{array}
    \right\}
  \end{equation}

  The following result guarantees that in the cases under consideration $\cG$ is a sufficiently large set.

  \begin{theorem}
    \label{thm:genericity-cG}
    The following hold:
    \begin{enumerate}
      \item If $\bG$ is a connected abelian Lie group, $\cG$ contains all non-zero compactly supported functions of $L^2(\bG)$.
      \item If $\bG$ is compact and separable, $\cG$ is residual.
      \item If $\bG=\bK\ltimes\bH$, under the assumptions of Section~\ref{sec:repr-semidirect} with $\bH$ connected Lie group, $\cG$ is residual.
      If, moreover, $\bH=\bR^N$ then $\cG$ is open and dense in the set of compactly supported functions of $L^2(\bG)$.
    \end{enumerate}
  \end{theorem}

  \begin{proof}
    \emph{Case 1: } If $\bG$ is a connected abelian Lie group, it holds that $\bG\cong\bR^N\times\bT^M$, and so $\widehat\bG\cong\bR^N\times\bZ^M$.
    By the Paley-Wiener Theorem, for any $f\in L^2(\bG)$ with compact support the function $\hat f(\cdot,k)$ is analytic for any $k\in\bZ^M$.
    Thus, $f\in\cG$ if and only if $\hat f(\lambda,k)\neq 0$ for any $k\in\bZ^M$ and an open-dense subset of $\lambda\in\widehat{\bR^N}$, property which is satisfied by every non-zero analytic function.
    This proves that $\cG$ contains all non-zero compactly supported functions of $L^2(\bG)$.

    \emph{Case 2: } If $\bG$ is compact separable, then $\widehat \bG$ is countable and discrete (see, e.g., \cite{Dixmier1977}) and thus for any fixed $T$ the set of those $f\in L^2(\bG)$ such that $\hat f(T)$ is invertible is open-dense. Moreover, 
    \begin{equation}
      \cG = \left\{ f\in L^2(\bG) \mid \hat f(T) \text{ is invertible for all } T\in\widehat\bG\right\}.
    \end{equation}
    Thus, $\cG$ is the countable intersection of open and dense sets and hence residual.

    \emph{Case 3: }
    Let $f\in L^2(\bG)$ and denote $f_k:=f(k,\cdot)$ for any $k\in \bK$.
    Since $\bH \cong \bR^N\times \bT^M$ and $\widehat\bH \cong \widehat{\bR^N}\times \bZ^M$, if $f$ is compactly supported the functions $\lambda\in\widehat{\bR^N}\mapsto \widehat f_k(\lambda,h)$ are analytic for any $h\in\bZ^M$.

    We now prove that, if $\bH \cong \bR^N$, the set $\cG$ is open and dense in the set of compactly supported functions.
    By Proposition~\ref{prop:FT-semidirect} the entries of $\hat f(T^{\lambda})$ are obtained by evaluations of $\widehat {f_k}$, and hence $\lambda\mapsto \det \hat f(T^{\lambda})$ is analytic.
    In particular, $f\in\cG$ if and only if there exists $\lambda_0$ such that $\det \hat f(T^{\lambda_0})\neq0$.

    Observe that $\cG\neq\varnothing$. 
    Indeed, it suffices to fix $\lambda_0\in\widehat{\bR^N}$ and consider $f$ such that $f_k\equiv 0$ for any $k\neq e$ and $f_e$ such that $\widehat{f_e}(\phi(h)\lambda)\neq 0$ for all $h\in\bK$.
    This ensures that $\hat f(T^{\lambda_0})$ is invertible and hence that $f\in\cG$.

    To prove that $\cG$ is dense, let us fix $f\in\cG$ and consider $g\notin \cG$. 
    Then, for some $\lambda_0$ such that $\hat f(T^{\lambda_0})$ is invertible it holds that $\widehat {g+\varepsilon f}(T^{\lambda_0})$ is invertible for any $\varepsilon>0$ sufficiently small\footnote{This follows from the linearity of the Fourier transform and the analyticity of the map $\varepsilon\mapsto \det(A+\varepsilon B)$ where $A,B$ are matrices.}.
    Hence, $g+\varepsilon f\in\cG$ for these $\varepsilon$'s, which entails $g\in\overline\cG$.
    
    Let us now prove that $\cG$ is open.
    Fix $f\in\cG$ and consider a sequence of compactly supported functions $f_n\rightarrow f$ in $L^2(\bG)$.
    This implies that $\hat f_n\rightarrow \hat f$ in $L^2(\bG)$ and thus in measure.
    In particular, $\hat {f_n} (T^\lambda)\rightarrow\hat f(T^\lambda)$ in measure and hence for any $n$ sufficiently big there exists $\lambda_0$ such that $\hat {f_n} (T^{\lambda_0})\neq 0$.
    This implies that $f_n\in\cG$ for any $n$ sufficiently big, and hence that $\cG$ is open.

    The result for the case $\bH \cong \bR^N\times \bT^M$ follows by considering the sets $\cG_h$, $h\in\bZ^M$, of compactly supported functions whose Fourier transforms $\hat f(T^{(\lambda,h)})$ are invertible for an open and dense set of $\lambda\in\widehat{\bR^N}$.
    The same arguments as above can be used to prove that $\cG_h$ is open and dense.
    Finally, since $\cG=\bigcap_{h\in\bZ^M} \cG_h$, this proves that $\cG$ is residual.
  \end{proof}

  \subsection{Abelian group}

  Let $\bG$ be an abelian group.
  Then all its representations are one dimensional and the Plancherel measure is the Haar measure on the character group $\widehat\bG$.
  In this case, the set $\cG$ defined in \eqref{eq:general-def-G}, becomes
  \begin{equation}
    \cG = \left\{ f\in L^2(\bG) \mid  
    \begin{array}{l}
          f\text{ has compact support and } \hat f(\lambda)= 0\\
          \text{ for a discrete subset of } \lambda\in\widehat\bG
    \end{array}
    \right\}
  \end{equation}

  Simple computations  shows that
  \begin{equation}
    \PS_f(\lambda) = |\hat f (\lambda)|^2\qquad\text{and}\qquad
   \label{eq:inv-abelian}
    \BS_f(\lambda_1,\lambda_2) = \hat f(\lambda_1)\hat f(\lambda_2) \bar{\hat f}(\lambda_1+\lambda_2).
  \end{equation}

  In this case, we have that $\BS_f\supset \PS_f$ for any $f,g\in L^1(\bG)\cap L^2(\bG)$.
  Indeed, observe that choosing $\lambda_1=\lambda_2=\hat o$ in the bispectral invariants yields $\avg(f)|\avg(f)|^2 = \avg(g)|\avg(g)|^2 $, which implies that $\avg(f)=\avg(g)$.
  This shows that $\BS_f(\lambda_1,0)=\PS_f(\lambda_1)$. 

  \begin{theorem}
    \label{thm:inv-compl-abelian}
    The bispectral invariants are complete on the set $\cG$. 
    In particular, if $\bG$ is either compact separable or a connected Lie group, they are weakly complete on compactly supported functions.
  \end{theorem}

  \begin{proof}
    The second part of the statement is a direct consequence of Theorem~\ref{thm:genericity-cG}.
    Let then $f,g\in\cG$ be such that $\BS_f=\BS_g$.
    Since this implies that $\PS_f=\PS_g$, we have that $|\hat f|=|\hat g|$.
    Thus $\hat f$ and $\hat g$ vanish on the same set $\cI$.
    Moreover, observe that since $f$ and $g$ are compactly supported, their Fourier transforms $\hat f$ and $\hat g$ are continuous.

    Let $u(\lambda)=\hat g(\lambda)/\hat f(\lambda)$ for any $\lambda\in\cI$.
    Since $u$ is the ratio of two continuous functions vanishing only on a discrete set, it is measurable. Moreover, $|u|\equiv1$ by the equality of the power spectrum invariants.
    Then, by the equality of the bispectral invariants and \eqref{eq:inv-abelian} it follows that $u$ satisfies
    \begin{equation}
      u(\lambda_1+\lambda_2) = u(\lambda_1)u(\lambda_2).
    \end{equation}
    This implies that $u$ is a measurable character of $\hat \bG$ and thus, by the well-known result \cite[Theorem 22.17]{Hewitt1963}, has to be continuous.
    By Pontryagin duality this proves the existence of $a\in\bG$ such that $u(\lambda)=\lambda(a)$.
    Thus, we have proved that $\hat f(\lambda) = \lambda(a) \hat g(\lambda)$, which by Theorem~\ref{thm:FT-fund-prop} implies that $f = \Lambda(a)g$, completing the proof.
  \end{proof}

  In the case $\bG=\bR^n$ the above result can be strengthened.

  \begin{corollary}
    The bispectral invariants on $\bR^n$ are complete on compactly supported functions of $L^2(\bR^2)$.
  \end{corollary}

  \begin{proof}
    It suffices to observe that by the Paley-Wiener Theorem Fourier transforms of compactly supported functions are analytic.
    Since analytic non-zero functions have a discrete zero-level set, this implies that the set $\cG$ of Theorem~\ref{thm:inv-compl-abelian} coincide with all the considered functions.
  \end{proof}

  \subsection{Compact group}

  Let $\bG$ be a compact separable group.
  In this case the set of irreducible unitary representations is endowed with the discrete topology and thus the set $\cG$ defined in \eqref{eq:general-def-G} becomes
  \begin{equation}
    \label{eq:def-cG-compact}
    \cG = \left\{ f\in L^2(\bG)\mid \hat f (T) \text{ is invertible for any } T\in\widehat\bG\right\}.
  \end{equation}



  With the same arguments used for abelian groups, it is possible to show that $\BS_f\supset \PS_f$.
  We then have the following.

  \begin{theorem}
    \label{thm:compact-completeness}
    The bispectral invariants are weakly complete on $\bG$. 
    More precisely, they discriminate on the above defined set $\cG$.
  \end{theorem}

  \begin{proof}
    The fact that $\cG$ is residual in $L^2(\bG)$ follows from Theorem~\ref{thm:genericity-cG}.
    Let $f,g\in\cG$ be such that $\BS_f=\BS_g$.
    The idea of the proof is to show that this allows to build a quasi-representation $u$ of $\widehat\bG$ such that $\hat f(T) \circ u(T) = \hat g(T)$.
    The conclusion then will follow by Chu (or Tannaka) duality.

    \begin{itemize}
      
    \item
    \emph{Step 1 - Definition of the candidate quasi-representation:}
    Since $\BS_f=\BS_g$ implies that $\PS_f=\PS_g$, it holds that $\hat f(T)\circ \hat f(T)^* = \hat g(T)\circ \hat g(T)^*$ for all $T\in\widehat \bG$.
    By invertibility of $\hat f(T)$ we can define $u(T) = \hat f(T)^{-1}\circ \hat g(T)$ for any $T\in\supp \mu_{\widehat\bG}$.
    Moreover, since $\BS_f=\BS_g$ implies the equality of the bispectral invariants \eqref{eq:bispectral-inv} for any (non-necessarily irreducible) unitary representation, the same definition holds for any representation in $\repr(\bG)$, the Chu dual of $\bG$.

    \item
    \emph{Step 2 - $u$ is indeed a quasi-representation:}
    Let us start by checking that $u(T)$ is unitary.
    This follows from the equality of the first invariants.
    Indeed,
    \[
    u(T)^* u(T) = \hat g(T)^* \left(\hat f(T) \hat f(T)^{*}\right)^{-1}\hat g(T) = \idty.
    \]
    We now check the properties of the quasi-representations.

    \begin{enumerate}
      \item \emph{Commutation with the direct sum:} This follows from the definition of $u(T)$ and the analogous property of the Fourier transform.
      \item \emph{Commutation with the tensor product:} From the equality of the bispectral invariants and the definition of $u$, for all $T_1,T_2\in\widehat\bG$ we obtain
      \begin{multline}
      \hat f(T_1)\otimes \hat f(T_2)\circ \hat f(T_1 \otimes T_2)^* = \\ 
      \hat f(T_1)\otimes \hat f(T_2)\circ u(T_1)\otimes u(T_2)\circ u(T_1 \otimes T_2)^*  \circ  \hat f(T_1 \otimes T_2)^*.
      \end{multline}
      Since $\hat f(T)$ is invertible for all $T\in\widehat\bG$, and hence also for their tensor products, this and the unitarity of $u$ yield $u(T_1)\otimes u(T_2)=u(T_1\otimes T_2)$.
      \item \emph{Commutation with the equivalences:} Again, this follows from the definition of $u(T)$ and the analogous property of the Fourier transform.
      \item \emph{Continuity:} The sets $\rep_n(\bG)$ are discrete, due to compactness of $\bG$, hence this is trivial.
    \end{enumerate}

    Thus, $u$ is a quasi-representation of $\widehat \bG$.

    \item
    \emph{Step 3 - Chu duality:}
    By Theorem~\ref{thm:tannaka}, the group $\bG$ has the Chu duality property.
    Thus, being $u$ a quasi-representation, there exists $a\in\bG$ such that for all $T\in\widehat\bG$ it holds $u(T)=T(a)$.
    Then, $\hat g(T)= \hat f(T)\circ T(a)$ for all $T\in \supp \mu_{\widehat\bG}\subset\widehat\bG$ which, by Theorem~\ref{thm:FT-fund-prop} implies that $f=\Lambda(a)g$, completing the proof. 
    \end{itemize}
  \end{proof}

  \subsection{Moore groups that are semi-direct products}
  \label{sssec:moore-semi-direct}

  We now consider the semi-direct product $\bG=\bK\ltimes\bH$ introduced in Section~\ref{sec:repr-semidirect}, where $\bK$ is finite with $N$ elements.
  Since $L^2(\bK) \cong \bC^N$, the group $\bG$ is a Moore group and the set $\cG$ becomes
  \begin{equation}
    \label{eq:def-cG-moore}
    \cG = \left\{ f\in L^2(\bG) \mid 
       \begin{array}{l}
          f\text{ has compact support and } \hat f(T^\lambda) \text{ is invertible on}\\
          \text{an open and dense subset of } \widehat\bH
      \end{array}
      \right\}.
  \end{equation}


  Due to the explicit structure of the irreducible representations given in Section~\ref{sec:repr-semidirect}, we can compute the expression of the invariants.

  \begin{proposition}
    \label{prop:moore-inv-expression}
    Let $f\in L^2(\bG)$.
    Then, for all $\lambda,\lambda_1,\lambda_2\in\widehat\bH\setminus\{\hat o\}$ and any $k,\ell\in\bK$,
    \begin{gather}
      \PS_f(T^\lambda) = \left( \sum_{h\in\bK} \hat f(T^\lambda)_{i,h} \,\overline{\hat f(T^\lambda)_{j,h}}  \right)_{i,j\in\bK},\\
      (A\circ \BS_f(T^{\lambda_1}\otimes T^{\lambda_2})\circ A^*)_{k,\ell} = \left( \sum_{h\in\bK} {\hat f(T^{\lambda_1})_{i,h}}\, {\hat f(T^{\lambda_2})_{i-\ell,h-k}} \, \overline{\hat f(T^{\lambda_1+\phi(k)\lambda_2})_{j,h}} \right)_{i,j\in\bK}.
    \end{gather}
  \end{proposition}

  \begin{proof}
    The first part of the statement follows immediately from the definition of the invariants.
    To prove the second part, it suffices to use the Induction-Reduction Theorem and the properties of the equivalence $A$ given in Proposition~\ref{prop:equivalence-prop}.
  \end{proof}

  Observe that similarly to the abelian case, it is easy to show that 
  \begin{equation}
    \BS_f\supset \{\PS_f(T)\mid T=T^\lambda \text{ for }\lambda\neq\hat o \text{ or } T=T^{\hat o\times\hat e}\}.
  \end{equation}

  \begin{theorem}
  \label{thm:moore-completeness}
  The bispectral invariants are complete on the set $\cG$ defined in \eqref{eq:def-cG-moore}.
  In particular, if $\bH$ is a connected Lie group, they are weakly complete on the set of compactly supported $L^2(\bG)$ functions.
  \end{theorem}


  \begin{proof}
    The last part of the statement follows from Theorem~\ref{thm:genericity-cG}.
    Let us consider $f,g\in \cG$ such that $\BS_f=\BS_g$.
    The idea of the proof is similar to the one of Theorem~\ref{thm:compact-completeness}.
    Namely, we start by defining a candidate quasi-representation $U$.
    Here, however, we will not prove that $U$ is a quasi-representation, since it is possible, and simpler, to directly prove that $U(T)=T(a)$ for some $a\in\bG$.

    Due to the added complexities arising in this case, we have delayed the technical parts of the proof to later lemmas, contained in Section~\ref{sec:auxiliary-lemmata-moore}.

    \begin{itemize}
      \item 
        \emph{Step 1 - Definition of the candidate quasi-representation:}
        From $\BS_f=\BS_g$ it follows that the sets where $\hat f$ and $\hat g$ fails to be invertible are the same.
        We will denote it with $I$.
        We then let
        \begin{equation}
        \label{eq:def-U-moore}
        U(T^\lambda) = \hat f(T^\lambda)^{-1}\hat g(T^\lambda) \in\bC^{N\times N} \qquad \forall \lambda\in I.
        \end{equation} 
        Clearly, $U(T^\lambda)$ is unitary for any $\lambda\in I$ (this can be proved as in step~2 of the proof of Theorem~\ref{thm:compact-completeness}).

        Since $\lambda\mapsto\hat f(T^\lambda)$ and $\lambda\mapsto\hat g(T^\lambda)$ are measurable, and $\hat G\setminus I$ is open and dense, by \eqref{eq:def-U-moore} also  $\lambda\mapsto U(T^\lambda)$ is measurable on $I$.

        By the equality of the second-type invariants and the definition of $U$, for any $\lambda_1,\lambda_2\in I$ it holds
        \[
        \hat f(\lambda_1)\otimes \hat f(\lambda_2)\circ  \hat f(T^{\lambda_1}\otimes T^{\lambda_2})^* =\hat f(\lambda_1)\otimes \hat f(\lambda_2)\circ U(T^{\lambda_1})\otimes U(T^{\lambda_2}) \circ \hat g(T^{\lambda_1}\otimes T^{\lambda_2})^*.
        \]
        By the invertibility of $\hat f(\lambda_1)\otimes \hat f(\lambda_2)$, this yields
        \begin{equation}
        \label{eq:U-on-tensor-moore}
        \hat f(T^{\lambda_1}\otimes T^{\lambda_2}) \circ U(T^{\lambda_1})\otimes U( T^{\lambda_2}) = \hat g(T^{\lambda_1}\otimes T^{\lambda_2}).
        \end{equation} 

      \item
        \emph{Step 2 - The function $\lambda\mapsto U(T^\lambda)$ is continuous on $I$:} 
        This is done in Lemma~\ref{lem:cont-U-moore}.

      \item
        \emph{Step 3 - The function $\lambda\mapsto U(T^\lambda)$ can be extended to a continuous function on $\widehat\bH\setminus\{\hat o\}$ for which \eqref{eq:U-on-tensor-moore} is still true:}
        This is done in Lemma~\ref{lem:extension-U-moore}.

      \item
        \emph{Step 4 - There exists $a\in\bG$ such that $U(T^\lambda)=T^\lambda(a)$ for any $\lambda\in\widehat\bH\setminus\{\hat o\}$:}
        This is done in Lemma~\ref{lem:U-quasi-rep-moore}.

      \item
        \emph{Step 5 - It holds that $\Lambda(a)f=g$:} 
        By definition of $U$ and Theorem~\ref{thm:repr-semidir}, the previous step proves that $\hat f (T^\lambda)\circ T^\lambda(a)=\hat g(T^\lambda)$ for any $\lambda\in\widehat\bH\setminus\{\hat o\}$.
        By Theorem~\ref{thm:FT-fund-prop}, this completes the proof of this step and hence of the statement.
      \end{itemize}
  \end{proof}


  \subsubsection{Auxiliary lemmas used in the proof of Theorem~\ref{thm:moore-completeness}} 
  \label{sec:auxiliary-lemmata-moore}
  
  \begin{lemma}
  \label{lem:cont-U-moore}
  For any $i,j$, the function $\lambda\mapsto U(T^\lambda)_{i,j}$ is continuous on $I$.
  \end{lemma}

  \begin{proof}
  By the Induction-Reduction theorem and the definition of $U$, formula \eqref{eq:U-on-tensor-moore} implies that for any $\lambda_1,\lambda_2\in I$ such that $\lambda_1+R_k \lambda_2\in I$ for any $k\in\bK$, it holds
  \begin{equation}
  \label{eq:ind-red-cont-U-moore}
  U(T^{\lambda_1+\phi(k)\lambda_2}) = \left(A \circ U(T^{\lambda_1})\otimes U(T^{\lambda_2})\circ A^*\right)_{k,k} \qquad \forall k\in\bK.
  \end{equation}

  Explicitly computing \eqref{eq:ind-red-cont-U-moore} with $k=e$ yields\footnote{For $k\neq 0$ the formula becomes \[U(T^{\lambda_1+\phi(k)\lambda_2})_{i,j} = U(T^{\lambda_1})_{i,j} U(T^{\lambda_2})^{T}_{i-k,j-k}.\]}     
  \begin{equation}
    \label{eq:ind-red-expl-moore}
    U(T^{\lambda_1+\lambda_2})_{i,j} = U(T^{\lambda_1})_{i,j} U(T^{\lambda_2})^{T}_{i,j}.
  \end{equation}

  Fix $\lambda_0\in I$ and choose an open set $V$ such that
  \begin{itemize}
    \item $\int_U U(\lambda_2)^T\,d\lambda_2 >0$;
    \item there exists a neighborhood $W$ of $\lambda_0$ such that $U+\lambda\subset I$ for any $\lambda\in W$.
  \end{itemize}
  This is possible since we can assume $f,g\not\equiv 0$, which yields $U\not\equiv 0$, and the set $I$ is open dense.
  Then, integrating \eqref{eq:ind-red-expl-moore} over $V$ w.r.t.\ $\lambda_2$ yields
  \[
  U(T^{\lambda})_{i,j}= \frac{\int_{V+\lambda} U(T^{\lambda_2})_{i,j}\,d\lambda_2}{\int_{V} U(T^{\lambda_2})_{i,j}^{T}\,d\lambda_2} \qquad\forall \lambda\in W.
  \]
  Since the function on the r.h.s.\ is clearly continuous on $W$ this proves the continuity at $\lambda_0$ of $U(T^{\lambda})$, completing the proof.
  \end{proof}

  \begin{lemma}
  \label{lem:extension-U-moore}
  The function $\lambda\mapsto U(T^\lambda)$ can be extended to a continuous function on $\widehat\bH\setminus\{\hat o\}$.
  Moreover, for any $\lambda_1,\lambda_2\neq \hat o$ it holds $\hat f(T^{\lambda_1}\otimes T^{\lambda_2}) \circ U(T^{\lambda_1})\otimes U( T^{\lambda_2}) = \hat g(T^{\lambda_1}\otimes T^{\lambda_2})$.
  \end{lemma}

  \begin{proof}
  Let $\lambda_0\notin I$.
  Since $I$ is an open and dense set, this implies that $\lambda_0$ is in its closure and that we can choose $\lambda_1,\lambda_2\in I$ such that $\lambda_0=\lambda_1+R_{k_0} \lambda_2$ for some $k_0\in\bK$ and $\lambda_1+\phi(k)\lambda_2\in I$ for any $k\neq k_0$.
  We then let
  \begin{equation}
    \label{eq:ind-red-Moore2}
    U(T^{\lambda_0}) \coloneqq \left(A \circ U(T^{\lambda_1})\otimes U(T^{\lambda_2})\circ A^*\right)_{k_0,k_0} \qquad \text{for } \lambda_0=\lambda_1+R_{k_0}\lambda_2,
  \end{equation}

  We now prove that the above definition does not depend on the choice of $\lambda_1$, $\lambda_2$ and $k_0$.
  By openness of $I$, there exists a neighborhood $V$ of $\lambda_2$ entirely contained in $I$.
  Then, up to taking a smaller $V$, it holds that $\lambda_1+R_{k_0}\lambda_2'\in I$ for any $\lambda_2'\in V\setminus\{\lambda_2\}$.
  By \eqref{eq:ind-red-cont-U-moore}, this implies that for any $\mu_1+R_{\ell}\mu_2=\lambda_0$ it holds $(A \circ U(T^{\lambda_1})\otimes U(T^{\lambda_2'})\circ A^*)_{k_0,k_0}=(A \circ U(T^{\mu_1})\otimes U(T^{\mu_2'})\circ A^*)_{\ell,\ell}$ for $\lambda_2'$ and $\mu_2'$ sufficiently near, but different, to $\lambda_2$ and $\mu_2$, respectively.
  By continuity of $U$ on $I$, proved in Lemma~\ref{lem:cont-U-moore}, this implies that this equation has to hold also for $\lambda_2'=\lambda_2$  and $\mu_2'=\mu_2$.
  Hence, \eqref{eq:ind-red-Moore2} does not depend on the choice of $\lambda_1,\lambda_2$ and $k_0$.

  Finally, the fact that $\hat f(T^{\lambda_1}\otimes T^{\lambda_2}) \circ U(T^{\lambda_1})\otimes U( T^{\lambda_2}) = \hat g(T^{\lambda_1}\otimes T^{\lambda_2})$ for any $\lambda_1,\lambda_2$ follows from \eqref{eq:ind-red-Moore2} and the Induction-Reduction theorem.
  \end{proof}

  \begin{lemma}
    \label{lem:U-quasi-rep-moore}
    There exists $a\in\bG$ such that $U(T^\lambda)=T^\lambda(a)$ for any $\lambda\in\widehat H\setminus\{\hat o\}$.
  \end{lemma}

  \begin{proof}
  By definition of $U$ it holds that
  \[
  \bigoplus_{k\in\bK} U(T^{\lambda_1+\phi(k)\lambda_2}) \circ A = A\circ U(T^{\lambda_1})\otimes U(T^{\lambda_2}) \qquad\forall \lambda_1,\lambda_2\neq\hat o.
  \]
  Then, for any $i,j,\ell,k$,
  \begin{equation}
  \label{eq:final-comp-moore}
  U(T^{\lambda_1})_{\ell,i}U(T^{\lambda_2})_{\ell-k,j} = 
      \begin{cases}
      U(T^{\lambda_1+\phi(k)\lambda_2})_{\ell,i} &\qquad \text{if }j=i-k,\\
      0 &\qquad \text{otherwise.}
      \end{cases}
  \end{equation}

  Since $U(T^{\lambda_1})$ is invertible, there exists $i_{0}$ such that $U(T^{\lambda_1})_{e,i_0}\neq 0$.
  Using \eqref{eq:final-comp-moore} one obtains that $U(T^{\lambda_2})_{-k,j} =0$ for any $j\neq i_{0}-k$.
  Namely, we have proved that $U(T^{\lambda_1})_{-k,\cdot} = \varphi_{-k}(\lambda_1) e_{i_{0}-k}$ for any $h$ for some $\varphi_{-k}:\bH\setminus\{\hat o\}\to \bC$.

  We can rephrase the above result as $U(T^\lambda) = \diag_k \varphi_k(\ell)\,S^{i_0}$.
  Thus, by the explicit expression of the representation $T^\lambda$, in order to complete the proof it suffices to prove that $\varphi_k(\lambda)=\lambda(R_{k}h_0)$ for some $h_0\in\bH$.

  By continuity and unitarity of $U$, the $\varphi_h$'s are continuous and satisfy $|\varphi_h(\lambda)|=1$.
  Using again \eqref{eq:final-comp-moore} with $j=i_0-k$, we obtain 
  \begin{equation}
    \label{eq:character-moore}
    \varphi_{\ell}(\lambda_1+\phi(k)\lambda_2) = \varphi_\ell(\lambda_1)\varphi_{\ell-k}(\lambda_2),\quad\text{ for any }\lambda_1,\lambda_2\neq\hat o\text{  and }\ell,k\in\bK.
  \end{equation}
  In particular, choosing $k=e$ and $\lambda_2=-\lambda_1$ in the above shows that $\varphi_\ell$ can be extended at $\hat o$.
  If $\hat o$ is an accumulation point this extension is continuous, as one can see letting $k=e$ and $\lambda_2\rightarrow \hat o$ in \eqref{eq:character-moore}.
  Then, \eqref{eq:character-moore} with $k=e$ implies $\varphi_\ell$ is a character of $\widehat\bH$.
  By Pontryagin duality, there exists $h_\ell\in\bH$ such that $\varphi(\lambda) = \lambda(h_\ell)$.
  Finally, by \eqref{eq:character-moore} with $k\in\bK$ one obtains that $R_{-k}h_\ell=h_{\ell-k}$, which proves that there exists $h_0\in\bH$ such that $\varphi_\ell(\lambda)=R_{\ell}h_0$.
  This completes the proof of the statement.
  \end{proof}


\section{Bispectral invariants for lifts} 
\label{sub:bispectral_invariants_for_lifts}

  Let us consider $\bG=\bK\ltimes\bH$ as in Section~\ref{sec:repr-semidirect}.
  In this section we will discuss bispectral invariants on $\range L$, where $L$ is one of the lift operators described in Section~\ref{lifts}.

  Henceforth, to lighten the notation, when an injective lift is fixed and only functions in $\range L$ are considered, we denote the invariants in $L^2(\bG)$ of $Lf$, $f\in L^2(\bH)$, by $\PS_f$ and $\BS_f$.

    The following two sections are devoted to show how, when the lift is either regular left-invariant of regular cyclic, the comparison of the invariants for lifted functions reduces to the computation of some basic quantities, depending only on the Fourier transform of the starting function on $\bH$. In particular, this allows to show that, in the case of a regular left-invariant lift, these quantities can be further reduced. Indeed, in this case, under some assumptions on the wavelet $\Psi$,  it is enough to simply compare the trace of the invariants.

\subsection{Regular left-invariant lifts} 
\label{ssub:regular_left_invariant_lifts}


  Let $L:L^2(\bH)\to L^2(\bG)$ be a regular left-invariant lift and let $\Psi\in L^2(\bH)$ be the associated wavelet given by Theorem~\ref{thm:left-inv-form}.

  \begin{proposition}
    \label{prop:invariants-expression}
    Let $f\in L^2(\bH)$. 
    Then, for any $\lambda,\lambda_1,\lambda_2\in\widehat\bH\setminus\{\hat o\}$, we have
    \begin{gather}
      \PS_f(T^\lambda) = \|\widehat{f}_\lambda \|_{L^2(\bK)}^2\, \widehat{\Psi^*}_\lambda \otimes \overline{\widehat{\Psi^*}_\lambda},\\
      \label{eq:bispectral-expression}
        \big(A\circ  \BS_f(T^{\lambda_1}, T^{\lambda_2}) \circ A^*\big)_{k,\ell} =
        \left\langle \widehat f_{\lambda_1} \widehat f_{\phi(\ell)\lambda_2}, \widehat f_{\lambda_1+\phi(\ell)\lambda_2} \right\rangle 
        (\widehat{\Psi^*}_{\lambda_1}\widehat{\Psi^*}_{\phi(k)\lambda_2})\otimes \overline{\widehat{\Psi^*}_{\lambda_1+\phi(k)\lambda_2}}.
    \end{gather}
    Here, $A$ is the equivalence from $L^2(\bK\times\bK)$ to $\bigoplus_{k\in\bK}L^2(\bK)$ defined in Theorem~\ref{thm:ind-reduction}. 
  \end{proposition}

  \begin{proof}
    The first part of the statement follows directly from Proposition~\ref{prop:ft-lift} and the properties of the tensor product.
    On the other hand, a simple manipulation by the Induction-Reduction Theorem yields
    \begin{equation}
      \label{eq:2inv-complete}
      \left(A\circ \BS_f(T^{\lambda_1}, T^{\lambda_2})\circ A\right)_{k,\ell}  = \left(A\circ \widehat{Lf}(T^{\lambda_1})\otimes \widehat{Lf}(T^{\lambda_2})\circ A^*\right)_{k,\ell} \circ \widehat{Lf} (T^{\lambda_1+\phi(\ell)\lambda_2})^*.
    \end{equation}
    Then, to prove the second part of the statement proof it suffices to apply \eqref{eq:equiv_tensor} of Proposition~\ref{prop:equivalence-prop} and Proposition~\ref{prop:ft-lift} to the above.
  \end{proof}

  We then have the following.

  \begin{corollary}
    \label{cor:reduced-inv}
    Let $\Psi\in L^2(\bH)$ be a weakly admissible wavelet and let $f,g\in L^2(\bH)$.
    Then,
    \begin{equation}
     \PS_f=\PS_g \iff \|\widehat f_\lambda\|_{L^2(\bK)}=\|\widehat g_\lambda\|_{L^2(\bK)} \text{ for a.e.\ }\lambda\in \widehat H\setminus\{\hat o\}.
    \end{equation}

    Moreover, if $\Psi$ is such that $\widehat{\Psi^*}_{\lambda_1}\widehat{\Psi^*}_{\lambda_2}\neq 0$ for a.e.\ $\lambda_1,\lambda_2$, then,
    \begin{equation}
      \label{eq:reduced-inv}
      \BS_f = \BS_g \iff 
      \langle \widehat f_{\lambda_1} \widehat f_{\lambda_2}, \widehat f_{\lambda_1+\lambda_2} \rangle = 
      \langle \widehat g_{\lambda_1} \widehat g_{\lambda_2}, \widehat g_{\lambda_1+\lambda_2} \rangle 
      \text{ for a.e.\ }\lambda_1,\lambda_2\in \widehat H\setminus\{\hat o\}.
    \end{equation}
  \end{corollary}

  \begin{proof}
    By Proposition~\ref{prop:invariants-expression}, the first statement is equivalent to $\widehat{\Psi^*}_\lambda\otimes \overline{\widehat{\Psi^*}_\lambda} \neq 0$ for a.e.\ $\lambda\in \widehat H\setminus\{\hat o\}$, which is equivalent to $\widehat{\Psi^*}_\lambda\neq 0$. By Theorem~\ref{thm:admissible-vect} this is true for any weakly admissible vector.

    To prove the second statement, it suffices to make the change of variables $\phi(\ell)\lambda_2\mapsto \lambda_2$ in \eqref{eq:bispectral-expression}.
    Indeed, the conclusion is then equivalent to $\widehat{\Psi^*}_{\lambda_1+\lambda_2}\neq 0$, which is satisfied by weak admissibility, and $\widehat{\Psi^*}_{\lambda_1}\widehat{\Psi^*}_{\lambda_2}\neq 0$, which is satisfied by assumption.
  \end{proof}
  
  \subsubsection{Trace invariants} 
\label{ssub:trace_invariants}

  In this section we show how, exploiting Proposition~\ref{prop:invariants-expression}, one can actually decrease the set of invariants.
  To this aim, let us recall that the trace of a trace class operator $C$ acting on the Hilbert space $\cH$ is defined as
  \begin{equation}
    \tr C = \sum_{i} \langle C e_i,e_i\rangle,
  \end{equation}
  where $\{e_i\}_i$ is a basis of $\cH$.
  Being the product of two Hilbert-Schimdt operators, the bispectral invariants are of trace class. 

  \begin{definition}
    The \emph{trace bispectral invariants} associated with the regular left-invariant lift $L$ of $f\in L^2(\bH)$ are the set 
    \begin{equation}
    \tr \BS_f = \{ \tr \BS_f(T^{\lambda_1},T^{\lambda_2}) \mid \lambda_1,\lambda_2\in \widehat H\setminus\{\hat o\} \}.  
    \end{equation}
  \end{definition}

  \begin{corollary}
    \label{cor:trace-invariants}
    Let $f\in L^2(\bH)$.
    Then, for any $\lambda_1,\lambda_2\in\widehat H\setminus\{\hat o\}$ it holds that
    \begin{equation}
      \tr \BS_f(T^\lambda_1,T^\lambda_2) =
       \sum_{k\in\bK} \langle \widehat f_{\lambda_1} \widehat f_{\phi(k) \lambda_2}, \widehat f_{\lambda_1\phi(k) \lambda_2} \rangle\: \tr \left((\widehat{\Psi^*}_{\lambda_1}\widehat{\Psi^*}_{\phi(k)\lambda_2})\otimes \overline{\widehat{\Psi^*}_{\lambda_1+\phi(k)\lambda_2}}\right)
    \end{equation}
  \end{corollary}

  \begin{proof}
    The statement is an immediate consequence of Proposition~\ref{prop:invariants-expression} and of the similarity-invariance of the trace.
  \end{proof}

  For any $\lambda_1,\lambda_2\in\widehat H\setminus\{\hat o\}$ and any $h\in \bK$ we let $B_\Psi^h\in L^2(\bK)$ be
  \begin{equation}
    B_\Psi^h(k) = \widehat{\Psi^*}_{\lambda_1}(h)\, \widehat{\Psi^*}_{\phi(k)\lambda_2}(h)\,  \overline{\widehat{\Psi^*}_{\lambda_1+\phi(k)\lambda_2}}(h).
  \end{equation}
  In particular, 
  \begin{equation}
    \label{eq:tr-B}
  \tr \left((\widehat{\Psi^*}_{\lambda_1}\widehat{\Psi^*}_{\phi(k)\lambda_2})\otimes \overline{\widehat{\Psi^*}_{\lambda_1+\phi(k)\lambda_2}}\right) 
  = \sum_{h\in\bK} B^h_\Psi(k).
  \end{equation}
  This justifies the following. 

  \begin{definition}
    A wavelet $\Psi\in L^2(\bH)$ is \emph{trace admissible} if for a.e.\ $\lambda_1,\lambda_2\in\widehat\bH\setminus\{\hat o\}$ it holds that $\widehat{\Psi^*}_{\phi(k)\lambda_1}\widehat{\Psi^*}_{\lambda_2}\neq 0$ for all $k\in\bK$ and that the family $\{B_{\Psi}^h\}_{h\in\bK}$ is a basis of $L^2(\bK)$.
  \end{definition}

  Observe that a trace admissible wavelet is always weakly admissible and satisfies the assumptions of Corollary~\ref{cor:reduced-inv}.

  \begin{proposition}
    Let $L$ be a regular left-invariant lift with associated wavelet $\Psi\in L^2(\bH)$.
    Then, if $\Psi$ is trace admissible it holds $\tr \BS_f\supset \BS_f$.
  \end{proposition}

  \begin{proof}
    Since $\Psi$ satisfies the assumptions of Corollary~\ref{cor:reduced-inv}, we only need to show that $\tr \BS_f=\tr \BS_g$ if and only if \eqref{eq:reduced-inv} is satisfied.
    Putting together Corollary~\ref{cor:trace-invariants} and \eqref{eq:tr-B}, and by exchanging the summation order, we obtain
    \begin{equation}
        \tr \BS_f(\lambda_1,\lambda_2) 
        =\sum_{h\in\bH} \left\langle k\mapsto \langle \widehat f_{\lambda_1} \widehat f_{\phi(k) \lambda_2}, \widehat f_{\lambda_1\phi(k) \lambda_2} \rangle,\overline{B_\Psi^h} \right\rangle,
    \end{equation}
    which completes the proof. Indeed, since $\{B_\Psi^h\}_{h\in\bK}$ is a basis of $L^2(\bK)$, this shows that $\tr \BS_f(\lambda_1,\lambda_2)=\tr \BS_g(\lambda_1,\lambda_2)$ if and only if for any $k\in \bK$
    \begin{equation}
      \langle \widehat f_{\lambda_1} \widehat f_{\phi(k) \lambda_2}, \widehat f_{\lambda_1\phi(k) \lambda_2} \rangle = \langle \widehat g_{\lambda_1} \widehat g_{\phi(k) \lambda_2}, \widehat g_{\lambda_1\phi(k) \lambda_2} \rangle.
    \end{equation}
  \end{proof}

\subsection{Regular cyclic lifts} 
\label{ssub:regular_cyclic_lifts}

    Let $L=P\circ \Phi:\cA\to L^2(\bG)$ be a regular cyclic lift where $\cA\subset L^2(\bH)$ is closed w.r.t.\ the quasi-regular representation $\pi$, as introduced in Section~\ref{sec:cyclic-lift}.
    Let $\Psi\in L^2(\bH)$ be the associated wavelet, whose existence is assured by Corollary~\ref{thm:cyclic-form}, and assume that the centering $\Phi$ be w.r.t.\ to the whole $\bH$.

    \begin{proposition}
    \label{prop:cyclic-invariants-expression}
    Assume that $\bK\simeq\bK^2$ and let $f\in \cA$. 
    Then, for any $\lambda,\lambda_1,\lambda_2\in\widehat\bH\setminus\{\hat o\}$ and $i,j,k,\ell\in\bK$, it holds that 
    \begin{gather}
      \PS_f(T^\lambda)_{i,j} = \left(\widehat{\Psi^*}_\lambda\otimes \overline{\widehat{\Psi^*}_\lambda}\right)_{i,j}\, \langle S^{i^{-1}j}\widehat {\Phi(f)}_\lambda, \widehat {\Phi(f)}_\lambda\rangle\\
      \begin{split}
        (&A\circ  \BS_f(T^{\lambda_1}, T^{\lambda_2}) \circ A^{*})_{k,\ell,i,j} =\\
        & \left((\widehat{\Psi^*}_{\lambda_1}\,S^\ell \widehat{\Psi^*}_{\lambda_2})\otimes \overline{\widehat{\Psi^*}_{\lambda_1+\phi(k)\lambda_2}}\right)_{i,j}\,
          \langle S^{i^{-1}j}\widehat {\Phi f}_{\lambda_1+\phi(k)\lambda_2}, \widehat {\Phi f}_{\lambda_1}\,S^{\ell^{-1}k}\widehat {\Phi f}_{\phi(k)\lambda_2}\rangle.
      \end{split}
    \end{gather}
    Here, $A$ is the equivalence from $L^2(\bK\times\bK)$ to $\bigoplus_{k\in\bK}L^2(\bK)$ defined in Theorem~\ref{thm:ind-reduction}.
  \end{proposition}


  \begin{proof}
    Due to the special form of cyclic lifts, it suffices to  replace $f$ with $\Phi(f)$ in the expressions of the invariants $\tilde \PS_f = \PS_{Pf} $ and $\tilde \BS_{f}=\BS_{Pf}$, corresponding to the almost left-invariant lift associated with $L$.

    By Proposition~\ref{prop:ft-almost-left-lift} and \ref{prop:moore-inv-expression}, direct computations yield 
    \begin{equation}
        \tilde \PS_f(T^\lambda)_{i,j} 
        = \widehat{\Psi^*}_\lambda(i)\overline{\widehat{\Psi^*}_\lambda(j)} \sum_{h\in\bK} \widehat f_\lambda(i^{-1}h^2) \overline{\widehat f_\lambda(j^{-1}h^2)}.
    \end{equation}
    Then, the statement follows via the change of variables $\ell = i^{-1}h^2$, which is justified by the assumption $\bK\simeq\bK^2$.
    The second statement is proved in a similar way.
  \end{proof}

  As an immediate consequence of the previous result we obtain: 

  \begin{corollary}
    \label{cor:invariants-expression-cyclic}
    Assume that $\bK\simeq\bK^2$ and that $\hat\Psi\neq 0$ a.e.\ on $\widehat\bH$.
    Then, for any $f,g\in \cA$, we have
    \begin{equation}
      \label{eq:reduced-inv-first-cyclic}
      \PS_f=\PS_g \iff   \langle S^{k}\widehat {\Phi(f)}_\lambda, \widehat {\Phi(f)}_\lambda\rangle = \langle S^{k}\widehat {\Phi(g)}_\lambda, \widehat {\Phi(g)}_\lambda\rangle \text{ for a.e.\ }\lambda\in\widehat\bH\setminus\{\hat o\}\text{ and }k\in\bK.
    \end{equation}
    \begin{multline}
      \label{eq:reduced-inv-cyclic}
      \BS_f = \BS_g \iff 
      \langle S^{k}\widehat {\Phi f}_{\lambda_1+\lambda_2}, \widehat {\Phi f}_{\lambda_1}\,S^{h}\widehat {\Phi f}_{\lambda_2}\rangle = \\
      \langle S^{k}\widehat {\Phi g}_{\lambda_1+\lambda_2}, \widehat {\Phi g}_{\lambda_1}\,S^{h}\widehat {\Phi g}_{\lambda_2}\rangle,
      \text{ for a.e.\ }\lambda_1,\lambda_2\in\widehat\bH\setminus\{\hat o\}\text{ and }k,h\in\bK.
    \end{multline}
  \end{corollary}

  Finally, due to the good properties of the Fourier transforms of cyclically lifted functions, we have the following.

  \begin{proposition}
    Assume that $\bK\simeq\bK^2$.
    Moreover, assume $L$ to be a regular cyclic lift such that the associated wavelet satisfies $\hat\Psi\neq0$ a.e.\ on $\widehat\bH$.
    Then, the bispectral invariants evaluated on lifted functions are weakly complete on $L^2(\bH)$.
  \end{proposition}

  \begin{proof}
    The result essentially follows from Theorem~\ref{thm:moore-completeness}.
    Indeed, by \eqref{eq:cyclic-lift-inv}, the assumptions on $\bK$ and the injectivity of $L$, it suffices to show that there exists a residual subset $\cG$ of $L^2(\bH)$ such that $\widehat{Lf}(T^\lambda)$ is invertible on an open and dense subset of $\widehat H\setminus\{\hat 0\}$.
    Indeed, this will trivially imply that $\cG$ is residual in $\cA$.
    From Proposition~\ref{prop:almost-inv-lift-ft-inv} it follows that the set $\cC$ of weakly-cyclic $L^2(\bH)$ functions has this property and is residual in $\cG$ by Theorem~\ref{thm:genericity-weakly-cyclic}.
  \end{proof}




\section{Rotational bispectral invariants for left-invariant lifts modulo the action of $\bH$}
\label{sec:rotational-bispectral-invariants}

Let $\bG=\bK\ltimes\bH$ be as in Section~\ref{sec:repr-semidirect}, with $\bK$ finite with $N$ elements. 
In the following we consider a stronger family of invariants than the bispectral ones, which will turn out to be weakly complete on functions lifted from $L^2(\bH)$ to $L^2(\bG)$ via a regular left-invariant lift.
More precisely, after defining these invariants, which we call \emph{rotational bispectral invariants}, we show how to compute them on lifted functions (giving thus a counterpart to Proposition~\ref{prop:invariants-expression}). After this, we prove their weak completeness in Theorem~\ref{thm:rot-bisp-completeness}, the main result of this whole monograph. Finally, we end the section by showing how the weak completeness result can be extended to encompass real valued functions, and how it can be strengthened when $\bH=\bR^2$ and the functions under considerations are compactly supported.


\begin{definition}
  The \emph{rotational power spectrum invariants} of $f\in L^2(\bG)$ are the set $\RPS_f=\{\RPS_f(\lambda,k)\mid \lambda\in \widehat\bH\setminus\{\hat o\} \text{ and } k\in\bK \}$ such that 
  \begin{equation}
    \RPS_f(\lambda,k) \coloneqq \hat f(T^{\phi(k)\lambda})\circ \hat f (T^{\lambda})^*.
  \end{equation}

  The \emph{rotational bispectral invariants} of $f\in L^2(\bG)$ are the set $\RBS_f=\{\RBS_f(\lambda_1,\lambda_2,k)\mid \lambda_1,\lambda_2\in \widehat\bH\setminus\{\hat o\} \text{ and } k\in\bK \}$ such that 
  \begin{equation}
    \RBS_f(\lambda_1,\lambda_2,k) \coloneqq \hat f(T^{\phi(k)\lambda_1})\otimes \hat f(T^{\lambda_2})\circ \hat f (T^{\lambda_1}\otimes T^{\lambda_2})^*.
  \end{equation}
\end{definition}

For any $k\neq e$, the above defined quantities are invariant only w.r.t.\ the action of $\bK$ on $\bG$.
This implies that they can only discriminate up to the action of $\bK$.

Since $\RBS_f\supset \BS_f$, as a consequence of Theorem~\ref{thm:moore-completeness} we immediately obtain the following.

\begin{corollary}
  Rotational bispectral invariants are complete w.r.t. the action of $\bK$ on the set $\cG$ defined in \eqref{eq:general-def-G}.
  Namely, for any $f,g\in\cG$ it holds that 
  \begin{equation}
    \RBS_f = \RBS_g \iff f = \phi(k) g\quad \text{for some }k\in\bK.
  \end{equation}
\end{corollary}


Let $L=P\circ\Phi:L^2(\bH)\to L^2(\bG)$ be the composition of a regular left-invariant lift $P$ and a centering $\Phi:\cA\to \cA$ w.r.t.\ $U\subset \bH$, see Definition~\ref{def:centering}.
Denote by $\Psi\in L^2(\bH)$ the wavelet associated with $P$, given by Theorem~\ref{thm:left-inv-form}.

In the following, for $f\in L^2(\bH)$ we let $\RPS_f=\RPS_{Lf}$ and $\RBS_f=\RBS_{Lf}$.
The following can be proved as Proposition~\ref{prop:invariants-expression}.

\begin{proposition}
  Let $f\in L^2(\bH)$. 
  Then, for any $\lambda,\lambda_1,\lambda_2\in\widehat H\setminus\{\hat o\}$ and any $k,h,\ell\in\bK$ it holds
  \begin{gather}
    \RPS_f(\lambda,k) = \widehat{\Psi^*}_\lambda\otimes \overline{\widehat{\Psi^*}_{\phi(k)\lambda}}\, \langle\widehat{\Phi(f)}_\lambda, S(k^{-1})\widehat{\Phi(f)}_\lambda\rangle,\\
    \begin{split}
      \big(A\circ  \RBS_f(\lambda_1,\lambda_2,k) & \circ A^*\big)_{h,\ell} =\\
      &\left\langle \widehat {\Phi f}_{\lambda_1+\phi(\ell)\lambda_2}, \widehat {\Phi f}_{\lambda_1} \widehat {\Phi f}_{\phi(k\ell)\lambda_2} \right\rangle 
        (\widehat{\Psi^*}_{\lambda_1}\widehat{\Psi^*}_{\phi(h k)\lambda_2})\otimes \overline{\widehat{\Psi^*}_{\lambda_1+\phi(h)\lambda_2}}.
    \end{split}
  \end{gather}
  Here, $A$ is the equivalence from $L^2(\bK\times\bK)$ to $\bigoplus_{h\in\bK}L^2(\bK)$ defined in Theorem~\ref{thm:ind-reduction}.
\end{proposition}

\begin{corollary}
    \label{cor:reduced-rot-inv}
    Let $\Psi\in L^2(\bH)$ be a weakly admissible wavelet.
    Then, for any $f,g\in L^2(\bH)$, we have
    \begin{multline}
     \label{eq:reduced-inv-first-rot}
     \RPS_f = \RPS_g  \iff  \\
     \langle\widehat{\Phi(f)}_\lambda, S^{h}\widehat{\Phi(f)}_\lambda\rangle = \langle\widehat{\Phi(g)}_\lambda, S^{h}\widehat{\Phi(g)}_\lambda\rangle \text{ for a.e.\ }\lambda\in \widehat H\setminus\{\hat o\}\text{ and }h\in\bK.
    \end{multline}

    If moreover $\Psi$ is such that $\widehat{\Psi^*}_{\lambda_1}\widehat{\Psi^*}(\phi(k)\lambda_2)\neq0$ for any $k\in\bK$ and a.e.\ $\lambda_1,\lambda_2$, 
    \begin{multline}
      \label{eq:reduced-inv-rot}
      \RBS_f=\RBS_g \iff 
      \left\langle \widehat {\Phi f}_{\lambda_1+\lambda_2}, \widehat {\Phi f}_{\lambda_1} S^h\widehat {\Phi f}_{\lambda_2} \right\rangle = \\
      \left\langle \widehat {\Phi g}_{\lambda_1+\lambda_2}, \widehat {\Phi g}_{\lambda_1} S^h\widehat {\Phi g}_{\lambda_2} \right\rangle, 
      \text{ for a.e.\ }\lambda_1,\lambda_2\in \widehat H\setminus\{\hat o\}\text{ and }h\in\bK
    \end{multline}
\end{corollary}

\begin{remark}
  Observe that in the above result no assumptions on the cardinality of $\bK$ are required. Moreover, comparing it to Corollary~\ref{cor:invariants-expression-cyclic}  seems to suggest that rotational bispectral invariants carry less information than the bispectral invariants of cyclic lifts.
  Indeed, while \eqref{eq:reduced-inv-first-rot} is identical to \eqref{eq:reduced-inv-first-cyclic}, in \eqref{eq:reduced-inv-rot} we consider $|\bK|$ invariants for couple $(\lambda_1,\lambda_2)$ against the $|\bK|^2$ of \eqref{eq:reduced-inv-cyclic}.
\end{remark}

The rest of this section is devoted to prove the weak completeness of the rotational bispectral invariants in this context.

\begin{theorem}
  \label{thm:rot-bisp-completeness}
  Let the lift $L=P\circ \Phi:\cA \to L^2(\bG)$ be the composition of a regular left invariant lift $P:L^2(\bH)\to L^2(\bG)$ and a centering $\Phi:\cA\to \cA$ w.r.t.\ $U\subset \bH$.
  Moreover, assume the wavelet $\Psi$ associated with $P$ to be weakly cyclic and such that $\hat\Psi\neq 0$ a.e..

  Then, the rotational bispectral invariants evaluated on lifted functions are complete on the set $\{ f\in \cC\cap \cA\mid \hat f\neq 0 \text{ a.e.\ on }\widehat\bH \}$ w.r.t.\ the action of elements in $U\times \bK$.
  Here, $\cC$ is the set of weakly-cyclic functions on $L^2(\bH)$ introduced in Definition~\ref{def:cc-weakly-cyclic}.
  Namely, for any couple $f,g\in\cA$ of weakly cyclic functions such that $\hat f$ and $\hat g\neq 0$ a.e.\ on $\widehat \bH$, it holds
  \begin{equation}
    \RBS_f=\RBS_g \iff f=\pi(a)g \quad\text{for some } a\in U \times\bK\subset \bG.
  \end{equation}
\end{theorem}

\begin{proof}
  Since $\Phi$ is a centering w.r.t.\ $U$, by the properties of the abelian Fourier transform w.r.t.\ translations follows that if $f$ is weakly-cyclic so is $\Phi (f)$.
  Then the statement is equivalent to the fact that for any couple $f,g$ of weakly-cyclic functions, $\RBS_{Pf}=\RBS_{Pg}$ if and only if $f = \phi(k) g$ for some $k\in\bK$.
  Given two such functions, we let
  \begin{equation}
    I = \left\{ \lambda \mid \det\Circ \widehat f_\lambda \neq 0 \text{ and } \det\Circ \widehat g_\lambda \neq 0 \right\} \subset\widehat \bH.
  \end{equation}
  By the weak-cyclicity of $f$ and $g$ this set is open and dense.

  The proof follows similar steps as the proof of Theorem~\ref{thm:moore-completeness}.
  One has however to pay additional care, due to the non-invertibility of the lifted Fourier transforms.
  The most delicate point is the commutation with the tensor product, which was proved in one line in step 1 of Theorem~\ref{thm:moore-completeness}.
  Here we delay the proof of this fact to Lemma~\ref{lem:comm-tensor-prod-rot}.

  \begin{itemize}
    \item \emph{Step 1.1 - Definition of the candidate quasi-representation $U$ on $T^\lambda$ for $\lambda\in I$:}
      For any $\lambda\in I$ we let 
      \begin{equation}
        \label{eq:U-rot-inv}
        U(T^\lambda)^* = \Circ \widehat g_\lambda \left(\Circ \widehat f_\lambda\right)^{-1}.
      \end{equation}
      Equivalently, $U(T^\lambda)$ is such that $U(T^\lambda)^* S^k\widehat f_\lambda = S^k \widehat g_\lambda$ for any $k\in\bK$.
      It is obvious that $U(T^\lambda)$ is circulant, see Appendix~\ref{app:circulant}, i.e., that it commutes with the shifts $S(k)$ for $k\in\bK$.
      Moreover, $U(T^\lambda)$ is unitary by the equality of the rotational power invariant and \eqref{eq:reduced-inv-first-rot} of Corollary~\ref{cor:reduced-rot-inv}.

      By the expression of the Fourier transform of $Pf$ given in Proposition~\ref{prop:ft-lift}, the definition of $U$ is also equivalent to
      \begin{equation}
        \widehat{Lf}(T^{\phi(k)\lambda})\circ U(T^\lambda) = \widehat{Lg}(T^{\phi(k)\lambda}) \qquad \forall \lambda\in I,\, \forall k\in\bK.
      \end{equation}
      Since $T^\lambda = S(\ell^{-1})\circ T^{\phi(\ell)\lambda}\circ S(\ell)$ by Theorem~\ref{thm:repr-semidir}, this implies that $\lambda\mapsto U(T^\lambda)$ is constant on the orbits $\{\phi(k)\lambda\}_{k\in\bK}$ of $\lambda$.

    \item \emph{Step 1.2 - Definition of $U$ on $T^{\lambda_1}\otimes T^{\lambda_2}$:}
      To extend the definition of $U$ to the tensor product of representations we use the Induction-Reduction Theorem.
      Let us call $I^\otimes$ the set of couples $(\lambda_1,\lambda_2)\in \widehat\bH\times \widehat\bH$ such that $\lambda_1+R_k \lambda_2\in I$ for any $k\in\bK$.
      Then, we let
      \begin{equation}
        \label{eq:U-rot-tensor}
        U\left( T^{\lambda_1}\otimes T^{\lambda_2} \right) = A^* \circ \left( \bigoplus_{k\in\bK} U(T^{\lambda_1+R_k \lambda_2})\right) \circ A \qquad \forall (\lambda_1,\lambda_2)\in I^\otimes.
      \end{equation}
      By the corresponding property of $\lambda\mapsto U(T^\lambda)$, this definition implies that $(\lambda_1,\lambda_2)\mapsto U(T^{\lambda_1}\otimes T^{\lambda_2})$ is constant on the orbits $\{(\phi(k)\lambda_1,\phi(k)\lambda_2)\}_{k\in\bK}$ of $(\lambda_1,\lambda_2)$.

      By the Induction-Reduction Theorem and the properties of the Fourier transform, \eqref{eq:U-rot-tensor} is equivalent to set, for all $(\lambda_1,\lambda_2)\in I^\otimes$ and $k\in\bK$,
      \begin{equation}
        \label{eq:U-FT-tensor}
        \widehat{Lf}(T^{\phi(k)\lambda_1}\otimes T^{\phi(k) \lambda_2})\circ U(T^{\lambda_1}\otimes T^{\lambda_2}) = \widehat{Lg}(T^{\phi(k)\lambda_1}\otimes T^{\phi(k) \lambda_2}).
      \end{equation}

    \item \emph{Step 1.3 - It holds that $U(T^{\lambda_1}\otimes T^{\lambda_2})=U(T^{\lambda_1}) \otimes U(T^\lambda_2)$:} 
      This is proved in Lemma~\ref{lem:comm-tensor-prod-rot}.

    \item \emph{Step 2 - The function $\lambda\mapsto U(T^\lambda)$ is continuous on I:}
      Since $\lambda\mapsto \widehat f_\lambda$ and $\lambda\mapsto \widehat g_\lambda$ are measurable on $I$, so it is $\lambda\mapsto U(T^\lambda)$.
      The same arguments used in Lemma~\ref{lem:cont-U-moore} can be then used to prove the continuity.
    \item \emph{Step 3 - The function $\lambda\mapsto U(T^\lambda)$ can be extended to a continuous function on $\widehat\bH\setminus\{\hat o\}$. 
    Moreover, the function $(\lambda_1,\lambda_2)\mapsto U(T^{\lambda_1}\otimes T^{\lambda_2})$ defined via \eqref{eq:U-rot-tensor} on $\widehat\bH\times \widehat\bH$ satisfies \eqref{eq:U-FT-tensor}:}
      This is proved exactly as in Lemma~\ref{lem:extension-U-moore}.
    \item \emph{Step 4 - There exists $k\in\bK$ such that $U(T^\lambda)=T^\lambda(o,k)$ for any $\lambda\in\widehat\bH\setminus\{\hat o\}$:}
      This is proved with the same arguments as in Lemma~\ref{lem:U-quasi-rep-moore}. 
      Indeed, the fact that now $\lambda\mapsto U(T^\lambda)$ is constant on the orbits $\left\{ \phi(k) \lambda\right\}_{k\in\bK}$ implies that the $\varphi_k$'s obtained there have to be independent of $k$. 
      Since $\varphi_k(\lambda)=\phi(k) x_0$ for some $x_0\in\bH$, this implies that $x_0 =0$ and hence $\varphi_k\equiv 0$.
      Obviously this proves that $U(T^\lambda)=S(k) = T^\lambda(o,k)$, for some $k\in\bK$.
    \item \emph{Step 5 - It holds that $\phi(k) f = g$:}
      This follows exactly as in Theorem~\ref{thm:moore-completeness}. 
  \end{itemize}
\end{proof}

The above result can be easily adapted to the subspaces of Besicovitch almost periodic functions introduced in Section~\ref{sub:finite_dimensional_subspaces_of_almost_periodic_functions}.

\begin{theorem}
  \label{thm:AP-rot-bisp-inv-complete}
  Let $E\subset \widehat\bH$ be a bispectrally admissible set, $K\subset \bH$ be compact and consider a lift $L=P\circ \Phi:\cA \to L^2(\bG)$.
  Here, $P:B_2(\bH)\to B_2(\bG)$ is a left invariant lift with associated wavelet $\Psi\in\bC^E$ and $\Phi:\cA\to \cA$ is a centering w.r.t.\ $K$.
  Moreover, assume the wavelet $\Psi$ to be AP-weakly cyclic in $\bC^E$ and such that $a_\Psi(\lambda)\neq0$ for any $\lambda\in E$.

  Then, the rotational bispectral invariants evaluated on lifted functions are complete on the set $\{ f\in \cC^{\text{AP}}\cap \bC^E\mid a_f(\lambda)\neq0 \text{ for all } \lambda\in E \}$ w.r.t.\ the action of elements of $K\times\bK\subset \bG$.
  Here, $\cC^{\text{AP}}$ is the set of AP-weakly cyclic functions on $\bC^E$ introduced in Definition~\ref{def:APweakly-cyclic}.
\end{theorem}

\begin{proof}
  To prove the result it suffices to replay step by step the proof of the previous theorem.
  The only point where one has to pay attention is step~4.
  Indeed, the arguments employed there allow to show that $U(T^\lambda)=T^\lambda(o,k)$ for some $k\in\bK$ when $\lambda\in E$ is such that either $\lambda=\lambda_1+\phi(h)\lambda_2$ for some $h\in\bK$ and a couple $(\lambda_1,\lambda_2)\in I^\otimes$ or there exists $\lambda'$ such that $(\lambda,\lambda')\in I^\otimes$.
  The fact that one of these properties is always satisfied for any $\lambda\in E$ is a consequence of the bispectral invariance of $E$.
\end{proof}

\subsubsection{Auxiliary lemma for the proof of Theorem~\ref{thm:rot-bisp-completeness}} 
\label{ssub:auxiliary_lemmata_rotational}

  \begin{lemma}
    \label{lem:comm-tensor-prod-rot}
    Let $(\lambda_1,\lambda_2)\in I^\otimes$ be such that
    \begin{enumerate}
      \item $\widehat{\Psi^*}_{\lambda_1},\,\widehat{\Psi^*}_{\lambda_2}\neq 0$,
      \item for any $k\in\bK$ it holds that $\widehat{\Psi^*}_{\lambda_1}\widehat{\Psi^*}_{\phi(k)\lambda_2} \neq 0$.
    \end{enumerate}
    Then, if $\hat f(\phi(\ell) \lambda_2)$ and $\hat g(\phi(\ell) \lambda_2)\neq 0$ for all $\ell\in\bK$ it holds
    \begin{equation}
      U(T^{\lambda_1}\otimes T^{\lambda_2}) = U(T^{\lambda_1})\otimes U(T^{\lambda_2}) \qquad .
    \end{equation}
  \end{lemma}

  \begin{proof}
    Let $C\coloneqq A\circ U(T^{\lambda_1})^*\otimes U(T^{\lambda_2})^*\circ A^*$.
    By the Induction-Reduction Theorem, to complete the proof it suffices to prove that $C = \bigoplus_{k\in\bK} U(T^{\lambda_1+\phi(k)\lambda_2})$.

    Since $U(T^\lambda)$ is circulant, we write $U(T^\lambda) = \sum_{j\in\bK} u_j(\lambda)S(j)$ where $u(\lambda)\in L^2(\bK)$.
    By Proposition~\ref{prop:equivalence-prop} we have that $(A\circ(S^i\otimes S^j)\circ A^*)_{k,\ell} = \delta_{k\ell^{-1},j^{-1}i} S(i)$, the block $C_{k,\ell}$ of $C$ is 
    \begin{equation}
      C_{k,\ell} = \sum_{j} u_j(\lambda_1) u_{jk^{-1}\ell)}(\lambda_2) S(j).
    \end{equation}
    This proves that $C_{k,\ell}$ is circulant and that $C$ is block-circulant, i.e., $C_{k,\ell}=C_{k\alpha,\ell\alpha}$.

    Let $v_k^\alpha = \widehat f_{\lambda_1}\widehat f_{\phi(k\alpha)\lambda_2}$.
    We claim that, for any $k\in\bK$, the vectors $\{v_\alpha^k\}_{\alpha\in\bK}$ form a basis of $L^2(\bK)$.
    Indeed, fix any ordering of $\bK$ and let $V$ be the matrix with $\alpha$-th row $v_\alpha^k$. Then,
    \begin{equation}
       V = \diag(\widehat f_{\lambda_1})\, \Circ (\widehat f_{\phi(k)\lambda_2}).
    \end{equation}
    Since, by assumption, the two matrices on the r.h.s.\ are invertible, the same is true for $V$. This proves the claim.
    
    The above claim allows us to define the operator $D_k\in\cL(L^2(\bK)$ by
    \begin{equation}
      D_k v^k_\alpha = \sum_{h\in\bK} C_{k,kh} v^k_{\alpha h} \qquad\forall \alpha\in\bK.
    \end{equation}
    Observe now that, by definition of $A$, it holds
    \begin{equation}
      v^\alpha_k = \widehat f_{\lambda_1}\widehat f_{\phi(k\alpha)\lambda_2} = p_k\circ A .(\widehat f_{\lambda_1}\otimes \widehat f_{\phi(\alpha)\lambda_2}). 
    \end{equation}
    Here, $p_k:\bigoplus_{k\in\bK}L^2(\bK)\to L^2(\bK)$ is the projection on the $k$-th component. Similarly, since $v^\alpha_\ell= v^{\alpha k^{-1}\ell}_k$,
    \begin{equation}
      \begin{split}
          (U(T^{\lambda_1})^*\widehat f_{\lambda_1})(U(T^{\lambda_2})^*\widehat f_{\phi(k\alpha)\lambda_2}) 
          & = p_k\circ A\circ U(T^{\lambda_1})^*\otimes U(T^{\lambda_2})^* (\widehat f_{\lambda_1}\otimes \widehat f_{\phi(\alpha)\lambda_2})\\
          &= p_k\circ C \circ A .(\widehat f_{\lambda_1}\otimes \widehat f_{\phi(\alpha)\lambda_2})\\
          & = \sum_{\ell} C_{k,\ell} p_\ell\circ A.(\widehat f_{\lambda_1}\otimes \widehat f_{\phi(\alpha)\lambda_2}\\
          &= \sum_{\ell} C_{k,\ell} v^{\alpha k^{-1}\ell}_k\\
          &= \sum_{h} C_{k,kh} v^{\alpha h}_k\\
          &= D_k v_k^\alpha
      \end{split}
    \end{equation}

    By the above equations, Corollary~\ref{cor:reduced-rot-inv}, and the unitarity of $U$, the equality of the rotational bispectral invariants implies
    \begin{equation}
      0 = \left\langle \widehat{f}_{\lambda_1+\phi(k)\lambda_2}, \left(U(T^{\lambda_1+\phi(k)\lambda_2})\circ D_k-\idty\right)v_k^\alpha \right\rangle, \qquad \forall \alpha,k\in\bK.
    \end{equation}
    Moreover, recall that $U(T^{\lambda})=U(T^{\phi(h)\lambda})$ for any $h\in\bK$. 
    Thus making the change of variables $(\lambda_1,\lambda_2)\mapsto (\phi(n^{-1})\lambda_1,\phi(n^{-1})\lambda_2)$ in the above yields
    \begin{equation}
      \label{eq:final-rot-inv}
      0 = \left\langle S(n)\widehat{f}_{\lambda_1+\phi(k)\lambda_2}, \left(U(T^{\lambda_1+\phi(k)\lambda_2})\circ D_k-\idty\right) S(n)v_k^\alpha\right\rangle, \qquad \forall \alpha,k,n\in\bK.
    \end{equation}
    Observe that for any $k,n\in\bK$, it holds that $\spn_{\alpha\in\bK}\{S(n)v_k^\alpha\} = L^2(\bK)$.
    Thus, the above is equivalent to 
    \begin{equation}
      \range \left(U(T^{\lambda_1 + R_k \lambda_2})^*\circ D_k - \idty\right) \perp S(n)\widehat{f}_{\lambda_1+\phi(k)\lambda_2} \qquad \forall k,n\in\bK.
    \end{equation}
    Since $\widehat{f}_{\lambda_1+\phi(k)\lambda_2}$, is cyclic it then follows that $U(T^{\lambda_1 + R_k \lambda_2})^*\circ D_{k} = \idty$ for any $k\in\bK$.
    By Lemma~\ref{lem:circulant-sum} this implies that
    \begin{equation}
      C_{k,\ell} = 
      \begin{cases}
        U(T^{\lambda_1+\phi(k)\lambda_2}) & \quad\text{if } k=\ell,\\
        0 & \quad\text{otherwise,}
      \end{cases}
    \end{equation}
    completing the proof of the statement.
    \end{proof}

  \begin{lemma}
    \label{lem:circulant-sum}
    Let $v,w\in L^2(\bK)$ be cyclic vectors and $e = A(v\otimes w)$.
    Moreover, let $C\in\cL(L^2(\bK))$ be the operator defined on the basis $e_j=A_j(v\otimes w)$ as
    \begin{equation}
      C e_j= \sum_{k\in\bK} C_{k} e_{k^{-1}j},
    \end{equation}
    where $C_{k}$ are circulant operators on $L^2(\bK)$.
    Then $C=\idty$ if and only if $C_0=\idty$ and $C_{k}=0$ for any $k\neq0$.
  \end{lemma}

  \begin{proof}
    The fact that $\{e_j\}_{j\in\bK}$ be a basis of $L^2(\bK)$ follows from the same argument used in Lemma~\ref{lem:comm-tensor-prod-rot}. Moreover, the sufficient part of the statement is obvious. 
    
    Let us assume that $C=\idty$ and define
    \begin{equation}
      \mathfrak C = \sum_{k\in\bK} \left(\bigoplus_{h\in\bK}C_{k}\right)\tilde S(k) \in\cL\left(\bigoplus_{h\in\bK}L^2(\bK)\right).
    \end{equation}
    A simple computation shows that $\mathfrak C_{j,\ell} = C_{\ell^{-1}j}$, which implies that $\mathfrak C e = (C e_j)_j = e$.

    Since $C_{k}$ is circulant, for any $k\in\bK$ there exists $c^{k}\in L^2(\bK)$ such that $C_{k} = \sum_{j\in\bK} c^{k}_j S(j)$.
    Thus, by Proposition~\ref{prop:equivalence-prop}, the fact that $\tilde S(k)$ and $\bigoplus_{h\in\bK}S(j)$ commutes, and that $e = A(v\otimes w)$, we obtain
    \begin{equation}
      \begin{split}
        v\otimes w 
            &= A^*\circ\mathfrak C\circ A (v\otimes w) \\
            &= \sum_{k,j\in\bK} c^{k}_j A^*\circ \left(\bigoplus_{h\in\bK} S(j)\right)\circ \tilde S(k)\circ A (v\otimes w)\\
            &= \sum_{k,j\in\bK} c^{k}_j S^{k^{-1}j}\otimes S^j (v\otimes w).
      \end{split}
    \end{equation}
    By cyclicity of $v$ and $w$, $\{(S^\alpha\otimes S^\beta)(v\otimes w)\}_{\alpha,\beta\in\bK}$ is a basis of $L^2(\bK)\otimes L^2(\bK)$, and thus applying $S^\alpha\otimes S^{\beta}$ to both sides of the above yields
    \begin{equation}
      \sum_{k,j\in\bK} c^{k}_j S^{k^{-1}j}\otimes S^j = \idty.
    \end{equation}
    Finally, this is equivalent to $c_e^e=1$ and $c_j^{k} = 0$ if $j$ or $k\neq e$, which proves the statement.
  \end{proof}



\subsection{Real valued functions} 
\label{sub:real_valued_functions}

As discussed in Section~\ref{ssub:real_valued_functions}, when the action of $\bK$ is \emph{even} (see Definition~\ref{def:even}) it holds that $\cC\cap L^2_\bR(\bR^2)=\varnothing$.
That is, Theorem~\ref{thm:rot-bisp-completeness} gives no information on real valued functions.
In this section we will show how to exploit the tools introduced in Section~\ref{ssub:real_valued_functions} to obtain the completeness for real valued functions.

\begin{theorem}
  \label{thm:rot-bisp-complete-real}
  Let the assumptions of Theorem~\ref{thm:rot-bisp-completeness} to be satisfied.
  Moreover, assume $\cA\subset L^2_\bR(\bH)$.

  Then, the rotational bispectral invariants evaluated on lifted functions are weakly complete on $\cA$ w.r.t.\ the action of elements of $U\times\bK\subset \bG$.
  Here, $\cC_\bR$ is the set of weakly $\bR$-cyclic functions introduced in Definition~\ref{def:cc-weakly-cyclic-real}.
\end{theorem}

\begin{proof}
  If $\bK$ is not even, since $\cC_\bR = \cC \cap L^2_\bR(\bH)$, the result is simply a restatement of Theorem~\ref{thm:rot-bisp-completeness}.
  Thus we only need to prove the result for $\bK$ even.

  Recall the notations introduced in Section~\ref{ssub:real_valued_functions} and define
  \begin{equation}
    \cY = \left\{ v \in L^2(\bK) \mid v(h) = -\overline{v(h+k_0)} \quad\forall h\in\bK \right\}.
  \end{equation}
  Considering the realification of $L^2(\bK)$, it splits $\bR$-orthogonally as $L^2(\bK) \cong \cX \oplus \cY$.

  We need the following observations:
  \begin{itemize}
    \item From the invariance w.r.t.\ the shifts of $\cX$ it follows that the equivalence $A$ restricts to an equivalence between $\cX\otimes\cX$ and $\bigoplus_{k\in\bK}\cX$.
    This allows us to define $A_\bR = B^{-1}\circ A\circ B$.
    \item From Proposition~\ref{prop:ft-lift}, for any $\lambda\in\widehat{\bH}$, it follows that $\ker\widehat{Lf}(T^\lambda)\supset \cY$ and that $\cX$ is an invariant subspace for $\widehat{Lf}(T^\lambda)$.
    Thus we define $\widehat{Lf}_\bR(T^\lambda)=B^{-1}\circ \widehat{Lf}(T^\lambda) \circ B$.
  \end{itemize}

  Let $f,g\in\cC_\bR\cap \cA$ satisfying the conditions in the statement.
  To complete the proof it now suffices to show that there exists $k\in\bK$ such that $\widehat{Lf}_\bR(T^\lambda) \circ S_\bR^k = \widehat{Lg}_\bR(T^\lambda)$ for all $\lambda\in\widehat{\bH}\setminus\{\hat o\}$.
  Indeed, since $\ker\widehat{Lf}(T^\lambda)\supset \cY$ and $S(\cY)=\cY$, this implies that $\widehat{Lf}(T^\lambda) \circ S^k = \widehat{Lg}(T^\lambda)$.

  To this aim, let $I\subset\widehat{\bH}$ be the set of $\lambda$'s such that $\Circ_\bR B^{-1}\widehat f_\lambda $ and $\Circ_\bR B^{-1}\widehat g_\lambda$ are invertible.
  We then let
  \begin{equation}
    U_\bR(T^\lambda) = \Circ_\bR B^{-1}\widehat g_\lambda \left(\Circ_\bR B^{-1}\widehat f_\lambda\right)^{-1} \qquad \forall \lambda\in I.
  \end{equation}
  Let also $I^\otimes$ to be the set of couples $(\lambda_1,\lambda_2)\in\widehat{\bH}\times \widehat{\bH}$ such that $\lambda_1+\phi(k)\lambda_2\in I$ for any $k\in\bK$ and define
  \begin{equation}
    U_\bR(T^{\lambda_1}\otimes T^{\lambda_2}) = A^*_\bR \circ \left(\bigoplus_{k\in\bK} U_\bR(T^{\lambda_1+\phi(k)\lambda_2})\right)\circ A_\bR \qquad \forall (\lambda_1,\lambda_2)\in I^\otimes.
  \end{equation}

  With these definitions to obtain that $\widehat{Lf}_\bR(T^\lambda) \circ S_\bR^k = \widehat{Lg}_\bR(T^\lambda)$ when the rotational bispectral invariants of $f$ and $g$ coincide it suffices to replay the exact same arguments of Theorem~\ref{thm:rot-bisp-completeness}, substituting $S_\bR$, $A_\bR$, $\Circ_\bR$ and $\widehat{Lf}_\bR$ to $S$, $A$, $\Circ$ and $\widehat{Lf}$, respectively.
\end{proof}

As for Theorem~\ref{thm:rot-bisp-completeness}, the above theorem can be easily adapted to $\cT(E)\subset B_2(\bH)$.
Observe that, for $f\in\cT(E)$ be real-valued it is necessary that $E=-E$.

\begin{theorem}
  \label{thm:AP-rot-bisp-inv-complete-real}
  Let the assumptions of Theorem~\ref{thm:AP-rot-bisp-inv-complete} to be satisfied.
  Moreover, assume $E=-E$, and that the wavelet $\Psi$ is real valued.

  Then, the rotational bispectral invariants evaluated on lifted functions are complete on the set $\{ f\in\cC_\bR^{\text {AP}}\cap\cT(E) \mid f \text{ is real-valued and } a_f(\lambda)\neq 0 \text{ for all } \lambda\in E \}$ w.r.t.\ the action of elements of $K\times\bK\subset\bG$.
  Here, $\cC_\bR^{\text{AP}}$ is the set of AP-weakly $\bR$-cyclic functions introduced in Definition~\ref{def:APweakly-cyclic}.
\end{theorem}


\subsection{Compactly supported real-valued functions on $\bR^2$}
\label{sec:compactly_real_value}

In this section we particularize and extend the results of the previous section to the case of $\bG=SE(2,N)$ and to real-valued compactly supported functions on the plane, introduced in Section~\ref{sec:square-integrable-functions-on-the-plane}.
The natural choice for a lift is $L=P\circ\Phi_c:\cV(D_r)\to L^2(SE(2,N))$, where $P$ is a regular left-invariant lift with a real valued associated wavelet, while $\Phi_c:\cV(D_r)\to \cV(D_R)$ is the centering operator defined in Section~\ref{sec:square-integrable-functions-on-the-plane}

Recall that in this context $\bZ_N$ is even if $N$ is even and not even if $N$ is odd.
The main theorem of this section is the following.

\begin{theorem}
  Assume the wavelet $\Psi$ to be weakly $\bR$-cyclic, and such that $\hat \Psi\neq 0$ a.e.\ on $\widehat{\bR^2}$.
  Then, two weakly $\bR$-cyclic functions in $\cV(D_R)$ can be deduced via the action of $SE(2,N)$ if these two conditions are satisfied
  \begin{itemize}
    \item their bispectral invariants coincide a.e.\ (i.e., $I_f^2(\lambda_1,\lambda_2)=I_g^2(\lambda_1,\lambda_2)$ for a.e.\ $(\lambda_1,\lambda_2)$);
    \item their rotational bispectral invariants coincide on an open set (i.e., $\RBS_f(\lambda_1,\lambda_2,k)=\RBS_g(\lambda_1,\lambda_2,k)$ for any $k\in\bZ_N$ and $(\lambda_1,\lambda_2)$ in an open set).
  \end{itemize}

  In particular, the rotational bispectral invariants evaluated on lifted functions are weakly complete on  $\cV(D_R)\cap L^2_\bR(\bR^2)$ and discriminate on an open and dense set.
\end{theorem}

\begin{proof}
  The fact that $\cC_\bR$ is an open dense subset of $\cV(D_R)$ is proved in Theorem~\ref{thm:genericity-weakly-cyclic}.
  The result then follows from Theorems~\ref{thm:rot-bisp-completeness} and \ref{thm:rot-bisp-complete-real}.
  Indeed, the Fourier transform $\hat f$ of $f\in\cV(D_R)$ is analytic and hence $\hat f\neq 0$ on an open and dense set.
  We then proceed exactly as in the proofs of Theorems~\ref{thm:rot-bisp-completeness} and \ref{thm:rot-bisp-complete-real}, observing that $\lambda\mapsto U(T^\lambda)$ is now an analytic function.
  Since the rotational bispectral invariants coincide only on an open set $V\subset \widehat{\bR^2}\times \widehat{\bR^2}$, the commutation with the tensor product of step $1.3$ holds only on there.
  However, this allows to jump directly to step $4$ and prove that there exists $k\in\bZ_N$ such that $U(T^\lambda)\equiv S^k$ for any $\lambda$ in a section of $V$.
  Then, $U(T^\lambda)\equiv S^k$ everywhere by analyticity, and the proof can be concluded.
\end{proof}

\section{Bispectral invariants for almost-periodic functions}
\label{sec:bisp-inv-ap}


  Let $\bG$ be a MAP group, in the sense of Section~\ref{sssec:MAP-and-AP}.
  Consider the set $B_2(\bG)$ of Besicovitch almost periodic functions.
  Since $\sigma^*:L^2(\bG^\flat)\to B_2(\bG)$ is an isomorphism of Hilbert spaces, we define the spectral invariants of $f\in B_2(\bG)$ as $\PS_f\coloneqq \PS_{f'}$ and $\BS_f\coloneqq \BS_{f'}$, where $f'\in L^2(\bG^\flat)$ is such that $f'\circ \sigma=f$.

  Since $L^2(\bG^\flat)$ is the space of square integrable functions, one could be induced to think the (weak) completeness of the invariants on $B_2(\bG)$ functions to be a consequence of the results of Section~\ref{thm:compact-completeness}.
  However this is not true, due to the lack of separability of $B_2(\bG)$ and to the fact that $\bG^\flat$ is much bigger\footnote{For example, observe that the measure of $\sigma(\bG)$ in $\bG^\flat$ is zero.} than $\bG$.

  Let $\cE$ be the image under $\sigma^*$ of the set $\cG$ defined in \eqref{eq:def-cG-compact}, that is
  \begin{equation}
    \cE=\left\{\sigma^* f' \mid f'\in L^2(\bG^\flat) \text{ and } \widehat{f'}(T) \text{ is invertible for all } T\in \widehat{\bG^\flat}\right\}.
  \end{equation}
  We then have the following.

  \begin{theorem}
    Let $\bG$ be a non-compact group.
    Then, the bispectral invariants are not complete on $B_2(\bG)$ nor they are complete on $\cE$.
  \end{theorem}

  \begin{proof}
    The main observation is that since $\bG$ is non-compact, it holds that $\bG\not\cong \bG^\flat$, and in particular $\bG^\flat\setminus\sigma(\bG)\neq\varnothing$.
    Indeed, consider $f',g'\in L^2(\bG^\flat)$ be such that $f'=\Lambda^\flat(\xi)g$, where $\Lambda^\flat$ is the left regular representation on $\bG^\flat$.
    Then, $f\coloneqq f'\circ\sigma$ and $g\coloneqq g'\circ\sigma$ satisfy $\BS_f=\BS_g$.
    However, if $\xi\in\bG^\flat\setminus\bG$ they cannot be deduced via $\Lambda=\Lambda^\flat\circ\sigma$, the left regular representation of $\bG$.
    This proves that bispectral invariants cannot be complete on $B_2(\bG)$.
    To complete the proof it suffices to observe that if $f'\circ\sigma\in\cE$ then $g'\circ\sigma\in\cE$, since $\widehat{g'}(T)=\widehat{f'}(T)\circ R(\xi)$.
  \end{proof}

  \begin{remark}
    The above proof actually shows that the bispectral invariants do not discriminate on any subset of $B_2(\bG)$ containing $f=f'\circ\sigma$ and $g=g'\circ\sigma$ such that $f'=\Lambda^\flat(\xi)g'$ for some $\xi\in\bG^\flat\setminus \sigma(\bG)$ .
  \end{remark}

  Regarding weak completeness of the invariants, let us restrict to the case $\bG=\bK\ltimes\bH$ considered  in Section~\ref{sssec:moore-semi-direct} (which taking $\bK=\{e\}$ contains the case of an abelian $\bG$).
  As already mentioned, $\bG$ is a MAP group and $\bG^\flat=\bK\ltimes\bH^\flat$, where the action of $\bK$ is obtained by density of the injection of $\bH$ in $\bH^\flat$.
  In this case functions $f\in B_2(\bG)$ are exactly those such that $f(k,\cdot)\in B_2(\bH)$ for any $k\in\bK$.

  Recall also that the unitary irreducible representations of $\bG^\flat$ are in bijection with those of $\bG$ and are parametrized by $\lambda\in\widehat{\bH^\flat}=\widehat\bH_d$ and $\hat k\in\widehat\bK$.
  Observe that the topology w.r.t.\ the $\lambda$ variable is the discrete one.


  We now describe the natural subsets of $B_2(\bG)$ that we consider for the weak completeness.
  Fix a bispectrally admissible set $F\subset\widehat\cS\subset\widehat\bH_d$ and decompose $\tilde F = \bigcup_{k\in\bK}\phi(k)F$ as $\tilde F = \tilde F_1\cup \tilde F_2$ as in Definition~\ref{def:bispectrally-admissible-set}.
  Then, consider the set 
  \begin{equation}
    \label{eq:residual-AP}
    \cG_F = \left\{ \sigma^* f'\in B_2(\bG) \mid 
    \begin{array}{l}
      \supp \cF (f'(\cdot,k))\subset \tilde F,\quad \forall  k\in\bK \text{ and }\\
        \widehat {f'}(T^\lambda)\text{ is invertible } \forall\lambda\in \tilde F_1  
    \end{array}
    \right\}.
  \end{equation}
  Depending on the structure of $F$, problems can arise even in this case:

  \begin{proposition}
    \label{prop:counterexample}
    Let $\cG_F\subset B_2(\bG)$ be the set defined in \eqref{eq:residual-AP} and corresponding to the bispectrally admissible set $F$.
    Then, if $\tilde F$ is a subgroup of $\widehat \bH_d$ which is dense in $\widehat\bH$ w.r.t.\ the usual topology, the bispectral invariants are not complete on $\cG_F$.
%
  \end{proposition}

  \begin{proof}
    Since $\tilde F$ is a dense subgroup of a locally compact abelian group, by Pontryagin duality it obviously hold $\bH \subset(\tilde F)^{\hat{}}$.
    However, due to the presence of an accumulation point, the discrete topology of $\tilde F\subset\widehat\bH_d$ is finer than that induced by $\widehat \bH$.
    Thus, by \cite{Hewitt1963a} there exist $\chi\in (\tilde F)^{\hat{}}\setminus \bH$.

    Given any $f = \sigma^*f'\in\cG$, define $g = \sigma^*g'$ letting $g'(k,\cdot)$ be the inverse Fourier transform of $\lambda\mapsto \chi(\lambda) \cF(f'(k,\cdot)))(\lambda)$, for any $k\in\bK$.
    By definition, $\hat g'(T^\lambda)=\chi(\lambda)\hat f'(T^\lambda)$.
    Since $\chi(\lambda)\neq 0$ everywhere, this implies that $\hat g'(T^\lambda)$ is invertible for $\lambda\in F$ and hence that $g\in\cG$.
    Moreover, using the fact that $\chi$ is a character of $\tilde F$, from Propositions~\ref{prop:FT-semidirect} and \ref{prop:moore-inv-expression} follows that $\BS_f=\BS_g$.
    Finally, since $\chi$ is not a character of $\widehat{H}$, we have that $g\neq\pi_{B_2}(a)f$ for any $a\in\bG$, proving the statement.
  \end{proof}


\subsection{Almost-periodic functions on the plane}

Let us fix a countable bispectrally invariant set $F$ of frequencies of $\bR^2$, such that $\tilde F=-\tilde F$, and consider the set $\AP_F(\bR^2)\subset B_2(\bR^2)$ of almost-periodic functions with frequencies in $\tilde F$, as introduced in Section~\ref{sub:finite_dimensional_subspaces_of_almost_periodic_functions}.
Moreover, let us consider a left-invariant lift $P:B_2(\bR^2)\to B_2(SE(2,N))$ with associated wavelet $\Psi\in\AP_F(\bR^2)$, a compact $K\subset \bR^2$, and the centering operator $\Phi:\cR_K\to \cR_K$ defined in Section~\ref{ssec:centering_almost_periodic_functions}.
We recall that the left-invariant lift $P$ is obtained from a left-invariant lift $P':L^2((\bR^2)^\flat)\to L^2(SE(2,N)^\flat)$ via the isomorphism $\sigma^*:L^2((\bR^2)^\flat)\to B_2(\bR^2)$ defined in Section~\ref{sssec:MAP-and-AP}.

\begin{theorem}
  Assume the wavelet $\Psi\in\AP_F(\bR^2)$ is real-valued, AP-weakly $\bR$-cyclic, and such that $\hat\Psi(n,\lambda)\neq0$ for all $\lambda\in F$ and $n\in\bK$.
  Then, the rotational bispectral invariants evaluated on lifted functions are weakly complete on real-valued functions of $\AP_F(\bR^2)$ w.r.t.\ the action of elements of $K\times\bZ_N\subset SE(2,N)$.
\end{theorem}

\begin{proof}
  By Theorem~\ref{thm:AP-rot-bisp-inv-complete-real} it suffices to show that the set
  \begin{equation}
    \{ f\in \cC_\bR^{\text{AP}}\cap \AP_F(\bR^2)\mid \hat f(n,\lambda)\neq 0 \text{ for all } \lambda\in F, \,n\in\bK\}
  \end{equation}
  is residual in the set of real-valued function of $\AP_F(\bR^2)$.
  This follows from Theorem~\ref{thm:genericity-AP-weakly-cyclic} and the fact that 
  \begin{equation}
    \{ f\in \AP_F(\bR^2) \mid f\text{ is real-valued and } \hat f(n,\lambda)\neq 0 \text{ for all } \lambda\in F,\,n\in\bK \}
  \end{equation}
  is residual, which is an immediate consequence of the countability of $F$ and the fact that $\bC\setminus\{0\}$ is open and dense.
\end{proof}

\begin{remark}
  From the proof of the previous theorem it follows that, when $F$ is finite, the rotational bispectral invariants discriminate on an open and dense subset of real-valued functions of $\AP_F(\bR^2)$.
\end{remark}

%% file: inpainting.tex
In this chapter we apply non-commutative Fourier analysis in order to build numerically efficient algorithms for heat diffusion on groups and its application to image reconstruction. 
In this part, we will consider only the case $\bG=SE(2,N)$.
That is, we assume $\bH = \bR^2$, with coordinates $x = (x_1,x_2)$, and $\bK=\bZ_N$.

\section{Hypoelliptic diffusions on Lie groups}

In this section we discuss a general technique to obtain hypoelliptic operators on a unimodular Lie group $\bG$ of type I with Lie algebra $\mathfrak g$. Recall that, thanks to the identification $T_o\bG = \mathfrak g$, to any element $h\in\mathfrak g$ we can naturally associate the left-invariant vector field $X(g)=gh$ on $\bG$.

The following theorem is classical, see, e.g., \cite{Agrachev2009}.

\begin{theorem}\label{thm:hormander}
  Let $\bG$ be a Lie group with Lie algebra $\mathfrak g$, and let $\mathfrak p\subset \mathfrak g$ be a subspace of $\mathfrak g$ satisfying the \emph{H\"ormander condition} (also known as Lie bracket generating condition):
  \begin{equation}
    \operatorname{Lie}\mathfrak p := \spn\left\{ [p_1,[p_2,\ldots,[p_{n-1},p_n]]] \mid n\in\bN, \, p_i\in\mathfrak p \right\} = \mathfrak g.
  \end{equation}
  Fix any basis $\{p_1,\ldots,p_k\}$ of $\mathfrak p$, and let $X_i(g)=gp_i$ be the associated left-invariants vector fields.
  Then, letting $L_{X_i}$ be the Lie derivative w.r.t.\ $X_i$, the operator
  \begin{equation}
    \Delta = \sum_{i=1}^k L_{X_i}^2,
  \end{equation}
  is essentially self-adjoint on $L^2(\bG)$ and hypoelliptic.
\end{theorem}

\begin{remark}
  As shown in \cite{Agrachev2009}, the operator defined above is actually the intrinsic sub-Laplacian of $\bG$ for the left-invariant sub-Riemannian structure on $\bG$ with distribution $\mathcal D(g) = g\mathfrak p$ endowed with scalar product for which $\{g p_1,\ldots,g p_k\}$ is an orthonormal basis.
\end{remark}

We are interested in the hypoelliptic heat equation 
\begin{equation}\label{eq:heat-se2}
  \partial_t f = \Delta f.
\end{equation}
It follows from classical results, that the above equation defines a (Markov) semi-group $e^{t\Delta}$ and, due to the left-invariance of the $X_i's$, it has a right-convolution kernel $p_t(\cdot)$. Namely,
\begin{equation}
  e^{t\Delta}f_0(g) = f_0\star p_t(g) = \int_\bG f_0(h)p_t(h^{-1}g)\,dh,\qquad \text{for all }f_0\in L^2(SE(2)).
\end{equation}
Moreover, the hypoellipticity of $\Delta$ guarantees that $p_t\in C^\infty(\bR_+\times SE(2))$ and that $p_t>0$ for $t>0$, i.e., that the heat diffusion \eqref{eq:heat-se2} has infinite speed of propagation.

Finally, it is classical that the operator $\Delta$ is the generator of the Markov process associated with the SDE
\begin{equation}
  dY_t = X_1(Z_t) dW^1_t + X_2(Z_t)\,dW^2_t,
\end{equation}
where $W^1_t,W^2_t$ are two independent Wiener processes on $\bR$. From this point of view, the heat kernel is the transition density of the process. That is, $p_t(\cdot)$ is the probability of $Y_t$ if $Y_0=0$. 

Let us focus on the case $\bG = SE(2)$, which satisfies the above assumptions. Its Lie algebra is $\mathfrak{se}(2)$, and admits a basis $\{\mathfrak{e}_1,\mathfrak{e}_2,\mathfrak{e}_3\}$ with the following commutation relations
\begin{equation}
  [\mathfrak{e}_1,\mathfrak{e}_3]=\mathfrak{e}_2, \quad   [\mathfrak{e}_1,\mathfrak{e}_2]=-\mathfrak{e}_3, \quad [\mathfrak{e}_2,\mathfrak{e}_3]=0.
\end{equation}
In particular, the subspace $\mathfrak p = \spn\{e_1,e_2\}$ satisfies the H\''ormander condition.

As usual, $T_eSE(2)\simeq \mathfrak{se}(2)$, and we can build a family of left invariant vector fields $\{X_1,X_2,X_3\}$ by left translation of the basis $\{\mathfrak{e}_1,\mathfrak{e}_2,\mathfrak{e}_3\}$. In $(\theta,x,y)$ coordinates these are
\begin{equation}
  X_1 = \partial_\theta,\quad X_2 = \cos\theta\partial_x+\sin\theta\partial_y, \quad X_3 = -\sin\theta\partial_x+\cos\theta\partial_x.
\end{equation}
Then, by Theorem~\ref{thm:hormander}, the operator
\begin{equation}\label{eq:citti-sarti}
  \Delta = X_1^2+X_2^2 = \partial_\theta^2 + \cos^2\theta\partial_x^2 + \frac12 \sin(2\theta) \partial_{xy} + \sin^2\theta\partial_y^2
\end{equation}
is hypoelliptic and essentially self-adjoint on $L^2(SE(2))$. 
Observe that, due to the special form of $X_1$ and $X_2$, letting $Y_t = (\Theta_t, Z_t)$, where $\Theta_t\in\bS^1$ and $Z_t\in\bR^2$ the SDE simplifies to
\begin{equation}\label{eq:sde-se2}
  dZ_t =   \left(
  \begin{array}{c}
  \cos \Theta_t \\ \sin \Theta_t    
  \end{array}
  \right)
  dW^1_t, \qquad 
  d\Theta_t = dW^2_t.
\end{equation}

To conclude this section, we recall a result from \cite{ABG}, see also \cite{Duits2009}, giving an explicit formula for the heat kernel on $SE(2)$ via the non-commutative Fourier transform on $SE(2)$. 
Via Mackey's machinery one obtains that unitary irreducible representations of $SE(2)$ are parametrized by the disjoint union of the real half-line $(0,+\infty)$ with $\bS^1$. Similarly to the semi-discrete case, the Plancherel measure is supported only on $(0,+\infty)$ and is $\lambda d\lambda$, and the corresponding representations act on $L^2(\bS^1)$ via
\begin{equation}
  [\mathfrak{X}^\lambda(\theta,\rho e^{i\varphi})\psi](\alpha) = e^{i\lambda\rho\cos(\varphi-\theta)}\psi(\alpha+\theta),\qquad \forall \psi\in L^2(\bS^1),\, \lambda>0.
\end{equation}
We then have the following.

\begin{theorem}
  The kernel of the hypoelliptic heat equation on $SE(2)$ is
  \begin{multline}
    p_t(\theta,\rho e^{i\varphi}) = \frac12 \int_0^{+\infty} \sum_{n=0}^{+\infty} e^{t a_n^\lambda} \left\langle \ce_n\left(\theta,\frac{\lambda^2}{4}\right), \mathfrak{X}^\lambda(\theta,\rho e^{i\varphi})\ce_n\left(\theta,\frac{\lambda^2}{4}\right) \right\rangle \\
    +\sum_{n=0}^{+\infty} e^{t b_n^\lambda}\left\langle \se_n\left(\theta,\frac{\lambda^2}{4}\right), \mathfrak{X}^\lambda(\theta,\rho e^{i\varphi})\se_n\left(\theta,\frac{\lambda^2}{4}\right) \right\rangle \,\lambda d\lambda.
  \end{multline}
  Here, $\ce_n$ and $\se_n$ be the $2\pi$-periodic Mathieu sine and cosine, and $a_n^\lambda = -\lambda^2/4-a_n(\lambda^2/4)$, $b_n^\lambda = -\lambda^2/4-b_n(\lambda^2/4)$, where $a_n$ and $b_n$ are the characteristic values of the Mathieu equation\footnote{ 
The Mathieu equation is
\begin{equation}
  \partial_x^2 f(x) + (a - 2q\cos(2x))f(x) = 0, \qquad a,q\in\bR.
\end{equation}
For fixed $q\in\bR$, there exist two ordered discrete sets of characteristic values, $\{a_n\}_{n\in\bN}$ and $\{b_n\}_{n\in\bN}$, such that the Mathieu equation with $a=a_n$ (resp.\ $a=b_n$) admits a unique even (resp.\ odd) $2\pi$-periodic  solution with $L^2$ norm equal to $1$, the Mathieu cosine $\ce_n(x,q)$ (resp.\ the Mathieu sine $\se_n(x,q)$. 
}.
\end{theorem}

For a review of different and more numerically exploitable representations of the above kernel, we refer to \cite{ZhangDuits2016}.

\section{Diffusions on semi-discrete semi-direct products}

Let $\bG=\bK\ltimes\bH$ be as in Section~\ref{sec:repr-semidirect}, but assume in addition that $\bH$ be an (abelian) Lie group with Lie algebra $\mathfrak h$ and that, for any $k\in\bK$, $\phi(k)$ is a \emph{smooth} automorphism. Although this implies that $\bG$ itself is a Lie group, it is disconnected. Thus, the above approach to build a diffusion would simply yield $|\bK|$ disjoint diffusions on each of the component $\{k\}\times\bH$ for $k\in\bK$. 

The Lie algebra of $\bG$ is still $\mathfrak h$. In particular, any left-invariant vector field $X$ over $\bG$ is uniquely determined by its value at the origin $v\in\mathfrak h$ by $X(k,x) = \phi(k)_* v$, where $\phi:\bK\to \operatorname{Aut}(\bH)$ is the action of $\bK$ on $\bH$ and $\psi(k)_*:\mathfrak h \to \mathfrak h$ denotes its differential.

The above consideration, yields us to define diffusions starting from the probabilistic point of view. Namely, consider a left-invariant jump Markov process $K_t$ on $\bK$ and fix $v_1,\ldots,v_n\in\mathfrak h$. Then, letting $X_1,\ldots,X_n$ be the associated left-invariant vector fields over $\bG$, it makes sense to consider the following SDE on $\bH$:
\begin{equation}\label{eq:gen_sde}
  dZ_t = \sum_{i=1}^n X_i(K_t,Z_t) \,dW^i_t = \sum_{i=1}^n \phi(K_t)_* v_i \,dW^i_t,
\end{equation}
where $W^1_t,\ldots,W^k_t$ are independent Wiener processes on $\bR$. This yields a Markov process $(K_t,Z_t)$ on $\bG$, whose generator is the operator $\Delta$ on $L^2(\bG)$.

Let us denote by $L_X$ the Lie derivative by the left-invariant vector field $X$ over $\bG$. In particular, if $X(k,x)=\phi(k)_*v$ we have $L_X = \bigoplus_{k\in\bK}\phi(k)_* L_v$, where we identified $v$ with the associated left-invariant vector field over $\bH$.
Letting $\Xi$ be the infinitesimal generator of $K_t$, as an operator on $L^2(\bK)$, and with the identification $L^2(\bG)\simeq\bigoplus_{k\in\bK}L^2(\bH)$, we have that
\begin{equation}\label{eq:hypo}
  \Delta = \sum_{i=1}^n\left( L_{X_i}\right)^2+ \Xi\otimes\idty_\bH. = \bigoplus_{k\in\bK} \left(\sum_{i=1}^n\left(\phi(k)_* L_{v_i}\right)^2\right) + \Xi\otimes\idty_\bH.
\end{equation}
Observe that, since $K_t$ is a left-invariant Markov process, the matrix $\Xi$ is symmetric and circulant. In particular, $\Delta$ is a symmetric operator on $C^\infty_c(\bG,\bR)\subset L^2(\bG)$.
Moreover, let us observe that $\sum_{i=1}^n\left(\phi(k)_* L_{v_i}\right)^2$ is an hypoelliptic operator over $L^2(\bH)$ for all $k\in\bK$, which entails that $\Delta$ itself is hypoelliptic.

In the following we will discuss the heat equation associated with the operator $\Delta$ on $L^2(\bG)$. Namely, we consider the following equation for the Friedrichs extension of $\Delta$:
\begin{equation}\label{eq:heat}
	\partial_t \psi = \Delta \psi.
\end{equation}
Standard results then guarantees that the associated evolution semigroup $e^{t\Delta}$ on $L^2(\bG)$ is Markov and admits a right convolution kernel $p_t(\cdot)$ such that
\begin{equation}\label{eq:def-pt}
	e^{t\Delta}\phi(a) = \phi\star p_t(a) = \int_{\bG} \phi(b)p_t(b^{-1}a)\,db, \qquad \phi\in L^2(\bG).
\end{equation}
Moreover, since $\Delta$ is hypoelliptic, $(t,a)\mapsto p_t(a)$ is in $C^{\infty}_c(\bR_+\times \bG)$.

\begin{theorem}\label{thm:decomp-delta}
  Let $\Delta$ be the operator in \eqref{eq:hypo}. Then, it holds that
  \begin{equation}\label{eq:decomp-delta}
    \widehat\Delta = \cF_\bG\circ\Delta\circ\cF_\bG^{-1} = \int_{\cS}^{\bigoplus}\widehat{\Delta}^\lambda\,d\hat\mu_\bG,\qquad
      \widehat\Delta^\lambda = \diag_h(L_{v_i}^2(\phi(h)\lambda)(o))+\Xi.
  \end{equation}
  Moreover, the associated heat kernel is given by
  \begin{equation}
  	p_t(a) = \int_{\cS} \tr(e^{t\widehat\Delta^\lambda} T^\lambda(a))\,d\hat\mu_\bG.
  \end{equation}
\end{theorem}

\begin{proof}
The first part of the statement is a direct consequence of Lemmas~\ref{lem:decomp-vect} and \ref{lem:decomp-xi}.

In order to complete the proof we observe that \eqref{eq:decomp-delta} implies
\begin{equation}
	\widehat{e^{t\Delta}\phi}(T^\lambda) = e^{t\widehat{\Delta}^\lambda}\hat\phi(T^{\lambda}),\qquad \text{for all }\lambda\in\cS.
\end{equation}
On the other hand, by \eqref{eq:def-pt} it follows
\begin{equation}
	\widehat{e^{t\Delta}\phi}(T^{\lambda}) = \widehat{\phi\star p_t}(T^{\lambda}) = \widehat{p_t}(T^\lambda)\circ \hat\phi(T^\lambda),\qquad \text{for all }\lambda\in\cS.
\end{equation}
Putting together these two equations yields $\widehat{p_t}(T^\lambda)=e^{t\widehat{\Delta}^\lambda}$, proving the statement via the inverse Fourier transform formula.
\end{proof}

%


\begin{lemma}\label{lem:decomp-vect}
  Let $v\in\mathfrak h$, $X(k,x)=\phi(k)_*v$ be the associated left-invariant vector field over $\bG$, and $L_X$ be its Lie derivative.
  Then,
  \begin{equation}
    \widehat{L_X} = \int_{\cS}^{\bigoplus} \widehat{L_X}^\lambda\,d\hat\mu_\bG,
  \end{equation}
  where each $\widehat{X}^\lambda$ acts on the set of Hilbert-Schmidt operators over $L^2(\bK)$ via
  \begin{equation}
    \widehat{L_X}^\lambda = \diag_h\left[ L_v(\phi(h)\lambda)(o) \right].
  \end{equation}
\end{lemma}

\begin{proof}
  For any $(k,x)\in\bG$, let $\varrho:\bG\to\cU(L^2(\bG))$ be the right regular representation, i.e., $R_{(k,x)}f(h,y) = f((h,y)(k,x))=f(k+h, y+\phi(h)x)$.
  Then, simple computations yield, for any $f\in L^2(\bG)$ and $a\in\bG$,
  \begin{equation}
    \widehat{\varrho_{a}f}(T^\lambda) = T^\lambda(a)\circ \hat f(T^\lambda), \quad \forall T^\lambda\in\widehat\bG
    \iff \widehat{\varrho_a} = \int_{\widehat\bG}^\oplus T^\lambda(a).
  \end{equation}
  Observe that, by definition, it holds
  \begin{equation}
    L_Xf(k,x) = \frac{d}{dt}\bigg|_{t=0} \varrho_{e^{tX}}f(k,x).
  \end{equation}
  Since the Fourier transform commutes with the derivative in $t$ appearing above, we then have
  \begin{equation}
    \widehat{L_X} = \int_{\widehat\bG}^\oplus \frac{d}{dt}\bigg|_{t=0} T^\lambda(e^{tX}).
  \end{equation}
  Finally, the statement follows by observing that $e^tX = (0, e^{tv}_\bH)$, where $e_\bH^{\cdot}$ is the exponential function on $\bH$, and using the explicit formula for $T^\lambda$ given in Theorem~\ref{thm:repr-semidir}.
\end{proof}

\begin{lemma}\label{lem:decomp-xi}
  Let $\Xi$ be a circulant matrix on $L^2(\bK)$. Then, 
  \begin{equation}
    (\Xi\circ\idty_{\bH})^{\widehat{}} = \int_{\cS}^\oplus \Xi\,d\hat\mu_\bG.
  \end{equation}
\end{lemma}

\begin{proof}
  Let $\lambda\in\cS$. By Proposition~\ref{prop:FT-semidirect}, simple computations yield
  \begin{equation}
    \big(\Xi \hat f(T^\lambda)\big)_{i,j} = \sum_{\ell\in\bK} \Xi_{i,\ell} \cF(f(\ell^{-1}j,\cdot))(\phi(j)\lambda).
  \end{equation}
  Then the statement follows by observing that, since $\Xi_{i^{-1}j,\ell^{-1}j} = \Xi_{i,\ell}$, we have
  \begin{equation}
    \begin{split}
      \big( \cF_\bG[(\Xi\circ\idty_{\bH})f] \big)_{i,j} 
        &= \sum_{\ell\in\bK} \Xi_{i^{-1}j,\ell}\cF(f(\ell,\cdot))(\phi(j)\lambda)\\
        &=\sum_{\ell\in\bK} \Xi_{i,\ell}\cF(f(\ell^{-1}j,\cdot))(\phi(j)\lambda)\\
        &=\big(\Xi \hat f(T^\lambda)\big)_{i,j}.
    \end{split}
  \end{equation}
\end{proof}

\subsubsection{Diffusion of lifted functions}

For the image reconstruction algorithm, we are interested in apply the hypoelliptic diffusion we just described to regular left-invariant lifts to $L^2(\bG)$ of functions in $L^2(\bH)$. It turns out that in this case it suffices to compute the hypoelliptic evolution of the mother wavelet given by Theorem~\ref{thm:left-inv-form}.

\begin{theorem}
  Let $L$ be a regular left invariant lift with mother wavelet $\Psi$. Then, for any $f\in L^2(\bH)$ we have that
  \begin{equation}
    \widehat {\Delta\,Lf} = \int_{\cS}^\oplus \left(\widehat{\Delta}^\lambda\widehat{\Psi^*}_\lambda\right) \otimes \hat f_\lambda\,d\hat\mu_\bG.
  \end{equation}
  In particular,
  \begin{equation}
    \Delta\,Lf(a) = \int_\cS \tr\left( \left(T(a)\circ\widehat{\Delta}^\lambda\widehat{\Psi^*}_\lambda\right) \otimes \hat f_\lambda \right)\, d\hat\mu_\bG, \qquad \forall a\in\bG.
  \end{equation}
\end{theorem}

\begin{proof}
  The first part of the result is an immediate consequence of Theorem~\ref{thm:decomp-delta}, Proposition~\ref{prop:ft-lift}. The second statement follows from the inversion formula for the Fourier transform (Theorem~\ref{thm:plancherel}) and the fact that $\tr(A\circ B) = \tr(B\circ A)$.
\end{proof}

\begin{corollary}
  Let $L$ be a regular left invariant lift with mother wavelet $\Psi$. Then, $t\in\bR_+\mapsto F_t\in L^2(\bG)$ is a solution of the heat equation with initial condition $F_0=Lf$ for $f\in L^2(\bH)$ if and only if
  \begin{equation}
    \hat F_t (T^\lambda) = \varphi_t\otimes \hat f_\lambda, \qquad\text{ where }\qquad \frac{d}{dt}\varphi_t = \hat\Delta^\lambda\circ \varphi_t,\quad \varphi_0 = \widehat{\Psi^*}_\lambda.
  \end{equation}
\end{corollary}

\begin{remark}
  Continuing the computations above, one gets that $F_t=Lf_t$ if and only if $\varphi_t = g_\lambda(t)\widehat{\Psi^*}_\lambda$ for all $\lambda\in \cS$. Explicitly solving the ODE for $\varphi_t$ yields that
  \begin{equation}
    g_\lambda(t) = \exp\left({t(\hat\Delta^\lambda \widehat{\Psi^*}_\lambda)_k}\right),\qquad \forall k.
  \end{equation}
  In particular, the evolution does not leave $\operatorname{rng} L$ if and only if $\hat\Delta^\lambda\widehat{\Psi^*}_\lambda = c\sum_{i=1}^N e_i$ (i.e., is the constant vector). This can be rewritten as
  \begin{equation}
    \widehat{\Psi^*}_\lambda \in \ker\left( S(k)\hat\Delta^\lambda-\idty\right) \qquad \forall k\in\bK.
  \end{equation}
\end{remark}

\subsection{Hypoelliptic diffusion of almost-periodic functions}

Let us consider, in the notations of Section~\ref{sub:finite_dimensional_subspaces_of_almost_periodic_functions}, a discrete set $F\subset \widehat{\cS}$, where $\widehat{\cS}$ is any representative of $\widehat{\bH}/\bK$. Recall that we denote with $\AP_F(\bG)$ the set of functions given as linear combinations of the coefficients of the representations $T^\lambda$ for $\lambda\in F$. By Proposition~\ref{prop:ap-simpl}, $f\in\AP_F(\bG)$ if and only if
\begin{equation}\label{eq:apfg}
	f(k,x) = \sum_{\lambda\in F}\sum_{n\in\bK} \phi(nk)\lambda(x) \hat f(k,n,\lambda), \qquad \hat f\in \bC^{\bK}\otimes\bC^\bK\otimes \bC^F.
\end{equation}
Since $x\mapsto \lambda(x)$ is smooth, we have that $\AP_F(\bG)\subset C^\infty(\bG)$ and hence it makes sense to consider the heat equation \eqref{eq:heat} on this space.

\begin{theorem}\label{thm:decomp-delta-ap}
  A map $t\in\bR_+\mapsto f_t\in\AP_F(\bG)$, where $f_t$ is of the form \eqref{eq:apfg} with coefficients $\hat f_t$, is a solution of the heat equation if and only if
  \begin{equation}
    \frac{d}{dt} \hat f_t(\cdot,\lambda) = \left(\diag_{k,n}[L_{v_i}^2(\phi(nk)\lambda)\big|_o]+ \tilde\Xi\right) \circ \hat f_t(\cdot,\lambda) \qquad\text{for all } \lambda\in F.
  \end{equation}  
  Here, letting $\Xi = \Circ(\xi)$, the operator  $\tilde\Xi \in\cL(\bC^\bK\otimes\bC^\bK)$ is given by
  \begin{equation}
    \tilde\Xi \phi(k,n) = \sum_{\ell\in\bK} \xi_\ell\, \phi(k\ell,n\ell^{-1}), \qquad\forall\phi\in\bC^\bK\otimes\bC^\bK.
  \end{equation}
\end{theorem}

\begin{proof}
  The result follows from arguments similar to those employed in Lemmas~\ref{lem:decomp-vect} and \ref{lem:decomp-xi}, although modified to account for the fact that the representations $T^\lambda$ are not, in general, square integrable.

  Let $X$ be a left-invariant vector field associated with $v\in\mathfrak{h}$. Then, simple computations yield, for all $n\in\bK$ and $\lambda\in F$,
  \begin{equation}
    L_X\left( (k,x)\mapsto \phi(nk)\lambda(x)\hat f_t(k,n,\lambda) \right) = L_v(\phi(nk)\lambda)\big|_o\, \phi(nk)\lambda(x)\hat f_t(k,n,\lambda).
  \end{equation}
  This implies that 
  \begin{equation}\label{eq:first-half}
  	\frac{d}{dt} \hat f_t=\sum_{i=1}^n L_{X_i}^2 f_t \iff \frac{d}{dt} \hat f_t(\cdot,\lambda) = \diag_{k,n}[L_{v_i}(\phi(nk)\lambda)\big|_o]\circ \hat f_t(\cdot,\lambda) \quad\text{for all }\lambda\in F.
  \end{equation}

  On the other hand, we have that
  \begin{equation}
    \begin{split}
      (\Xi\otimes\idty_\bH)f (k,x) 
      &= \sum_{\lambda\in F}\sum_{\ell,n\in\bK} \phi(n\ell)\lambda(x) \Xi_{k,\ell} \hat f(\ell,n,\lambda) \\
      &= \sum_{\lambda\in F}\sum_{h\in\bK} \phi(hk)\lambda(x) \sum_{\ell\in\bK} \Xi_{k,\ell} \hat f(\ell, hk\ell^{-1},\lambda) \\    
      &= \sum_{\lambda\in F}\sum_{h\in\bK} \phi(hk)\lambda(x) \sum_{r\in\bK} \Xi_{k,nkr^{-1}} \hat f(nkr^{-1}, r,\lambda). 
    \end{split}
  \end{equation}
  Since $\Xi$ is symmetric and circulant, we have $\Xi_{k,nkr^{-1}}=\Xi_{n,r}=\xi_{nr^{-1}}$, which implies 
  \begin{equation}
    \frac{d}{dt} \hat f_t= (\Xi\otimes\idty_\bH) f_t \iff \frac{d}{dt} \hat f_t(\cdot,\lambda) = \tilde \Xi\circ \hat f_t(\cdot,\lambda) \quad\text{for all }\lambda\in F.
  \end{equation}
  This and \eqref{eq:first-half} yield the statement.
\end{proof}

\section{Hypoelliptic diffusion on $SE(2,N)$}

We now particularize the above analysis to the case $\bG=SE(2,N)$, that is, $\bH = \bR^2$ and $\bK = \bZ_N$ for some $N\in\bN$. Since $\bZ_N$ is cyclic, we have that $S(k) = S^k$ where $S=S(e)$.
As mentioned in the Introduction, we are interested to the hypoelliptic diffusion on $SE(2,N)$ associated with the left invariant vector field
\begin{equation}
  X(k,x) = \phi(k)_*\partial_{x_1} = \sin(\theta_k)\partial_{x_1} + \cos(\theta_k)\partial_{x_2}, \qquad \theta_k = \frac{2\pi}{N}k.
\end{equation}
Namely, letting $\Theta_t$ be a jump process on $\bZ_N$ and $W_t$ a Wiener process on $\bR$, we consider the following SDE, simplifying \eqref{eq:gen_sde},
\begin{equation}
  dZ_t = 
  \left(
  \begin{array}{c}
  \cos \Theta_t \\ \sin \Theta_t    
  \end{array}
  \right)
  dW_t.
\end{equation}

In order to precise our model, we have to fix the jump process $\Theta_t$, which model the short range connectivity between neurons in the primary visual cortex. 
We assume the law of the first jump time to be exponential, with parameter $\beta>0$, and with probability $\frac12$ on both sides. 
Then, $\Theta_t$ is a Poisson process and the probability of having $k$ jumps in the interval $[0,t]$ is
\begin{equation}\label{eq:poisson}
  P(k \text{ jumps}) = \frac{(\beta t)^k}{k!} e^{-\beta t}.
\end{equation}

The infinitesimal generator of $\Theta_t$ is the matrix $\Xi = (\xi_{ij})_{ij}\in \bC^N\otimes\bC^N$, where
\begin{equation}
  \xi_{i,j} = \lim_{t\downarrow 0} \frac{P(\Theta_t=e_j \mid \Theta_t = e_i)}t \quad\text{ for }i\neq j, \quad \xi_{j,j} = -\sum_{i\neq j} \xi_{i,j}.
\end{equation}
In particular,  $\Xi_N = -\beta \idty + \frac 1 2 \beta (S + S^{-1})$. Indeed, \eqref{eq:poisson} yields
\begin{gather}
  P(\Theta_t = k\pm 1 \mid \Theta_0=k) = \frac12\left(\beta t + \bigo(t^2)\right) e^{-\beta t}, \\
  P(\Theta_t = k\pm h \mid \Theta_0=k) = \bigo(t^h) e^{-\beta t}, \qquad h=2,3,\ldots, N-2.
\end{gather}
Finally, the infinitesimal generator of the process $(Z_t,\Theta_t)$ is
\begin{equation}
  \label{eq:fokker-planck-operator}
  \Delta_N = \frac 1 2 \bigoplus_{k\in\bZ_N} \bigg( \cos(\theta_k) \partial_{x_1} + \sin(\theta_k) \partial_{x_2} \bigg)^2 + \Xi.
\end{equation}

The associated evolution, applied to $t\mapsto \psi_t\in L^2(SE(2,N))$ is
\begin{multline}\label{eq:fokker-planck-eq}
  \frac{d}{dt} \psi_t(k,x) = \frac 1 2 \bigg( \cos(\theta_k) \partial_{x_1} + \sin(\theta_k) \partial_{x_2} \bigg)^2 \psi(k,x) \\
  + \frac \beta 2 \bigg( \psi(k-1,x) -2\psi(k,x) +\psi(k+1,x) \bigg).
\end{multline}
Observe that, by construction, this equation is invariant under the left regular action of $SE(2,N)$ on $L^2(SE(2,N))$.

\begin{remark}
  As shown in \cite{Remizov2013}, setting $\beta=\left(N/2\pi\right)^2$ and letting $N\rightarrow +\infty$ in \eqref{eq:fokker-planck-operator} yields the usual Petitot--Citti--Sarti operator over $SE(2)$ of equation \eqref{eq:citti-sarti}.
\end{remark}

\begin{remark}
  In view of practical reconstruction results if we consider that for $N=30$ the limit is attained we get $\beta\simeq 25$. This $\beta$ has a clear neurophysiological interpretation in terms of the strength of neuronal connections.
\end{remark}

\begin{remark}
  In the above operator, the dependence on $k\in\bZ_N$ appears only in terms of the form $\cos^2\theta_k$, $\sin^2\theta_k$, and $\sin(2\theta_k)$, which are $N/2$ periodic. Hence, although the model for the visual cortex is on the projectivization $\bZ_{N/2}\times \bR^2$ of $SE(2,N)$, we can ignore this fact from the point of view of the hypoelliptic diffusion.
\end{remark}

Using the fact that $\lambda(x) = e^{i\langle x,\lambda\rangle}$, we have the following particularization of Theorems~\ref{thm:decomp-delta}.

\begin{proposition}\label{prop:se2n-delta}
  The Plancherel measure on the dual space of $SE(2,N)$ is supported on the slice $\cS_N = \{ \lambda= |\lambda|e^{i\omega}\mid \omega\in[0,2\pi/N)\}$, and is $|\lambda|\,d\lambda$.
  Then, letting $\lambda = (\lambda_1,\lambda_2)$, it holds that
  \begin{equation}
    \widehat\Delta_N  = \int_{\cS_N}^\bigoplus \hat\Delta_N^{\lambda}\,|\lambda| d\lambda, 
    \qquad
    \hat\Delta_N^{\lambda} = - \bigoplus_{k\in\bZ_N}(\lambda_1\cos\theta_k+\lambda_2\sin\theta_k)^2 + \Xi\otimes \idty_{\bR^2}.
  \end{equation}
\end{proposition}

Let $F\subset \cS_N$. Observe that $f\in \AP_F(SE(2,N))$ if and only if
\begin{equation}\label{eq:se2n-ap}
  f(k,x) = \sum_{\lambda\in F}\sum_{n\in\bZ_N} e^{i\langle R_{\theta_{n+k}}x,\lambda\rangle} \hat f(k,n,\lambda), \qquad \hat f\in\bC^N\otimes\bC^N\otimes \bC^F.
\end{equation}
By Theorem~\ref{thm:decomp-delta-ap}, this immediately yields the following.

\begin{proposition}\label{prop:ap-discretization}
  A map $t\in\bR_+\mapsto f_t\in \AP_F(SE(2,N))$ is a solution of the heat equation if and only if, for all $\lambda\in F$ and $k,n\in\bZ_N$, we have
  \begin{multline}
    \frac{d}{dt}\hat f_t(k,n,\lambda) = -(\lambda_1\cos\theta_{k+n}+\lambda_2\sin\theta_{n+k})^2\hat f_t(k,n,\lambda) + \\
    \frac\beta2\left( \hat f_t(k+1,n-1,\lambda) -2\hat f_t(k,n,\lambda)+\hat f_t(k-1,n+1,\lambda)\right).
  \end{multline}
\end{proposition}

\begin{remark}
  The above result coincides with \cite[Eq. (3.2)]{Remizov2013}. Indeed, the coefficients $a_{r,\lambda}$, $r\in\bZ_N$ and $\lambda\in \bigcup_{n\in\bZ_N}\phi(n)\lambda$, used in that paper to represent functions in $\AP_F(SE(2,N))$ are given by $a_{r,\phi(nr)\lambda}=\hat f(r,n,\lambda)$, for $r,n\in\bZ_N$ and $\lambda\in F$.
\end{remark}

\subsection{Image reconstruction algorithm}\label{sec:image-rec-algo}

As mentioned in Section~\ref{sec:square-integrable-functions-on-the-plane}, we represent images as functions $f:\bR^2\to[0,1]$, with support in the disk $D_R\subset\bR^2$. Up to replacing $f$ by $f+\varepsilon$ for some small $\varepsilon>0$, we can always assume $f>0$. Thus, an image corrupted on a set $\Omega\subset D_R$ can be represented as $f:\bR^2\to[0,1]$ such that $f^{-1}(0)=\Omega$.

\begin{figure}
  \centering
  \includegraphics[width=.8\textwidth]{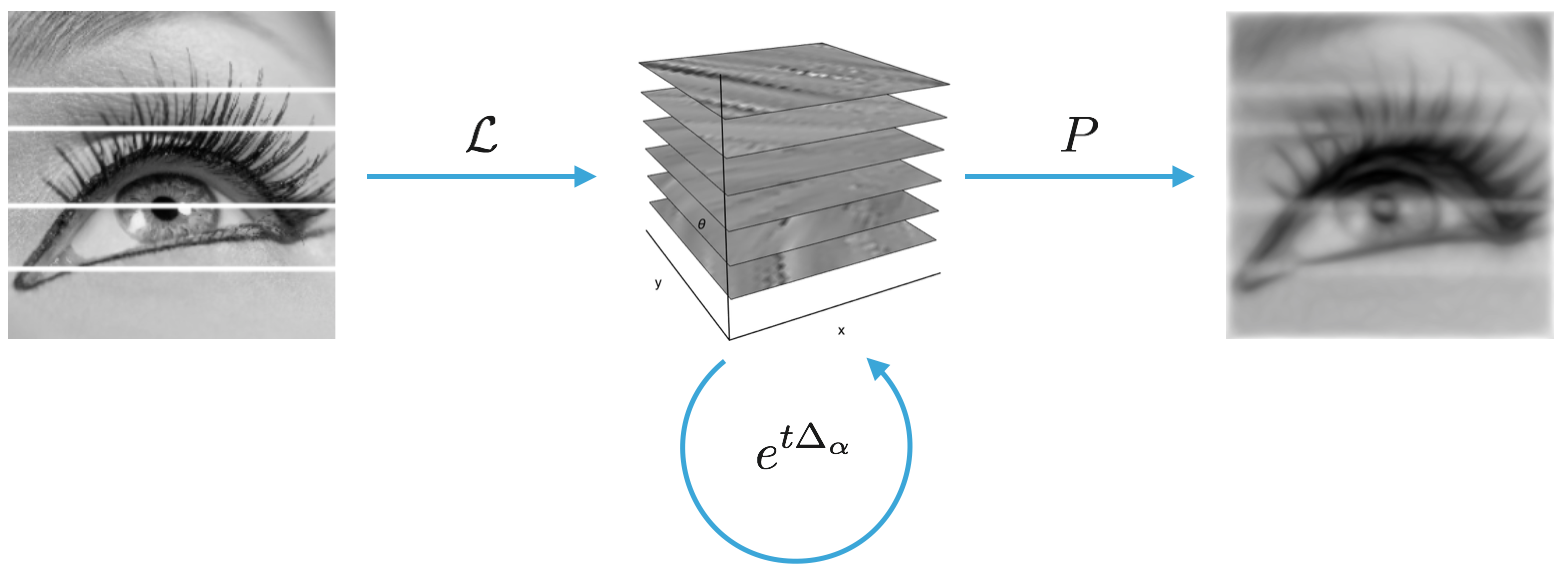}
  \caption{The image reconstruction pipeline.}
  \label{fig:pipeline}
\end{figure}

To reconstruct a corrupted image we fix a lift $L:L^2(\bR^2)\mapsto L^2(SE(2,N))$ and use the following algorithm (see Figure~\ref{fig:pipeline}):
\begin{enumerate}
  \item[1.] Lift the image to $Lf\in L^2(SE(2,N))$;
  \item[2.] Evolve the image through $\frac{d}{dt} Lf = \Delta_N Lf$ for a certain interval of time $[0,T]$, in order to obtain $\widetilde{Lf}\in L^2(SE(2,N))$;
  \item[3.] Project $\widetilde{Lf}$ to $\tilde f\in L^2(\bR^2)$, the reconstructed image.
\end{enumerate}

In the next chapter we will describe in detail how to exploit the group structure of $SE(2,N)$ to easily integrate the evolution equation. Let us observe that, in general, the resulting image $\widetilde{Lf}$ will not be in $range L$, and hence it is not possible to exploit the injectivity of the lift to project it on $\bR^2$. A reasonable and neurophysiologically sound choice for the projection operator is to simply sum all levels. Namely, we consider
\begin{equation}
  P: \psi\in L^2(SE(2,N)) \to \sum_{k\in\bZ_N}\psi(k,\cdot)\in  L^2(\bR^2).
\end{equation}

  In practical applications, we will not consider a left-invariant lift. Indeed, the heat evolution on $SE(2,N)$ commutes with the left-regular representation and the projection operator intertwines the latter with the quasi-regular representation. Hence, as pointed out in \cite{DF}, if the lift was left-invariant the above algorithm would commute with the quasi-regular regular representation, i.e., it would be invariant w.r.t.\ translations and rotations in $\bR^2$. In particular, it would be an isotropic evolution, thus rendering completely pointless the construction.

%% file: applications.tex
This chapter collects the results of numerical testing in image processing applications of the various concepts explained throughout this work. These are mostly taken from the already mentioned papers \cite{ap_interp, Remizov2013, cdc-dario, highly, G3}.
 
\section{AP Interpolation and approximation}\label{sec:ap_interp_numerical}

In this section we present numerical results regarding the AP interpolation and approximation procedure introduced in Chapter~\ref{ch:ap_interp}.
The main \verb+julia+ program and the tests are contained in the package \verb+ApApproximation.jl+, which is available at \url{http://github.com/dprn/ApApproximation.jl}, and in particular in \href{http://nbviewer.jupyter.org/github/dprn/ApApproximation.jl/blob/master/notebooks/AP\%20Interpolation\%20and\%20approximation\%20tests.ipynb}{this Jupyter notebook}. 

For $\xi\in \mathbb R^2$ let $\tau_\xi$ be the translation $\tau_\xi f(x) = f(x-\xi)$. Then, in the notations of Section~\ref{sec:image}, if $f$ is of the form \eqref{eq:interpol}, the same is true for $\tau_\xi f$, with
\begin{equation}
	\widehat{\tau_\xi f}_{k,m} = e^{-i\big\langle R_{\frac{2\pi m}N}\Lambda_k, \xi \big\rangle}\hat f_{k,m}.
\end{equation}
Let $\hat\tau_\xi$ be defined by $\hat\tau_\xi(\hat f):=\widehat{\tau_\xi f}$. In our tests we exploited this operator to check the results of the AP interpolation and approximation. Indeed, in general, applying a translation will completely change the points on which $f$ is sampled by $\ev$ and hence highlights the presence of high variability in between the points of interpolation/approximation.

For the tests, we fixed $Q=340$, $N=64$, and defined a specific set $E=F$. We then computed the AP interpolation, resp. approximation, with respect to these sets. As weights for the AP approximation we chose
\begin{equation}\label{eq:weights}
    d(\Lambda) = 
    \begin{cases}
    \frac \alpha{10} &\qquad \text{if } |\Lambda|\le 1\\
    \alpha &\qquad \text{if } 1<|\Lambda|\le\frac32\\
    100\alpha &\qquad \text{if } |\Lambda|>\frac32.\\
    \end{cases}
\end{equation}
In Figure~\ref{fig:ap-approx} we present, from left to right, a plot of the
power spectrum of $\hat f$ and the results of the evaluation of $\ev\hat f$,
of $\ev(R_\gamma\hat f)$ for the angle $\gamma = 10\times 2\pi/N$,  and of
$\ev(\hat\tau_\xi\hat f)$ for $\xi\sim (15,26)$. In Table~\ref{tab:norms}, we present the corresponding $L^2$ norms. In particular, we observe that the $L^2$ norm of $\ev(\hat\tau_\xi\hat f)$ is stable for the AP approximation, contrarily to what happens for the AP interpolation.
Obviously, we see also that the effect of discrete rotations is perfect for both interpolation and approximation.

\begin{figure}
	\begin{minipage}[b]{.24\linewidth}
		\centering\includegraphics[width=.7\textwidth]{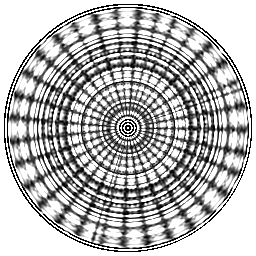}
	\end{minipage}
	\begin{minipage}[b]{.24\linewidth}
		\centering\includegraphics[width=.7\textwidth]{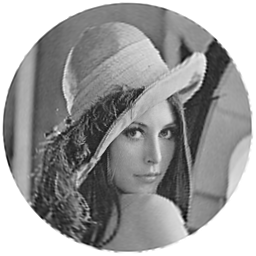}
	\end{minipage}%
  \begin{minipage}[b]{.24\linewidth}
		\centering\includegraphics[width=.7\textwidth]{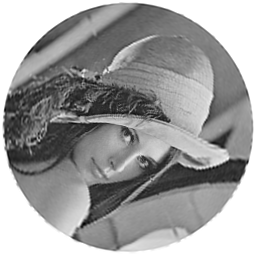}
	\end{minipage}
  \begin{minipage}[b]{.24\linewidth}
		\centering\includegraphics[width=.7\textwidth]{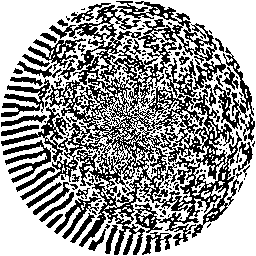}
	\end{minipage}
	\begin{minipage}[b]{.24\linewidth}
		\centering\includegraphics[width=.7\textwidth]{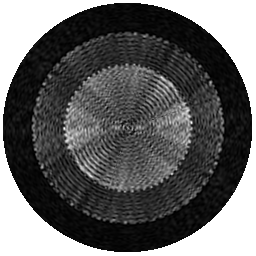}
		\subcaption{Magnitudes: $|\hat f|^2$.\\ $ $\\ $ $}
		\label{fig:1b}
	\end{minipage}
	\begin{minipage}[b]{.24\linewidth}
		\centering\includegraphics[width=.7\textwidth]{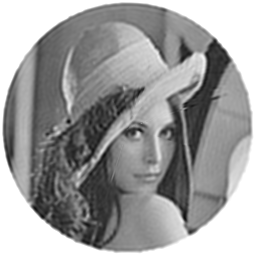}
		\subcaption{Evaluation: $\ev \,\hat f$.\\ $ $\\ $ $}
		\label{fig:1a}
	\end{minipage}%
	\begin{minipage}[b]{.24\linewidth}
		\centering\includegraphics[width=.7\textwidth]{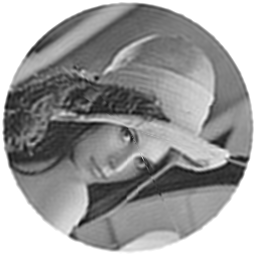}
		\subcaption{Rotation of $\gamma = 20\pi/N$ 
      and evaluation: $\ev \big(R_\gamma\hat f\big)$.}
		\label{fig:1b}
	\end{minipage}
  \begin{minipage}[b]{.24\linewidth}
		\centering\includegraphics[width=.7\textwidth]{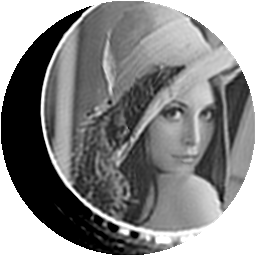}
		\subcaption{Translation of $\xi\sim(15,26)$ and
      evaluation: $\ev\big(\hat\tau_{\xi}\hat f\big)$.}
		\label{fig:1b}
	\end{minipage}
	\caption{Tests on a $246\times256$ image with $E=F$, $Q = 340$, $N=64$. \emph{First row:} AP interpolation. \emph{Second row:} AP approximation with weights \eqref{eq:weights} and $\alpha = 100$.}
	\label{fig:ap-approx}
\end{figure}

\begin{table}[pb]
	\centering
	\begin{tabular}{l || c | c | c | c |}
		 & $\hat f$ & $\ev\big(\hat f\big)$ & $\ev \big(R_\gamma\hat f\big)$& $\ev\big(\hat\tau_{\xi}\hat f\big)$ \\
		 \hline
		AP interpolation & $6.7\times 10^{5}$ & $80.3$ & $80.3$ & $3.0 \times 10^{10}$\\
		AP approximation & $0.2$ & $80.0$ & $80.0$ &  $79.0$ \\
	\end{tabular}	
	\caption{$L^2$ norms of the AP interpolation and approximation of
    Figure~\ref{fig:ap-approx} of an image with $L^2$ norm $80.1$ on the set
    $E$. In the first column we have the $L^2$ norms of the vector of
    frequencies $\{\hat f(\Lambda)\}_{\Lambda\in F}$, while in the second, third
    and last ones we have the $L^2$ norms of the vector obtained by applying the
    evaluation operator to $\hat f$, to its rotation by $\gamma =
    10\times2\pi/N$, and to its translation by $\xi\tilde (15,26)$, respectively.}
	\label{tab:norms}
\end{table}

\section{Image reconstruction}\label{sec:image-reconstruction}

In Figure~\ref{fig:inpainting}, we present the numerical implementation of the image reconstruction algorithm presented in Section~\ref{sec:image-rec-algo}. Depending on the chosen point of view, this can be done in two different ways: either we spatially discretize the hypoelliptic operator on $L^2(SE(2,N))$ or we interpolate the lifted function with AP functions and we exploit the (frequency) decomposition given in Proposition~\ref{prop:ap-discretization}. While the AP interpolation technique has been already discussed, the spatial discretization of the hypoelliptic operator is described below.

Let us remark that, in order to simplify the implementation, in these examples we chose to use the distributional lift introduced in \cite{G1}. Let, $\theta_k=\frac{2\pi k}N$, then,
\begin{equation}
  Lf(k,x) = 
  \begin{cases}
    f(x) & \qquad \text{if } \theta_k \simeq \theta(x),\\
    0 & \qquad \text{if } \theta_k \not\simeq \theta(x).\\
  \end{cases}
\end{equation}
Here, the $\theta_k \simeq \theta(x)$ means that $\theta_k$ is the nearest point to $\theta(x)$ among all points $\{\theta_1,\ldots,\theta_{N}\}$, and $\theta(x)$ is the slope angle of the level curve $f^{-1}(f(x))$ passing through $x$. Namely, $\theta(x)$ is defined as
\begin{equation}
  \tan\theta(x) = -\frac{\partial_{x_1}f(x)}{\partial_{x_2}f(x)}.
\end{equation}
If $\partial_{x_1}f(x)=\partial_{x_2}f(x)$, which corresponds to a critical point, we let 
\begin{equation}
  Lf(k,x) = \frac{f(x)}{N}\qquad \forall k\in\bZ_N.
\end{equation}

\begin{figure}
  \centering
	\begin{minipage}[b]{.3\linewidth}
		\centering\includegraphics[width=.7\textwidth]{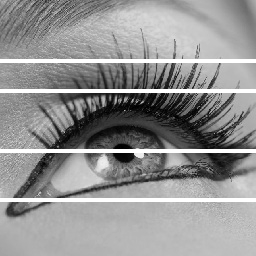}
		\subcaption{Original corrupted image.\\$ $}
	\end{minipage}
	\begin{minipage}[b]{.3\linewidth}
		\centering\includegraphics[width=.7\textwidth]{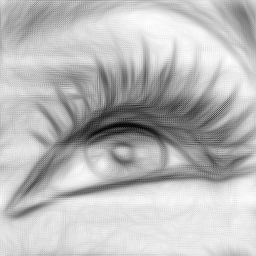}
		\subcaption{Reconstruction via spatial discretization.}
		\label{fig:1b}
	\end{minipage}
  \begin{minipage}[b]{.3\linewidth}
		\centering\includegraphics[width=.7\textwidth]{imgs/petitot-ds}
		\subcaption{Reconstruction via AP interpolation.}
		\label{fig:1b}
	\end{minipage}
	\caption{Image reconstruction via hypoelliptic diffusion.}
	\label{fig:inpainting}
\end{figure}

In any case, the resulting algorithm does not use any information on the corrupted area. Later, we will present an heuristic technique, introduced in \cite{Remizov2013}, that implements this information in the algorithm, allowing for remarkable reconstructions.

\subsection{Spatial discretization of the hypoelliptic operator}

Let the input image $f$ being given as an $M\times M$ table of real values between $[0,1]$.
We consider $G\subset \bR^2$ to be the $M\times M$ grid on the plane with discretization step $\Delta x = \Delta y = \sqrt{M}$, i.e., such that the mesh points are $x_k ={ (k-1)/\sqrt{M}}$ and $y_l = {(l-1)/\sqrt{M}}$ for $k,l = 0,\ldots, M-1$.
In the following, for any function $\psi$ defined on $SE(2,N)$, we will denote $\psi_{k,l}^r = \psi(r,x_k,y_l)$.

We are interested in the evolution equation \eqref{eq:fokker-planck-eq}, that is:
\begin{gather}\label{eq:fokker-to-disc}
  \frac{d}{dt} \psi_t(k,x) = A_k \psi(k,x) + \frac \beta 2 \bigg( \psi(k-1,x) -2\psi(k,x) +\psi(k+1,x) \bigg), \\
  A_k = \frac 1 2 \bigg( \cos(\theta_k) \partial_{x_1} + \sin(\theta_k) \partial_{x_2} \bigg)^2 \qquad \text{ and }\qquad A = \bigoplus_{k\in\bZ_N}A_k.
\end{gather}
To discretize this equation, we replace the differential operators $\partial_{x_1}$ and $\partial_{x_2}$  by their finite element approximations
\begin{gather*}
  D_1 \psi_{k,l}^r = \frac {\psi_{k+1,l}^r - \psi_{k-1,l}^r} {x_{k+1} -     x_{k-1}} = \frac {\sqrt{M}} 2 (\psi_{k+1,l}^r - \psi_{k-1,l}^r),\\
  D_2 \psi_{k,l}^r = \frac {\psi_{k,l+1}^r - \psi_{k,l-1}^r} {y_{l+1} -    y_{l-1}} = \frac {\sqrt{M}} 2 (\psi_{k,l+1}^r - \psi_{k,l-1}^r).
\end{gather*}
Then, the discretized version of $A$ is $D = \diag (\cos \theta_r D_x + \sin \theta_r D_y).$ Replacing $A$ with $D$ in \eqref{eq:fokker-to-disc}, we obtain its discretized versions.
The initial condition for these equations will be the discrete analogue of the function $Lf$ on $SE(2,N)$ obtained by lifting the original image.

Let us denote by $\widehat \psi_{k,l}^r$ the discrete Fourier transform (DFT) of $\psi$ w.r.t.\ the variables $k,l$.
Then, a straightforward computation shows that $A \widehat\psi_{k,l}^r = i\sqrt M a_{k,l}^r \psi_{k,l}^r$, where
\begin{equation*}
  a_{k,l}^r = \cos\theta_r \sin\left( 2\pi \frac {k-1} M \right) + \cos\theta_r \sin\left( 2\pi \frac {l-1} M  \right).
\end{equation*}
This is essentially a discretized version of Proposition~\ref{prop:se2n-delta}.

Hence, the diffusion equation \eqref{eq:fokker-to-disc} is mapped by the DFT in the completely decoupled system of $M^2$ ordinary linear differential equations on $\mathbb C^N$:
\begin{gather}
  \frac {d \widehat \psi_{k,l}}{dt} = \bigg( \Xi_N - \frac M 2 \diag_r(a_{k,l}^r)^2 \bigg) \widehat \psi_{k,l},
\end{gather}
where $\Xi_N= -\beta \idty + \frac 1 2 \beta (S + S^{-1})$, $k,l = 0,\ldots,M-1$ and $\widehat \psi_{k,l} = (\widehat \psi_{k,l}^0,\ldots, \widehat \psi_{k,l}^{N-1})^*$.

These discretized equations can then be solved through any numerical scheme.
We chose the Crank-Nicolson method, for its good convergence and stability properties.
Let us remark that the operators appearing on the r.h.s.\ are periodic tridiagonal matrices, i.e. tridiagonal matrices with non-zero $(1,N)$ and $(N,1)$ elements.
Thus, the linear system appearing at each step of the Crank-Nicolson method can be solved through the Thomas algorithm for periodic tridiagonal matrices, of computational cost $\mathcal{O}(N)$.

\subsection{Heuristic complements: masking and AHE algorithm}

In this section we present a technique to implement information on the location of the corruption in the inpainting algorithm. Assume that a partition of the grid $G = G_g \cup G_b$ is given, where points in $G_g$ are ``good'', i.e., non-corrupted, while those in $G_b$ are ``bad'', i.e., corrupted. The idea is now to periodically ``mix'' the solution $\psi_t$ of the diffusion on $SE(2,N)$ with the initial function $Lf$ on $G_g$, while keeping tabs on the ``evolution'' of the set of good points. 

Namely, fix $n\in\bN$ and split the segment $[0,T]$ into $n$ intervals $t_r=r\tau$, $r=0,\ldots,n$, $\tau=T/n$. Let $G_g(0)=G_g$, $G_b(0)=G_b$ and iteratively solve  the hypoelliptic diffusion equation on each $[t_r,t_{r+1}]$ with initial condition 
\begin{equation}
  \psi_{t_r}(k,x)
  \begin{cases}
    \psi^-_{t_r}(k,x)& \quad\text{if } x\in G_b(r)\\
    \sigma(x,t_k)\psi^-_{t_r}(k,x)& \quad\text{if } x\in G_g(r).
  \end{cases}
\end{equation}
Here, the function $\psi^-$ is the solution of the diffusion on the previous interval (or the starting lifted function if $r=0$, and the coefficient $\sigma$ is given by 
\begin{equation}
  \sigma(x,t_r)=\frac12 \frac{h(x,0)+h(x,t_r)}{h(x,t_r)}, \qquad h(x,t) = \max_k \psi_t(k,x).
\end{equation}
Moreover, after each step, $G_g(r+1)$ and $G_b(r+1)$ are obtained from $G_g(r)$ and $G_b(r)$ as follows:
\begin{enumerate}
  \item[1.] Project the solution $\psi_{t_{r+1}}$ to the image $f_{r+1}$.
  \item[2.] Define $\text{A}f_{r+1}(x)$ as the average of $f_{r+1}$ on the $9$-point neighborhood of $x$ in $G$.
  \item[3.] Define the set $W = \{x\in\partial G_b(r) \mid  f_{r+1}(x)\ge \text{A}f_{r+1}(x) \}$.
  \item[4.] Let $G_g(r+1) = G_g(r)\cup W$, $G_b(r+1)=G_b(r)\setminus W$.
\end{enumerate}

Some reconstruction results via the masking procedure are presented in Figure~\ref{fig:masking}. In order to conclude this section, we remark that in \cite{highly} a more refined heuristic procedure has been introduced: the Averaging and Hypoelliptic Evolution (AHE) algorithm. Some reconstruction results can be seen in Figure~\ref{fig:ahe}.

  However, to be perfectly honest, as we said at the end of the Introduction (Section~1.5), these results are not significantly better than the state of the art.
  We repeat here that, for image processing purposes, the interest of Citti-Petitot-Sarti model and our model lies more in the field of pattern recognition.

\begin{figure}
  \centering
  \includegraphics[width=.3\textwidth]{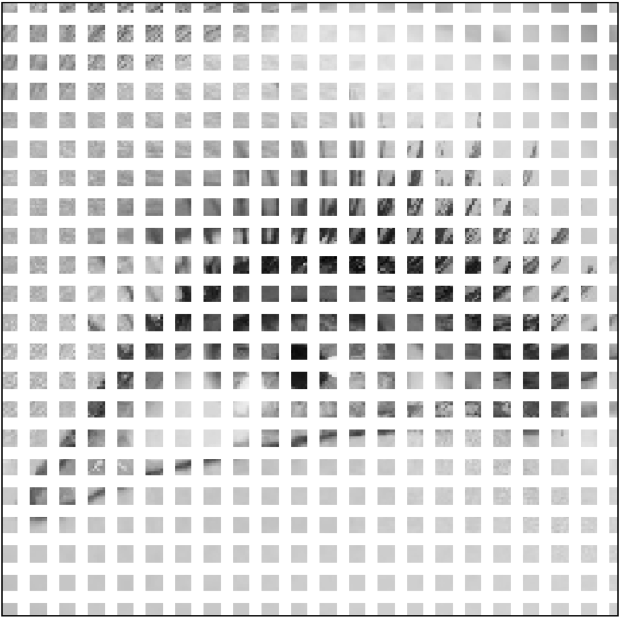}
  \hspace{.05\textwidth}
  \includegraphics[width=.3\textwidth]{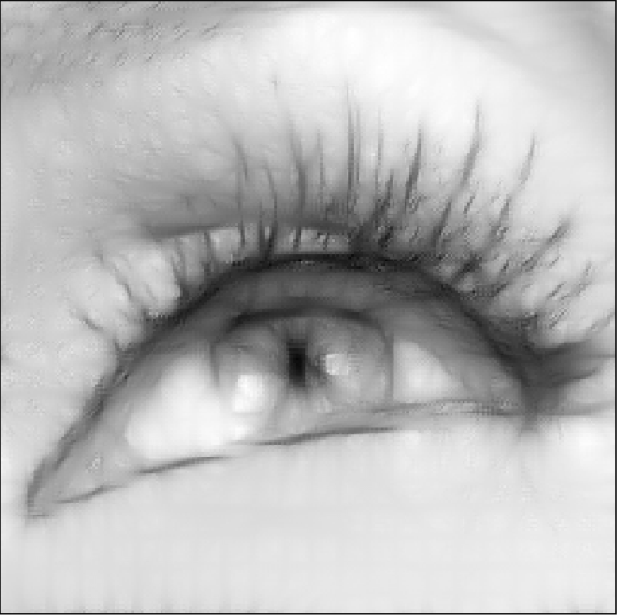}
  \\
  \vspace{1em}
  \includegraphics[width=.3\textwidth]{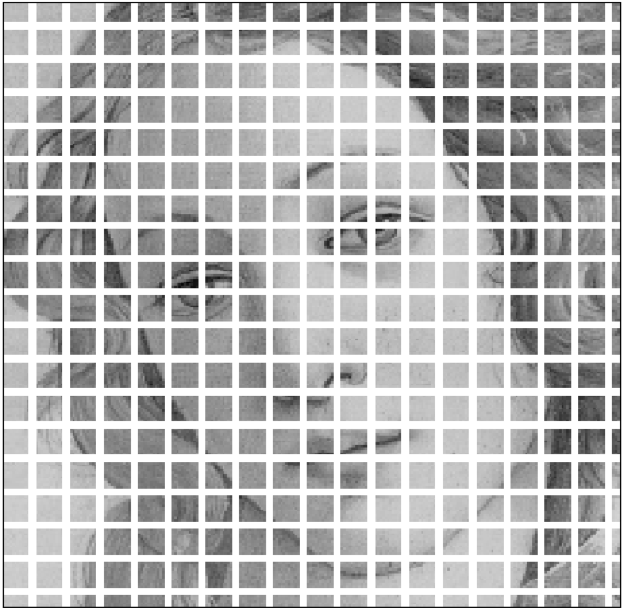}
  \hspace{.05\textwidth}
  \includegraphics[width=.3\textwidth]{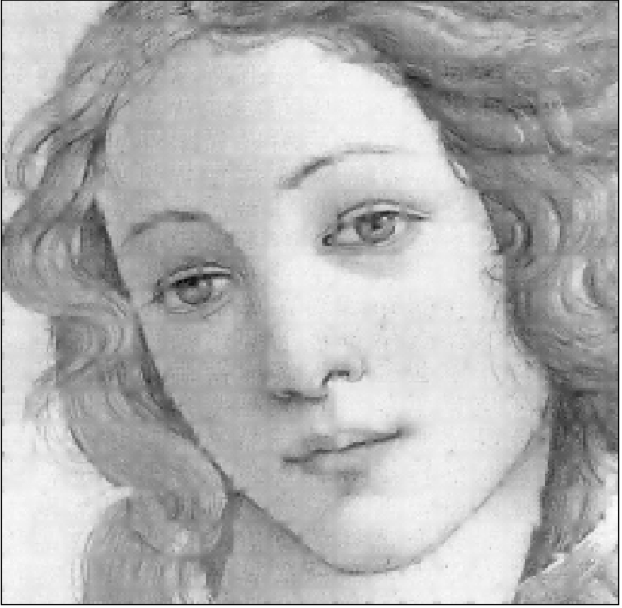}
  \caption{Image reconstruction with masking procedure. \emph{Left:} Original images. \emph{Right:} Reconstructions.}
  \label{fig:masking}
\end{figure}

\begin{figure}
  \centering
  \includegraphics[width=.7\textwidth]{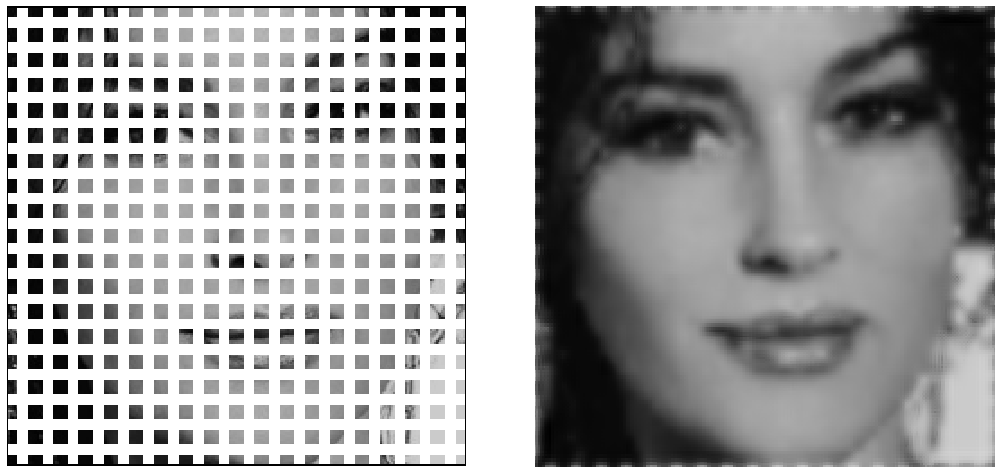}
  \includegraphics[width=.7\textwidth]{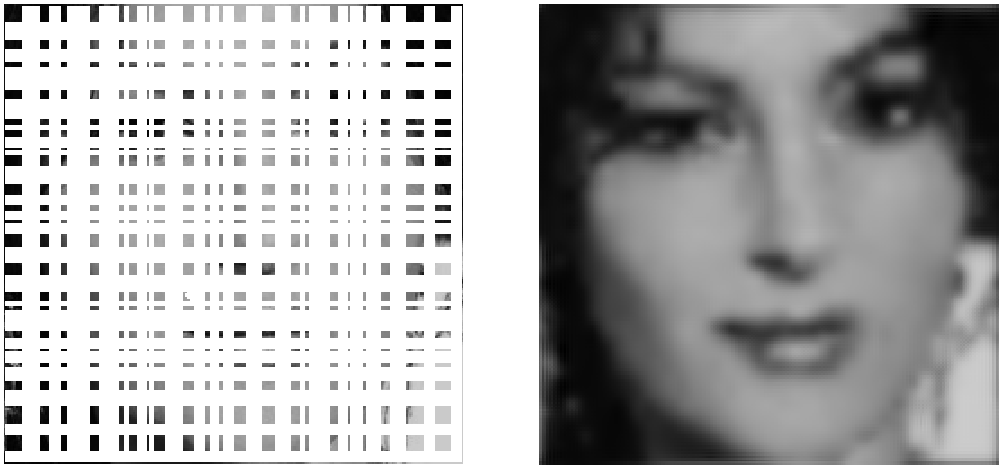}
  \includegraphics[width=.7\textwidth]{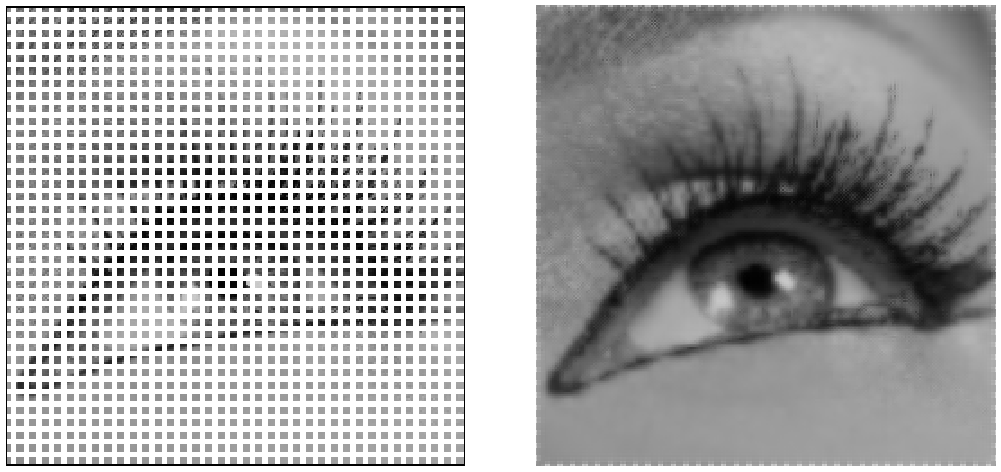}
  \includegraphics[width=.7\textwidth]{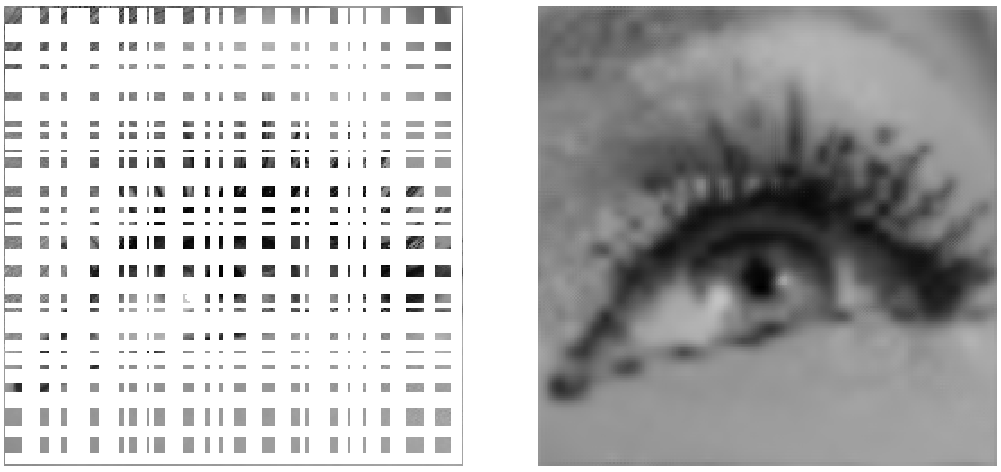}
  \caption{Image reconstruction with AHE algorithm.}
  \label{fig:ahe}
\end{figure}

\section{Object recognition} 
\label{sec:experimental_results}

The goal of this section is to evaluate the performance of the invariant Fourier descriptors defined in Chapter~\ref{bispectrum} on a large image database for object recognition. 
In addition to the generalized power-spectrum (PS) and bispectrum (BS) and the rotational power-spectrum (RPS) and bispectrum (RBS), we also consider the combination of the RPS and BS descriptors.
Indeed, combining these two descriptors seems to be a good compromise between the theoretical result of completeness given by Theorem~\ref{thm:rot-bisp-completeness}, which only holds for the RBS, and computational demands, as the results on the COIL-100 database will show.

After showing how to efficiently compute these descriptors and presenting the image data set, we analyze some experimental results. 
In order to estimate the features capabilities, we use a support vector machine (SVM)  \cite{VAP}	 as supervised classification method. 
The recognition performances of the different descriptors regarding invariance to rotation, discrimination capability and robustness against noise are compared.

\subsection{Implementation}

As proved in Corollaries~\ref{cor:reduced-inv} and \ref{cor:reduced-rot-inv}, the equality of the Fourier descriptors we introduced does not depend on the choice of the mother wavelet $\Psi$. Accordingly, in our implementation we only computed the quantities introduced in Corollary~\ref{cor:reduced-rot-inv}, whose complexity is reduced to the efficient computation of the vector ${\hat f_\lambda}$, for a given $\lambda\in\cS$. 
We recall that this vector is obtained by evaluating the Fourier transform of $f$ on the orbit of $\lambda$ under the action of discrete rotations $R_{-k}$ for $k \in \bZ_N$.

Let us remark that, although in our implementation we chose this approach, in principle fixing a specific mother wavelet could be useful to appropriately weight descriptors depending on the associated frequencies.
Indeed, preliminary tests with a Gabor mother wavelet showed slightly better results at a bigger computational cost.

For the implementation\footnote{MATLAB sample code for the implementation of the rotational bispectral invariants can be found at \url{https://nbviewer.jupyter.org/github/dprn/bispectral-invariant-svm/blob/master/Invariant_computation_matlab.ipynb}} we chose to consider $N = 6$ and to work with images composed of hexagonal pixels. 
There are two reasons for this choice:
\begin{itemize}
	\item It is well-known that retinal cells are distributed in a hexagonal grid, and thus it is reasonable to assume that cortical activations reflect this fact.
	\item Hexagonal grids are invariant under the action of $\bZ_6$ and discretized translations, which is the most we can get in the line of the invariance w.r.t.\ $SE(2,6)$. Indeed, apart from the hexagonal lattice, the only other lattices on $\bR^2$ which are invariant by some $\bZ_N$ and appropriate discrete translations are obtained with $N=2,3,4$.
\end{itemize}
The different steps of computation of the descriptors are as follows:
\begin{enumerate}
	\item The input image is converted to gray-scale mode, the Fourier transform is computed via FFT, and the zero-frequency component is shifted to the center of the spectrum.
	\item For cost computational reasons and since we are dealing with natural images, for which the relevant frequencies are the low ones, we extract a grid of $16\times16$ pixels around the origin.
	\item The invariants are computed from the shifted Fourier transform values, on all frequencies in an hexagonal grid inside this $16\times 16$ pixels square.
	A bilinear interpolation is applied to obtain the correct values of ${\hat f_\lambda}$.
	The final dimension of the feature-vector is given in Table~\ref{tab:feat-vec}. 
\end{enumerate}

\begin{table}
	\centering
	\caption{Dimension of the feature vectors for the Fourier descriptors under consideration}
	\label{tab:feat-vec}
	\begin{tabular}{|c|c|}
	\hline
	Descr. &  Dim. \\
	\hline
	\hline 
	PS  	&   136  	\\ \hline
	BS  	& 	717 	\\ \hline
	RPS  	& 	816 	\\ \hline
	RBS  	& 	4417 	\\ \hline
	RPS + BS  	& 	1533 	\\ \hline
	\end{tabular}
\end{table}

\subsection{Test protocol}

We use the Fourier descriptors to feed an SVM classifier, via the MATLAB Statistics and Machine Learning Toolbox, applying it on a database of 7200 objects extracted from the Columbia Object Image Library (COIL-100)
and a database of 400 faces extracted from ORL
face database.
Finally, we compare the results obtained with those obtained using traditional descriptors.

The result of the training step consists of the set of support vectors determined by the SVM based method.
During the decision step, the classifier computes the Fourier descriptors and the model determined during the training step is used to perform the SVM decision. 
The output is the image class.

For COIL-100 database, two cases are studied: a case without noise and another with noise. In the first one, tests have been performed using 75\% of the COIL-100 database images for training and 25\% for testing. In the second one, we have used a learning data-set composed of all the 7200 images (100 objects with 72 views) without noise and a testing data-set composed of 15 randomly selected views per object to which an additive Gaussian noise with $S_d$ of 5, 10 and 20 was added. (See Fig. \ref{fig:noise}). 

We evaluate separately the recognition rate obtained using the four previous invariant descriptors and the combination of the RPS \& BS invariants to test their complementarity.
Then, we compare their performance with the Hu's moments (HM), the Zernike's moments (ZM), the Fourier-Mellin transform (FM) (see the Appendix in \cite{G3}), and the local SIFT and HOG descriptors \cite{dalal2005histograms} whose performance under the same conditions has been tested in \cite{Choksuriwong2008}, 

Since we use the RBF kernel in the SVM classification process, this depends on the kernel size $\sigma$. 
The results presented here are obtained by choosing empirically the value $\sigma _{opt}$ that provided maximum recognition rate.

\subsection{Experiments}
The performances of the different invariant descriptors are analyzed with respect to the recognition rate given a learning set. Hence, for a given ratio, the learning and testing sets have been built by splitting randomly all examples. Then, due to randomness of this procedure, multiple trials have been performed with different random draws of the learning and testing set. 
In the case of an added noise, since as mentioned before the learning set is comprised of all images, this procedure is applied only to the testing set.

The parameters of our experiments are the following:

\begin{enumerate}
\item The learning set $c_i$ corresponding to the values of an invariant descriptor computed on an image from the database;
\item The classes $\hat c_i  \in \left\{ {1,100} \right\}$ corresponding to the object class.
\item Algorithm performance: the efficiency is given \\through a percentage of the well recognized objects composing the testing set.
\item Number of random trials: fixed to 5.
\item Kernel K: a Gaussian kernel of bandwidth $\sigma$ is chosen
\begin{equation}
K(x,y) = e^{\frac{{ - \left\| {x - y} \right\|^2 }}{{2\sigma ^2 }}} \\ 
\end{equation}
$x$ and $y$ correspond to the descriptors vectors of objects.
\end{enumerate}

For solving a multi-class problem, the two most popular approaches are the one-against-all (OAA) method and the one-against-one (OAO) method \cite{milgram:inria-00103955}. 
For our purpose, we chose an OAO SVM because it is substantially faster to train and seems preferable for problems with a very large number of classes. 

\subsubsection{COIL-100 databases}
The Columbia Object Image Library (COIL-100, Fig. \ref{fig:coil}) is a database of color images of 100 different objects, where 72 images of each object were taken at pose intervals of $5^\circ$.

\subsubsection*{\textbf{Classification performance}}
         
\begin{figure}
\centering
  \includegraphics[width=0.48\textwidth]{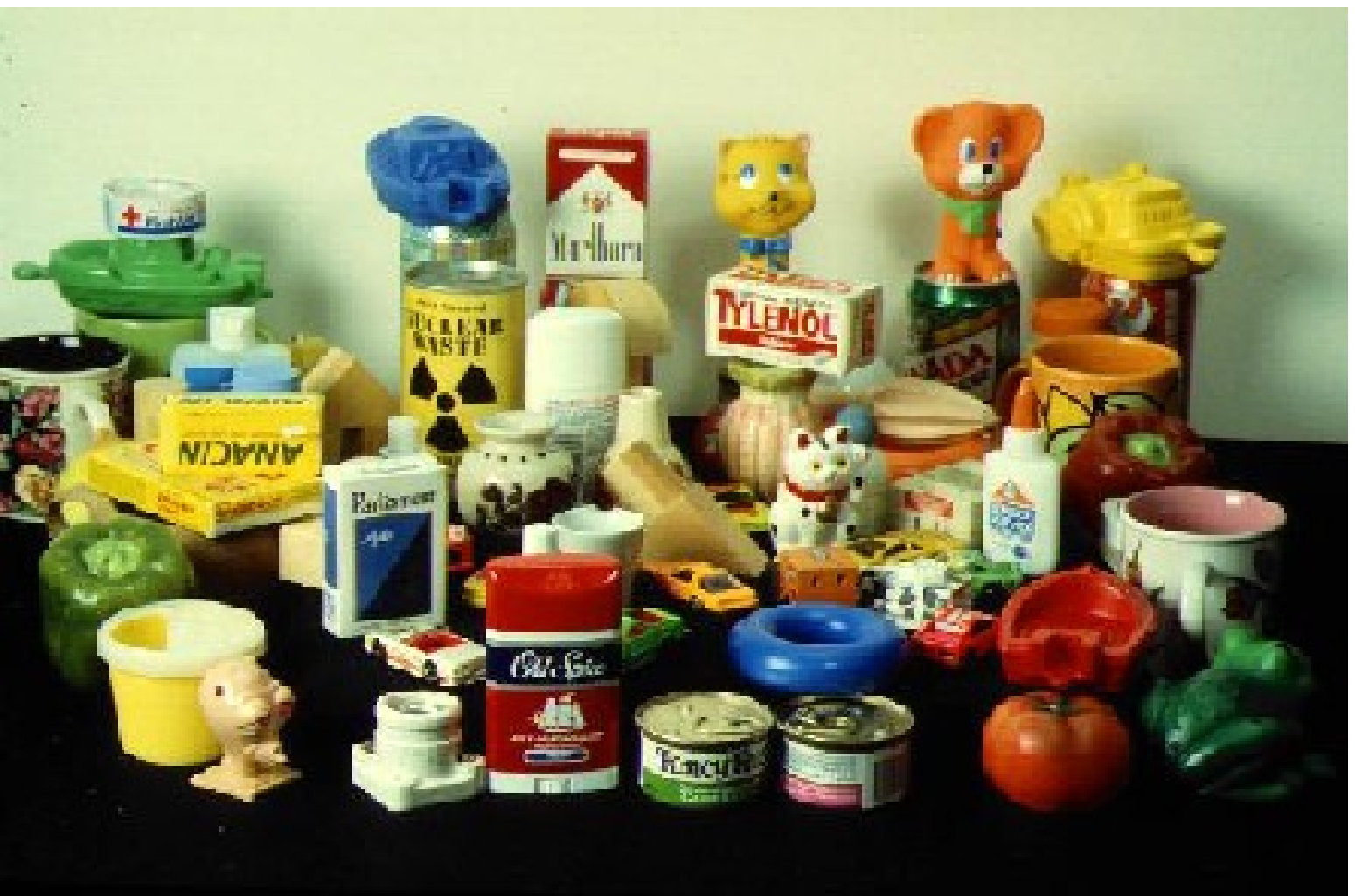}
\caption{Sample objects of COIL-100 database}
\label{fig:coil}       
\end{figure}

Table~\ref{tab:coilTab} presents results obtained testing our object recognition method with the COIL-100 database.
The best results were achieved using the local SIFT descriptor. The RBS comes in the second place and the local HOG features come third. Indeed it has been demonstrated in the literature, these local methods currently give the best results. However, if noise is added on the image, the use of global approach is better than the use of local ones. The main reason is that the key-points detector used in the local method produce in these cases many key-points that are nor relevant for object recognition. This will be shown in the next subsection.

\begin{table}[ht]
\centering
\caption{Recognition rate for each descriptor using the COIL-100 database. The test results for ZM, HM, FM, and SIFT are taken from \cite{Choksuriwong2008}.}
\label{tab:coilTab}       
\begin{tabular}{|c|c|}
\hline
Descriptors & Recognition rates \\
\hline
\hline
RBS & \textbf{95.5\%}\\
BS & 88\%\\
PS & 84.3\%\\
RPS &89.8\%\\
RPS+BS & \textbf{92.8\%}\\
ZM & 91.9\%\\
HM & 80.2\%\\
FM & 89.6\% \\
HOG & \textbf{95.3\%} \\
SIFT & \textbf{100\%} \\\hline
\end{tabular}
\end{table}

\begin{figure}
\centering
  \includegraphics[width=0.5\textwidth]{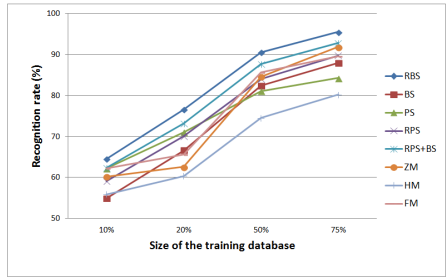}
\caption{Classification rate for different size of the training database. The test results for ZM, HM, FM, and SIFT are taken from \cite{Choksuriwong2008}.}
\label{fig:graph}       
\end{figure}

In Figure \ref{fig:graph} we present the recognition rate as a function of the size of the training set. 
As expect, this is an increasing function and we remark that the RBS and the combination of the RPS and the BS give better results than the other global invariant descriptors.

\subsubsection*{\textbf{Robustness against noise}}

Also in this case, test results for ZM, HM, FM, and SIFT are taken from \cite{Choksuriwong2008}.

\begin{figure}
\centering
  \includegraphics[width=0.48\textwidth]{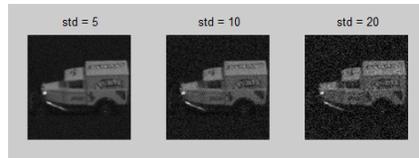}
\caption{Sample of COIL-100 noisy object}
\label{fig:noise}       
\end{figure}

Results presented in Table~\ref{tab:feat-noise} show that noise has little influence on classification performance when we use a global descriptor such as RBS, BS, the combination of BS \& RPS, ZM, HM and FM. It has however a sensible effect on the SIFT local descriptor, and a big one on the HOG local descriptor. 

\begin{table*}[t]
	\centering
	\caption{Classification rate on COIL-100 noisy database. The test results for ZM, HM, FM, and SIFT are taken from \cite{Choksuriwong2008}.}
	\label{tab:feat-noise}
	\begin{tabular}{|c|c|c|c|c|c|c|c|c|c|c|}
	\hline
	$S_d$ & RBS & BS & PS & RPS & RPS+BS & ZM & HM & FM & SIFT & HOG\\
	\hline
	\hline 
	5 & \textbf{100\%} & \textbf{100\%} & 71.5\% & 99.8\% & \textbf{100\%} & \textbf{100\%} & 95.2\% & 98.6\% & 89.27\% & 4\%  	\\ 
	10 & \textbf{100\%} & \textbf{100\%} & 71.2\% & 99.8\% & \textbf{100\%} & \textbf{100\%} & 95.2\% & 95.2\% & 88.89\% & 1.2\%  	\\ 
	20 & \textbf{100\%} & \textbf{100\%} & 67.8\% & 99.8\% & \textbf{100\%} & \textbf{100\%} & 91.4\% & 90.2\% & 85.46\% & 1\%  	\\ \hline

	\end{tabular}
\end{table*}

\subsubsection{The ORL database}

The Cambridge University ORL face database (Fig. \ref{fig:orl}) is composed of 400 gray level images of ten different patterns for each of 40 persons. 
The variations of the images are across time, size, pose and facial expression (open/closed eyes, smiling/not smiling), and facial details (glasses/no glasses). 

\begin{figure}
\centering
  \includegraphics[width=0.48\textwidth]{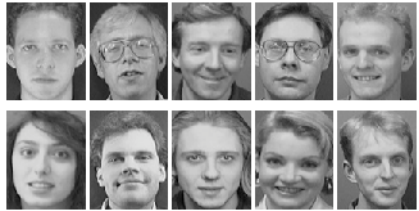}
\caption{Face samples from the ORL database}
\label{fig:orl}       
\end{figure}

In the literature, the protocol used for training and testing is different from one paper to another. In \cite{341300}, a hidden Markov model (HMM) based approach is used, and the best model resulted in recognition rate of 95\%, with high computational cost. In \cite{hjelmas2001}, Hjelmas reached a 85\% recognition rate using the ORL database and feature vector consisting of Gabor coefficients.

We perform experiments on the ORL database using the RBS, BS, PS, RPS, ZM, HU, FM, and the combination of the RPS \& BS descriptors. 
The results are shown in Table \ref{tab:orlTab}, where we clearly see that the RBS invariant descriptor gives the best recognition rate $c = 89.8\%$, faring far better than before w.r.t.\ the combination of RPS and BS descriptors.

\begin{table}
\centering
\caption{Recognition rate for each descriptor using the ORL database}
\label{tab:orlTab}       
\begin{tabular}{|c|c|}
\hline
Descriptors & Recognition rates\\
\hline
\hline
RBS & \textbf{89.8\%}\\
BS & 67.9\%\\
PS & 49.2\%\\
RPS & 76.9\%\\
RPS+BS & 79.8\%\\
ZM & 75\%\\
HM & 43.5\%\\
FM & 47.6\%\\
HOG & \textbf{99.8\%} \\
SIFT & \textbf{99.9\%} \\\hline
\end{tabular}
\end{table}

%% file: appendix.tex
%
%
%

\appendix

\chapter{Circulant matrices}\label{app:circulant}

Let $\bK$ be a finite abelian group and denote by $S:\bK\to \cU(L^2(\bK))$ (the shift operator) its left-regular representation.

\begin{definition}
  An operator $M\in L^2(\bK)\otimes L^2(\bK)$ is circulant if and only if $M\circ S(k) = S(k)\circ M$ for all $k\in\bK$.
  In particular, this is equivalent to 
  \begin{equation}
    M_{i,jk}=M_{ik^{-1},j}\qquad \text{for all } k,i,j\in\bK.
  \end{equation}
  The space of circulant operators over $\bK$ is denoted by $\Circ(\bK)$
\end{definition}

Clearly, $\Circ(\bK)$ is a vector subspace of $L^2(\bK)\otimes L^2(\bK)$. Moreover, it is closed under composition and it is easy to check that $A\circ B = B \circ A$ for all $A,B\in\Circ(\bK)$. Hence, $\Circ(\bK)$ is a commutative algebra, with the identity matrix $\idty$ as identity element.

A simple computation shows that any circulant matrix $M$ is completely determined by the vector $v\in L^2(\bK)$ defined by $v_i = M_{i,e}$, where $e$ is the identity of $\bK$. Indeed, $M_{i,j}=v_{j^{-1}i}$. This allows to define, with abuse of notation, the vector space isomorphism 
\begin{equation}
  \operatorname{Circ}:L^2(\bK)\to \Circ(\bK).
\end{equation} 
As a consequence, $\Circ(\bK)$ has dimension $|\bK|$. Moreover, if $\{e_k\}_{k\in\bK}$ is the canonical basis of $L^2(\bK)$ a simple computation shows that $S(k)=\Circ(e_k)$ for all $k\in\bK$. This shows that $\{S(k)\}_{k\in\bK}$ is a basis for $\Circ(\bK)$.

An important fact is that the Fourier transform on $\bK$ is a bijection between $\Circ(\bK)$ and the set of diagonal matrices on $\widehat \bK$. In other words, the vector of eigenvalues of a circulant matrix $\Circ(v)$ is exactly the Fourier transform $\cF_\bK v$.

\chapter{Bispectrally admissible sets}\label{app:bispectral}

In this appendix we present a concrete procedure to generate bispectrally admissible sets (see Definition~\ref{def:bispectrally-admissible-set}), and some theoretical considerations on the structure of such sets.
In particular, we focus on the procedure given in Algorithm~\ref{algo:bisp}.


  \begin{algorithm}
	\KwData{A rotationally invariant set $F_1\subset \widehat{\bR^2}$ and a number $M\in\bN\cup\{+\infty\}$}
	\KwResult{A bispectrally admissible set $F\subset\widehat\cS$}
	
 \For{$k\leftarrow 2$ \KwTo $M$}{
 	$F_k$ $\leftarrow$ $\left\{\lambda+\mu\mid \lambda,\mu\in F_{k-1}\right\} \cup F_{k-1}$\;
 }
 $F \leftarrow$ quotient $F_M/\bK$\;
 \KwRet{$F$}
 \caption{Algorithm for generating a bispectrally admissible set.}
 \label{algo:bisp}
\end{algorithm}

\section{Structure of bispectrally admissible sets}

  An important case is when $M=+\infty$ and the starting set $F_1$ is chosen to be the set of the $N$-th roots of unity, i.e., $F_1=\{e^{2\pi i k/N}\mid k=0,\ldots,N-1\}$.
  When $N=2$ this yields $F_\infty=\bZ\subset \bR^2$, while for $N=3,4,6$ the set $F_\infty$ turns out to be one of the possible lattices of $\bR^2$.
  Moreover, we have the following.

  \begin{proposition}
    \label{prop:E-group}
    If $N$ is even then the set $F$ obtained from the above procedure, with $M=+\infty$ and $F_1= \{e^{2\pi i k/N}\mid k=0,\ldots,N-1\}$, is a countable additive subgroup of $\mathbb R^2$ on which $\bZ_N$ acts.
  \end{proposition}

  \begin{proof}
    The fact that $F$ is countable is a consequence of the construction. 
    Let us prove that it is a subgroup of $\mathbb R^2$.
    By construction, for any $\lambda,\mu\in F$ we have that $\lambda+\mu\in F$, so the set is closed w.r.t.\ addition.
    Moreover, $0\in F$, since $0=1-1$ and $1,-1\in F_1$ by parity of $N$.
    Finally, let us prove by induction that for any $x\in F_n$ it holds that $-x\in F_n$, which will complete the proof.
    Clearly, again by parity of $N$, this is true for $n=1$.
    Assume this to be true for $n$, and observe that $\nu\in F_{n+1}$ if and only if $\nu=\lambda+\mu$ for $\lambda,\mu\in F_n$.
    Then, $-\nu=-\lambda-\mu\in F_{n+1}$ as well, completing the proof of the claim.

    A simple induction procedure shows that $F$ is rotationally invariant, and hence that the action of $\bZ_N$ restricts to it, which completes the proof.
  \end{proof}
  
  The following proposition clarifies what happens for not necessarily even values of $N$.

  \begin{proposition}
    \label{prop:dense-bisp-adm-set}
    Let $N\ge 5$ but $N\neq 6$.
    Then, the set $F$ obtained from the above procedure, starting with $F_1= \{e^{2\pi i k/N}\mid k=0,\ldots,N-1\}$ and with $M=+\infty$, is dense in $\bR^2$.
  \end{proposition}

  \begin{proof}
    For $\theta\in\bS^1$, let us denote with $L_\theta\subset\bR^2$ the line passing through the origin forming an angle $\theta$ with the $x$-axis.
    Obviously, $L_0=\bR\times\{0\}$.

    In the following we will use these two well-known facts:
    \begin{itemize}
      \item For any irrational number $\alpha$, the set $\alpha\bZ+\bZ$ is dense in $\bR$.
      \item For any $N\ge 5$ and $N\neq 6$, $\cos(2\pi/N)$ is irrational.
    \end{itemize}

    We divide the proof in three steps.

    \begin{enumerate}
      \item \emph{Let $\theta_1,\theta_2\in\bS^1$ and $A\subset L_{\theta_1}$, $B\subset L_{\theta_2}$ be two dense subsets. Then, for $\theta=(\theta_2-\theta_1)/2$ the set $\left(A+B\right)\cap L_\theta$ is dense in $L_\theta$:}
      Without loss of generality we can assume $\theta_1=0$, and  $\theta_2=2\theta$.
      Let us define $f:L_0\to L_\theta$ by $f(p)=p+R_{2\theta}p$, which is clearly bi-continuous.
      Thus, $f(A)$ is dense in $L_\theta$ and to complete the proof it suffices to show that $\overline{f(A)}\subset \overline{A+B}$.
      To this aim, let $y=f(a)\in\overline{f(A)}$ and consider a sequence $(a_n)_n\subset A$ such that $a_n\rightarrow a$.
      Then, $f(a_n)=a_n+R_{2\theta}a_n\rightarrow f(a)$, where $R_{2\theta}a_n\in L_{2\theta}$.
      For any $a_n$ let us consider $(b_{n,k})_k\subset B$ such that $b _{n,k}\rightarrow R _{2\theta}a _n$ as $k\rightarrow+\infty$.
      Then, we have 
      $$
      \lim_n a_n+b_{n,n} = \lim_n a_n+R_{2\theta} a_n= f(a) = y.
      $$
      This proves that $\overline{f(A)}\subset \overline{A+B}$, completing the proof.

      \item \emph{If $\cos(2\pi/N)$ is irrational, then $F\cap L_0$ contains a dense subgroup of $\bR$:}
      Clearly, by construction of $F$, $\mathbb Z\subset F\cap L_0$.
      Simple trigonometric considerations yield $e^{2\pi i/N} + e^{-2\pi i/N} = (2\cos(2\pi/N),0) \in F_2\subset F$.
      This implies that $2\cos(2\pi/N) \mathbb Z \subset F\cap L_0$. 
      Hence, $2\cos(2\pi/N) \mathbb Z + \mathbb Z\subset F\cap L_0$.
      Using then the above cited well-known facts, we complete the proof of the claim.

      \item \emph{The set $F$ is dense in $\bR^2$:}
      Consider the set $V\subset \mathbb S^1$ obtained with the following iterative procedure.
      Fix $V_1 = \{2\pi k/N \mid k=0,\ldots,N-1\}$ and then $V_n = \{ (\theta - \theta')/2 \mid \theta,\theta'\in V _{n-1} \} \cup V _{n-1}$.
      Finally, $V=\bigcup_{n\in\mathbb N} V_n$.
      It is easy to prove that $V$ is a dense subset of $\mathbb S_1$.
      Moreover, by the previous steps of the proof, we have that for any $\theta\in V$ it holds that $F\cap L_\theta$ is dense in $L_\theta$.
      Indeed, $F\cap L_0$ is dense in $L_0$ by the previous step and, for any $k$, $F\cap L_{2\pi k/N}$ is dense in $L_{2\pi k/N}$ by rotational invariance of $F$ which yields the claim using the first step of the proof.

      Let us show that $F$ is dense proceeding by contradiction.
      Assume that there exists an open set $U\subset \mathbb R^2$ such that $U\cap F = \varnothing$.
      Consider the set $P=\{\theta\in\mathbb S^1\mid L_\theta\cap U\neq \varnothing \}$. 
      It is easy to see that $P$ is open in $\bS^1$, which implies that $V\cap P\neq\varnothing$.
      Thus, for some $\theta\in V$ we have $L_\theta\cap U\neq\varnothing$. 
      Since $L_\theta\cap U$ is open in $L_\theta$ and $F\cap L_\theta$ is dense, we finally have that $F\cap U\cap L_\theta\neq\varnothing$ which contradicts the assumption $F\cap U=\varnothing$, completing the proof. 
    \end{enumerate}
    
  \end{proof}
  
  An immediate consequence of the above and Proposition~\ref{prop:counterexample}, is the following.
  
  \begin{corollary}
    If $N$ is even, bispectral invariants are not complete on the set $\cG_F\subset B_2(\bG)$, defined in \eqref{eq:residual-AP}.
  \end{corollary}

%% file: glossary.tex
%
%


\renewcommand\nomgroup[1]{%
  \item[\bfseries
  \ifstrequal{#1}{G}{Groups and sets}{%
  \ifstrequal{#1}{TF}{Transformation of functions}{%
  \ifstrequal{#1}{R}{Representations}{%
  \ifstrequal{#1}{S}{Functional spaces}{%
  \ifstrequal{#1}{O}{Operations}{%
  \ifstrequal{#1}{T}{Operators}{}}}}}}%
]}

\nomenclature[G]{$\bR$}{The (additive) group of real numbers.}
\nomenclature[G]{$\bC$}{The complex plane.}
\nomenclature[G]{$\bS^n$}{The $n$-dimensional sphere.}
\nomenclature[G]{$\bT^n$}{The $n$-dimensional flat torus.}
\nomenclature[G]{$\bG,\bH,\bK$}{Groups. Usually, in order, a generic, an abelian, and a compact and/or finite one.}
\nomenclature[G]{$\bZ_N$}{The finite cyclic group of order $N$.}
\nomenclature[G]{$SE(2)$}{The Euclidean motion group $\bS^1\ltimes \bR^2$.}
\nomenclature[G]{$SE(2,N)$}{The semi-discrete group of motions $\bZ_N\ltimes \bR^2$.}
\nomenclature[G]{$PT\bR^2$}{The projective tangent bundle of $\bR^2$.}
\nomenclature[G]{$\widehat{\bG}$}{Dual object of the group $\bG$. If $\bG$ is abelian, it coincides with the Pontryiagin dual.}
\nomenclature[G]{$\operatorname{Rep}(\bG)$}{Chu dual of $\bG$. (See Section~\ref{sec:chu}.)}
\nomenclature[G]{$\bG^\flat$}{Bohr compactification of the topological group $\bG$.}

\nomenclature[S]{$L^2(\bG)$}{Space of complex-valued functions $f:\bG\to \bC$ that are square-integrable w.r.t.\ the Haar measure.}
\nomenclature[S]{$L^2_\bR(\bG)$}{Space of real-valued functions $f:\bG\to \bR$ that are square-integrable w.r.t.\ the Haar measure.}
\nomenclature[S]{$\AP(\bG)$}{Bohr almost-periodic functions over $\bG$. (See Definition~\ref{def:bohr-ap}.)}
\nomenclature[S]{$\AP_F(\bG)$}{Almost-periodic functions over $\bG$ with frequencies in $F$. (See Section~\ref{sub:finite_dimensional_subspaces_of_almost_periodic_functions}.)}
\nomenclature[S]{$C(\bG)$}{Space of continuous functions over $\bG$.}
\nomenclature[S]{$\bC^F$}{Space of functions $\varphi:F\to \bC$.}
\nomenclature[S]{$C_b(\bG)$}{Space of continuous bounded functions over $\bG$.}
\nomenclature[S]{$B_2(\bG)$}{Besicovitch almost-periodic functions over $\bG$. (See Definition~\ref{def:bohr-ap}.)}
\nomenclature[S]{$\cX$}{$\bR$-linear subspace of $L^2(\bK)$. (See Proposition~\ref{prop:omega-real}.)}
\nomenclature[S]{$\cC$}{Space of weakly-cyclic functions in $L^2(\bH)$. (See Definition~\ref{def:cc-weakly-cyclic}.)}
\nomenclature[S]{$\cC_\bR$}{Space of $\bR$-weakly-cyclic functions in $L^2_\bR(\bH)$. (See Definition~\ref{def:cc-weakly-cyclic-real}.)}
\nomenclature[S]{$\cC^{\AP}$}{Space of almost-periodic weakly-cyclic functions. (See Definition~\ref{def:APweakly-cyclic}.)}
\nomenclature[S]{$\cV(D_R)$}{Space of $L^2(\bR^2)$-functions supported in the disk $D_R$. (See Section~\ref{sec:square-integrable-functions-on-the-plane}.)}

\nomenclature[TF]{$\hat f$}{Fourier transform of $f\in L^2(\bG)$ w.r.t.\ the group structure of $\bG$.}
\nomenclature[TF]{$\hat f_\lambda$}{Vector of $L^2(\bK)$ obtained by evaluating $\hat f$ on the inverse orbit of $\lambda$ w.r.t.\ $\bK$. That is, $\hat f_\lambda(k)=\phi(k^{-1})\hat f(\lambda)$. (See Section~\ref{sub:weakly_cyclic_functions}.)}
\nomenclature[TF]{$\Psi^*$}{The involution of $\Psi\in L^2(\bG)$, defined by $\Psi^*(a) = \overline{\Psi}(g^{-1})$.}

\nomenclature[T]{$\cF_\bG$}{Fourier transform $\cF_\bG:L^2(\bG)\to L^2(\widehat\bG)$.}
\nomenclature[T]{$W_\Psi f$}{Wavelet transform of $f$. (See Section~\ref{sec:wavelet}.)}
\nomenclature[T]{$\PS_f$}{Power spectrum invariant of $f$.}
\nomenclature[T]{$\BS_f$}{Bispectrum invariant of $f$.}
\nomenclature[T]{$\RPS_f$}{Rotational power spectrum invariant of $f$.}
\nomenclature[T]{$\RBS_f$}{Rotational bispectrum invariant of $f$.}
\nomenclature[T]{$\cJ$}{Generalized Fourier-Bessel operator. (See Definition~\ref{def:gen-bessel}.)}
\nomenclature[T]{$\cJ^E$}{Discrete Fourier-Bessel operator. (See Definition~\ref{def:discr-bessel}.)}
\nomenclature[T]{$\Xi$}{Bijection between $\bK\rtimes\bH$ and $\bK\times\bK\times \bH/\bK$ used in Chapter~\ref{ch:ap_interp}.}
\nomenclature[T]{$\sampl$}{Sampling operator. (See Section~\ref{sec:ap-interp}.)}
\nomenclature[T]{$\ev$}{Evaluation operator. (See Section~\ref{sec:ap-interp}.)}
\nomenclature[T]{$R_\theta$}{Rotation of angle $\theta$ in $\bR^2$.}
\nomenclature[T]{$\diag v$}{Diagonal matrix with entries given by the vector $v$.}
\nomenclature[T]{$\Circ v$}{Circulant matrix with entries given by the vector $v$. (See Appendix~\ref{app:circulant}.)}


\nomenclature[R]{$\tau_x$}{Left-regular representation, or translation, of an abelian group at $x\in\bH$. (See Section~\ref{sec:repr-semidirect}.)}
\nomenclature[R]{$S(k)$}{Left-regular representation, or shift, of a finite abelian group at $k\in\bK$. (See Section~\ref{sec:repr-semidirect}.)}
\nomenclature[R]{$\Lambda(a)$}{Left-regular representation of a non-commutative group at $a\in\bG$. (See Section~\ref{sec:fourier-non-commutative}.)}
\nomenclature[R]{$\pi(k,x)$}{Quasi-regular representation of a semidirect product at $(k,x)\in\bK\rtimes\bH$. (See Section~\ref{sec:repr-semidirect}.)}
\nomenclature[R]{$\pi_{B_2}$}{$B_2(\bH)$-quasi-regular representation. (See Section~\ref{sec:MAP-and-AP}.)}

\printnomenclature